\documentclass[pdflatex,sn-apa]{sn-jnl}

\usepackage{amsmath, amssymb}
\usepackage{graphicx}
\usepackage{booktabs}
\usepackage{siunitx}
\usepackage{url}
\usepackage{float}
\usepackage{hyperref}
\usepackage{algorithm}      
\usepackage{algpseudocode}
\usepackage{tikz}
\usetikzlibrary{arrows.meta,positioning,shapes.geometric}
\usepackage{booktabs}
\usepackage{multirow}
\usepackage{tabularx}
\usepackage{graphicx}
\usepackage{makecell}
\usepackage{lmodern}
\usepackage{float}
\usepackage{biblatex}
\usepackage{subcaption}
\newtheorem{theorem}{Theorem}

\theoremstyle{plain}
\newtheorem{assumption}{Assumption}

\numberwithin{intassumption}{assumption}

\graphicspath{ {./images/} }

\begin{document}

\title[Article Title]{SIGMA: Symmetry-aware, Intelligent, Geometric, Multi-faced Adaptive Traffic Control for Reliable Intelligent Transportation Systems}


\author[1,2]{\fnm{Pratham} \sur{Payra}}\email{prathamfbm@gmail.com}

\author[2]{\fnm{Jagadish} \sur{}}\email{jagadishbaridabad.s@gmail.com}

\author[1]{\fnm{Tanmay} \sur{Sen}}\email{sentanmay518@gmail.com}
\author[3]{\fnm{Tanujit} \sur{Chakraborty}\footnote{Corresponding Author}}\email{tanujit.chakraborty@sorbonne.ae}

\affil[1]{\orgdiv{Statistical Quality Control \& Operations Research Unit}, \orgname{Indian Statistical Institute, Kolkata}, \orgaddress{\country{India}}}

\affil[2]{\orgdiv{Statistical Quality Control \& Operations Research Unit}, \orgname{Indian Statistical Institute, Bangalore}, \orgaddress{\country{India}}}

\affil[3]{\orgdiv{Sorbonne Center for Artificial Intelligence}, \orgaddress{\street{Sorbonne University}, \city{Abu Dhabi}, \country{UAE}}}


\abstract{Traffic signal control is a challenging sequential decision-making problem that requires reliable and timely adaptation while balancing competing objectives, including traffic throughput, fairness in vehicle delays, signal stability, predictable switching, and emergency vehicle prioritisation. Existing reinforcement learning (RL) approaches typically optimise predefined objectives, provide limited support for dynamically changing operational priorities, and may generalise poorly across geometrically equivalent intersections. We propose SIGMA (Symmetry-aware, Intelligent, Geometric, Multi-faced Adaptive traffic control), a multi-faced RL framework augmented with a large language model (LLM) for dynamic objective adaptation and orientation-invariant policy learning. SIGMA translates natural-language emergency instructions into priority vectors that guide a multi-faced actor-critic controller. Symmetry-aware rotational augmentation improves transferability across homogeneous four-way intersections, while an offline-to-online learning strategy provides stable initialisation and subsequent adaptation to evolving traffic conditions. To assess system reliability, we establish structural properties characterising emergency service levels, graceful degradation under LLM failure, and sensitivity to traffic-demand perturbations, complemented by bootstrap-based statistical validation. We evaluate SIGMA using SUMO on four urban intersections modelled after traffic scenarios in Kolkata, India, against fixed-time, actuated, and Deep Q-Network controllers. SIGMA reduces average and emergency waiting times and queue lengths while improving traffic throughput. Empirical reliability and ablation analyses further demonstrate robustness to component failure and geometric orientation. Overall, SIGMA provides a reliable and adaptive intelligent traffic-control framework that integrates language-guided decision-making, multi-faced learning, and statistical reliability assessment.}





\maketitle

\section{Introduction}
\label{sec:intro}
Traditional traffic signal control strategies \cite{manikandan2026systematic}, including fixed-time, adaptive, and actuated control, typically rely on predetermined timing plans or local traffic measurements to determine signal phases. While these approaches are computationally efficient and relatively straightforward to implement, their ability to respond to rapidly evolving traffic conditions and simultaneously accommodate multiple, potentially conflicting, operational objectives remains limited. To improve traffic-responsive decision-making, metaheuristic optimisation methods, including Genetic Algorithms (GA), Particle Swarm Optimisation (PSO), and Ant Colony Optimisation (ACO), have been extensively investigated for optimising signal timing plans \cite{shaikh2020review, shirke2022metaheuristic}. These methods can improve traffic performance by searching for near-optimal signal configurations under specified traffic conditions. However, their effectiveness in highly dynamic environments can be constrained by repeated optimisation, sensitivity to parameter choices, and dependence on explicit traffic representations, which may limit their suitability for reliable real-time control \cite{cascetta2006models, zhang2022datadriven, article54}.

Reinforcement learning (RL) has emerged as an effective framework for adaptive traffic signal control \cite{b2,b3,b10}. In RL-based systems, traffic conditions define the state, signal phases constitute the actions, and the controller learns a policy that optimises long-term traffic performance through sequential interaction with the environment. Recent advances in deep reinforcement learning, graph neural networks, and multi-agent coordination have further improved the scalability and effectiveness of learning-based traffic controllers \cite{b10,b11,b13,b14}. Despite this progress, several challenges limit their flexibility and reliability in dynamic traffic environments \cite{b19,b20,b33}. First, most RL-based controllers assume that optimisation objectives are specified during training and remain fixed at deployment. In practice, operational priorities can change abruptly: an approaching ambulance may require immediate signal priority, while an accident or public event may require temporary prioritisation of a particular traffic stream. Accommodating such changes typically requires manual reward redesign or application-specific rules, limiting dynamic adaptation \cite{b5,b27}. Second, many RL controllers rely primarily on local descriptors, such as queue lengths and waiting times, which may not adequately represent traffic pressure, downstream congestion, spillback effects, or interactions among competing objectives. Consequently, jointly balancing traffic efficiency, waiting-time fairness, signal stability, predictable switching, and emergency responsiveness remains challenging \cite{b10,b20,b45}. Third, learned policies may exhibit limited transferability across geometrically equivalent intersections. For example, two homogeneous four-way intersections that differ only by a $90^\circ$ rotation can be treated as distinct environments when state and action representations are tied to absolute cardinal directions rather than relative traffic geometry. Such orientation dependence can necessitate site-specific retraining, increasing deployment cost and limiting generalisability \cite{b14,b19}.

Despite these advances, traffic signal control remains fundamentally a multi-faced sequential optimisation problem \cite{mahmoodi2025uav,  kohler2018traffic, rahman2025optimising, article50}. Practical traffic management requires the simultaneous consideration of traffic throughput, vehicle delays, fairness among competing traffic streams, smooth and predictable signal transitions, and rapid response to emergency vehicles \cite{ binsfeld2025green, 9740468}. These objectives are inherently competing: prioritising one operational criterion may degrade performance with respect to another. Multi-objective reinforcement learning (MORL) provides a natural framework for addressing such trade-offs by incorporating multiple performance criteria through weighted reward formulations or Pareto-based optimisation. More broadly, multi-faced optimisation has been extensively studied across traffic and transportation systems, including cooperative control, hierarchical decision-making, evolutionary optimisation, routing, and resource-allocation problems \cite{CHOW2017248, CHEN2023738, ZAJAC20211, saezaguado2009fixed}. Collectively, these studies demonstrate the value of explicitly accounting for competing operational objectives and motivate the development of learning-based traffic controllers capable of balancing multiple criteria within a unified decision-making framework.

Although these developments represent important progress toward multi-faced traffic signal optimisation, several limitations remain from an operational decision-making perspective. First, existing approaches commonly assume that the optimisation objectives and their relative importance are specified a priori and remain fixed during deployment. Real-world traffic management, however, involves time-varying operational priorities: emergency vehicles may require immediate preemption, particular traffic streams may need temporary priority following an accident or public event, and congestion mitigation may become dominant during peak periods. Accommodating such changes often requires manual modification of reward functions, objective weights, or policy retraining, limiting the ability of learned controllers to adapt reliably during operation. Second, existing formulations typically consider only a subset of relevant traffic-control criteria rather than jointly accounting for traffic pressure, waiting-time fairness, signal-transition stability, predictable switching, and emergency responsiveness. The interactions and trade-offs among these criteria make their simultaneous optimisation particularly challenging. Third, policy transfer across geometrically equivalent intersections remains limited. A controller trained for one intersection orientation may require retraining when deployed at an otherwise equivalent intersection with rotated approaches, increasing computational and deployment costs and restricting practical scalability \cite{xu2026multicriteria, deng2026adaptive}. These limitations motivate a traffic-control framework in which operational priorities can be modified dynamically, competing objectives can be coordinated within a unified learning architecture, and learned policies can exploit geometric symmetries to improve transferability.

Recent advances in large language models (LLMs) provide a promising mechanism for introducing dynamic, human-interpretable operational priorities into intelligent traffic-control systems. While conventional reinforcement learning typically relies on numerical objectives specified during model design and training, LLMs can interpret high-level instructions expressed in natural language and translate them into structured representations suitable for downstream decision-making. In this setting, the LLM need not replace the reinforcement learning controller; instead, it can serve as a semantic interface between traffic operators and the underlying control policy, allowing operational priorities to be communicated dynamically without manually redesigning the reward structure. Such a separation between high-level instruction interpretation and low-level sequential control is particularly relevant to intelligent transportation systems, where emergency events and changing traffic-management requirements may demand rapid and context-dependent adaptation that is difficult to accommodate through single optimisation objectives alone \cite{huang2023three, reyad2023real}.

To address these limitations, we propose SIGMA (\textbf{S}ymmetry-aware, \textbf{I}ntelligent, \textbf{G}eometric, \textbf{M}ulti-faced \textbf{A}daptive Traffic Control), a reliability-aware multi-faced actor-critic framework for adaptive and intelligent traffic signal control. SIGMA integrates an LLM-based semantic interface that translates high-level traffic-management directives expressed in natural language into structured priority vectors, which subsequently guide the RL controller without requiring manual redesign of the underlying reward functions. The resulting controller jointly accounts for traffic throughput, waiting-time fairness, signal stability, predictable signal switching, and emergency responsiveness within a unified multi-faced learning framework. To improve transferability across homogeneous four-way intersections, SIGMA further incorporates symmetry-aware rotational augmentation, allowing the learned policy to exploit geometric equivalence across different cardinal orientations. An offline-to-online learning strategy combines policy initialisation from pretraining with subsequent online refinement, enabling stable initialisation and adaptation to evolving traffic conditions. Beyond operational performance, SIGMA is designed with reliability considerations that enable its behaviour to be characterised under changing priorities, component failure, traffic-demand perturbations, and geometric transformations.

The primary contributions of this work are summarised as follows.
\begin{itemize}

\item We propose SIGMA, a reliability-aware multi-faced actor-critic framework that translates natural-language traffic-management instructions into priority vectors, enabling dynamic adaptation of operational priorities without manual reward redesign.

\item We develop a symmetry-aware rotational augmentation strategy for homogeneous four-way intersections that exploits geometric equivalence across orientations to improve policy transferability. We formulate traffic signal control as a differentiable multi-faced learning problem that jointly considers traffic throughput, waiting-time fairness, signal stability, predictable switching, and emergency responsiveness.

\item We establish structural reliability properties of SIGMA, characterizing performance degradation under LLM fallback, rotational invariance and equivariance, queue stability, and smooth adaptation to changing operational priorities.

\item We develop an offline-to-online learning strategy and evaluate SIGMA on SUMO-based urban traffic scenarios from Kolkata, India. Extensive experiments, including reliability, sensitivity, and ablation analyses, demonstrate improved traffic efficiency, emergency responsiveness, and robustness compared with conventional and learning-based controllers. We also analyse reward-balancing strategies tailored to different operational objectives.
\end{itemize}

The remainder of this paper is organised as follows. Section~\ref{sec:relatedworks} reviews the related literature, and Section~\ref{sec:prob_for} formulates the traffic signal control problem. Section~\ref{sec:res_met} introduces the proposed SIGMA framework, while Section~\ref{sec:reliability} establishes its structural and reliability properties. Section~\ref{sec:data_gen} presents the experimental evaluation, and Section~\ref{sec:ablation_sensitivity} discusses the ablation study and sensitivity analysis. Section~\ref{sec:limfurwork} discusses the limitations and directions for future research, and Section~\ref{sec:conclusion} concludes the paper.

\section{Related Work} \label{sec:relatedworks}
Traffic signal control has evolved from conventional rule-based strategies toward increasingly adaptive and learning-based decision-making frameworks. Traditional approaches are commonly classified into pre-timed and vehicle-actuated control. Pre-timed methods determine fixed signal schedules using historical traffic patterns \cite{b45,b46,b47}, whereas vehicle-actuated controllers adjust signal timings in response to real-time detector measurements \cite{b48,b49}. Although computationally efficient and widely applicable, these approaches rely largely on predetermined timing structures or handcrafted control rules, which can limit their ability to anticipate future traffic conditions and coordinate decisions under dynamically evolving demand. These limitations have motivated the development of data-driven and learning-based traffic-control strategies.

Reinforcement learning (RL) provides a natural framework for adaptive traffic signal control by formulating signal operation as a sequential decision-making problem in which controllers learn policies through interaction with the traffic environment \cite{b50,b51}. Early applications based on tabular Q-learning \cite{b53,b54} were constrained by discrete and relatively low-dimensional state spaces. Deep reinforcement learning subsequently extended this paradigm by using neural function approximators to accommodate richer traffic representations, including queue lengths, delays, and vehicle-position information \cite{b56,b57,b59}. For example, IntelliLight \cite{b10} employs deep RL with phase-gated representations and historical traffic information to learn adaptive signal-control policies from real-world traffic data. Nevertheless, RL-based controllers operating primarily at the intersection level may have limited awareness of downstream congestion and broader supply–demand interactions, motivating the incorporation of traffic-pressure information into learning-based control.

Pressure-based control provides an alternative perspective by explicitly accounting for imbalances between upstream demand and downstream traffic conditions. Max-pressure (MP) control \cite{b37,b38} selects signal phases according to pressure differentials and provides important theoretical properties related to network stability and throughput. Pressure-based principles have also been integrated with deep reinforcement learning; for example, \cite{b9} combines max-pressure information with deep RL for arterial traffic coordination, allowing the controller to capture supply–demand imbalances that may not be adequately represented by incoming queue lengths alone. Such approaches provide an important foundation for congestion-aware control. However, their primary emphasis is typically traffic efficiency and network pressure, rather than the simultaneous treatment of broader operational criteria such as waiting-time fairness, signal stability, predictable switching, and dynamically specified emergency priorities.

More recently, large language models (LLMs) have introduced new possibilities for incorporating natural-language information into intelligent traffic-management systems. LLMLight \cite{b39} investigates the use of LLMs as traffic signal controllers through knowledge-informed prompting, whereas TrafficGPT \cite{b42} explores conversational traffic management through an LLM-based framework. These studies demonstrate the potential of LLMs to interpret traffic-related information and operational instructions, but they adopt different roles for language models within the control process. Direct LLM-based control places the language model within the tactical decision loop, whereas conversational systems primarily support higher-level interaction and decision assistance. This distinction motivates an intermediate architecture in which semantic interpretation and real-time control are separated: natural-language instructions can be converted into structured operational information, while a dedicated learning-based controller remains responsible for sequential signal decisions.

Despite substantial progress across these research directions, several gaps remain in developing reliable and adaptive traffic signal-control systems. Existing formulations generally address only subsets of the operational requirements encountered in practice, including traffic-pressure reduction, waiting-time fairness, transition smoothness, signal-sequence consistency, and emergency responsiveness. Moreover, dynamically changing priorities are often handled through predefined rules or modifications to the optimisation objective, while learned policies may remain sensitive to the geometric orientation used during training. Existing approaches also provide limited integration between high-level semantic instructions and low-level multi-faced sequential control. These observations motivate a unified framework that combines pressure-aware state representations, language-guided operational priorities, multi-faced reinforcement learning, and symmetry-aware policy learning. SIGMA is developed from this perspective, while additionally incorporating structural and empirical reliability analyses to characterise its behaviour under component failure, demand perturbations, and geometric transformations.

\section{Problem Formulation}
\label{sec:prob_for}
We consider a homogeneous four-legged signalised intersection with approaches indexed by $\mathcal{D}=$ $\{e, n, w, s\}$, corresponding to east, north, west, and south, respectively. Each approach is assumed to have the same lane configuration, turning permissions, and operational characteristics so that the intersection degree is $d_v=4$. This homogeneous setting provides the geometric structure required for the symmetry-aware formulation developed subsequently.

At each discrete decision epoch $t$, the traffic environment is represented by the local state
\[
\mathbf{s}^{(t)} = \left[ \mathbf{p}^{t-1}, \mathbf{M}_q^{t}, \boldsymbol{\Gamma}^{t}, \boldsymbol{\theta}^{t} \right] \in \mathcal{S},
\]
where $\mathbf{p}^{t-1}$ represents the previously executed signal phase, $\mathbf{M}_q^t$ captures movement-level traffic pressure, $\boldsymbol{\Gamma}^t$ records approach-level waiting times, and $\boldsymbol{\theta}^t$ encodes emergency-movement priorities. These components characterise the information available to the controller for sequential signal-phase decisions.


To represent congestion and downstream supply conditions, assume
$
\mathbf{L}_{\mathrm{in}}^t=\left(\frac{l_e^{\mathrm{in}, t}}{l_{e(\max )}^{\mathrm{in}}}, \frac{l_n^{\mathrm{in}, t}}{l_{n(\max )}^{\mathrm{in}}}, \frac{l_w^{\mathrm{in}, t}}{l_{w(\max )}^{\mathrm{in}}}, \frac{l_s^{\mathrm{in}, t}}{l_{s(\max )}^{\mathrm{in}}}\right)^{\top} \in \mathbb{R}^4
$. Similarly, the normalized outgoing queue vector is defined as
$
\mathbf{L}_{\mathrm{out}}^{t}
=
\left(
\frac{l_{e}^{\mathrm{out},t}}{l_{e}^{\mathrm{out}(\max)}},
\frac{l_{n}^{\mathrm{out},t}}{l_{n}^{\mathrm{out}(\max)}},
\frac{l_{w}^{\mathrm{out},t}}{l_{w}^{\mathrm{out}(\max)}},
\frac{l_{s}^{\mathrm{out},t}}{l_{s}^{\mathrm{out}(\max)}}
\right)^{\top}
\in \mathbb{R}^{4}.
$ Where \(l_i^{\mathrm{in},t}\) and \(l_i^{\mathrm{out},t}\) denotes    the queue length on the corresponding incoming and outgoing leg. \(l_i^{\mathrm{in}(\max)}\) and \(l_i^{\mathrm{out}(\max)}\) denote the physical detection-zone capacities of incoming and outgoing directions, respectively.

Let each potential movement be indexed by an ordered pair $(i, j) \in \mathcal{D} \times \mathcal{D}$, where $i$ denotes the incoming approach and $j$ the outgoing approach. The movement level incoming and outgoing pressure masks are defined by
$$M_{q(\text { in }), i j}^t=L_{\text {in }, i}^t, \quad M_{q(\text { out }), i j}^t=L_{\text {out }, j}^t, \quad i, j \in \mathcal{D}.$$
Stacking these quantities over the 16 origin destination movements gives
$\mathbf{M}_{q(\mathrm{in})}^t, \mathbf{M}_{q(\mathrm{out})}^t \in \mathbb{R}^{16}.$
Thus, each movement $i \rightarrow j$ inherits the normalized queue level of its incoming approach $i$ and the normalized downstream queue level of its outgoing approach $j$. The corresponding movement level net pressure is
$\widetilde{\mathbf{M}}_q^t=\mathbf{M}_{q(\mathrm{in})}^t-\mathbf{M}_{q(\mathrm{out})}^t,$
so that, for movement $i \rightarrow j$,
$\widetilde{M}_{q, i j}^t=L_{\mathrm{in}, i}^t-L_{\mathrm{out}, j}^t.$
Positive values therefore indicate movements for which upstream demand exceeds downstream congestion, whereas negative values indicate relatively constrained downstream conditions. For numerical convenience, the pressure vector is shifted to a non-negative representation,
$$\mathbf{M}_q^t=\widetilde{\mathbf{M}}_q^t-\min _{i, j} \widetilde{M}_{q, i j}^t \mathbf{1}_{16},$$
which preserves all pairwise pressure differences while setting the smallest component to zero. Unlike scalar or approach-level congestion measures, $\mathbf{M}_q^t$ retains movement-specific information on the imbalance between upstream demand and downstream traffic conditions. The construction of the incoming and outgoing queues is illustrated in Figure~\ref{fig:Inc_otg_li}.

\begin{figure*}[!t]
    \centering
    \includegraphics[width=\linewidth]{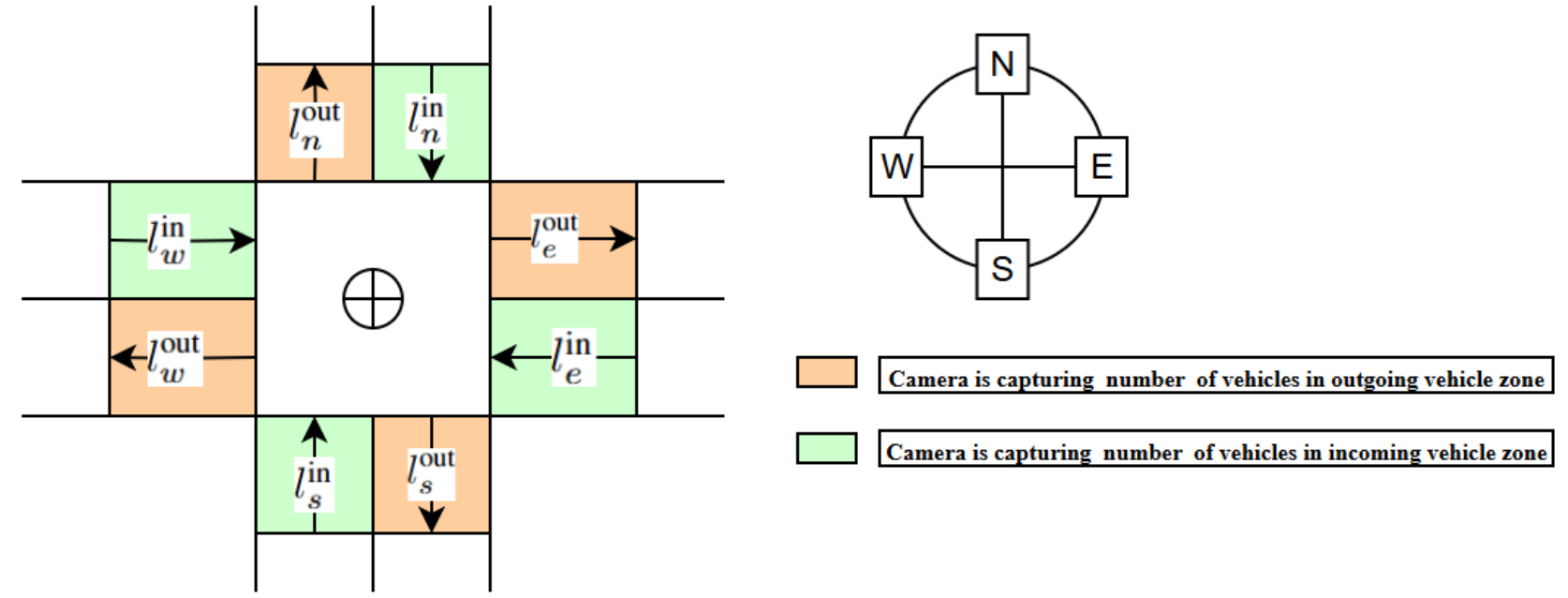}
    \caption{Illustration of incoming and outgoing queue measurements at a four-way intersection. The incoming $l_i^{\mathrm{in}}$ and outgoing $l_i^{\text{out}}$ queue lengths, $i \in\{E, N, W, S\}$, are measured within the corresponding approach-specific detection zones and used to construct the movement-level traffic-pressure representation.}
    \label{fig:Inc_otg_li}
\end{figure*}

The phase history vector $\mathbf{p}^{t-1} \in\{0,1\}^{16}$ encodes the set of movements activated by the signal phase executed at the previous decision epoch. Let $a^{t-1} \in\{0,1\}^8$ denote the corresponding one-hot representation over the eight admissible signal phases. The movement-level phase history is obtained through the fixed phase-to-movement transformation matrix ${A}_{\mathrm{TR}} \in\{0,1\}^{8 \times 16}$ as
$\mathbf{p}^{t-1}=a^{t-1} {A}_{\mathrm{TR}}.$
This representation retains information about the previously served movements and is subsequently used to promote temporal consistency and smooth signal transitions. To characterize waiting-time fairness, let
$
\Gamma^t
=
\left(
\tau_e^t,\tau_n^t,\tau_w^t,\tau_s^t
\right)^\top
\in \mathbb{R}^4
$
denote the maximum waiting times observed on the east, north, west, and south approaches at time \(t\), respectively. Specifically,
$
\tau_i^t
=
\max_{v\in\bar{\mathcal{V}}_i^t}
\tau_v^t,
$
where \(\bar{\mathcal{V}}_i^t\) is the set of vehicles waiting on approach \(i\) at time \(t\), and \(\tau_v^t\) denotes the waiting time of vehicle \(v\). The corresponding movement level waiting time mask $\mathbf{M}_w^t \in \mathbb{R}^{16}$ assigns the waiting time of each incoming approach to all movements originating from that approach:
$$\mathbf{M}_w^t=(\underbrace{\tau_e^t, \ldots, \tau_e^t}_4, \underbrace{\tau_n^t, \ldots, \tau_n^t}_4, \underbrace{\tau_w^t, \ldots, \tau_w^t}_4, \underbrace{\tau_s^t, \ldots, \tau_s^t}_4)^{\top}.$$

The approach-level vector \(\Gamma^t\) is used directly as part of the state representation, while the derived movement-level mask \(M_w^t\) is used subsequently in the computation of the waiting-time fairness utility defined in Table~\ref{tab:lossreward}.

The emergency-priority vector $\boldsymbol{\theta}^t \in\{0,1\}^{16}$ is a sparse binary vector indicating the movements that require priority at time $t$. Each component corresponds to an origin-destination movement $i \rightarrow j$, with $\theta_{i j}^t=1$ indicating that the corresponding movement should receive emergency priority and $\theta_{i j}^t=0$ otherwise. The priority vector is generated from natural language traffic management instructions by the LLM based semantic module described in Section~\ref{par:llm}. Here, $\boldsymbol{\theta}^t$ specifies the desired emergency priority rather than directly determining the signal phase; the multi-faced controller selects the final control action.

The action space consists of eight admissible signal phases, organized into three groups according to the movements they serve, as illustrated in Figure~\ref{fig:overview}. Type-1 phases $\left(G_1\right)$ serve the permitted movements originating from a single approach, yielding four possible phases. Type-2 phases $\left(G_2\right)$ simultaneously serve straight-through and left-turn movements from opposite approaches, with separate phases for the north-south and east-west pairs. Type-3 phases $\left(G_3\right)$ similarly serve right-turn and U-turn movements from opposite approaches, again yielding one north-south and one east-west phase. Hence, the admissible action space contains eight signal phases in total,
$$\mathcal{A}=G_1 \cup G_2 \cup G_3, \quad\left|G_1\right|=4, \quad\left|G_2\right|=2, \quad\left|G_3\right|=2,$$
so that $|\mathcal{A}|=8$.

\begin{figure*}[!t]
    \centering
    \includegraphics[width=\linewidth]{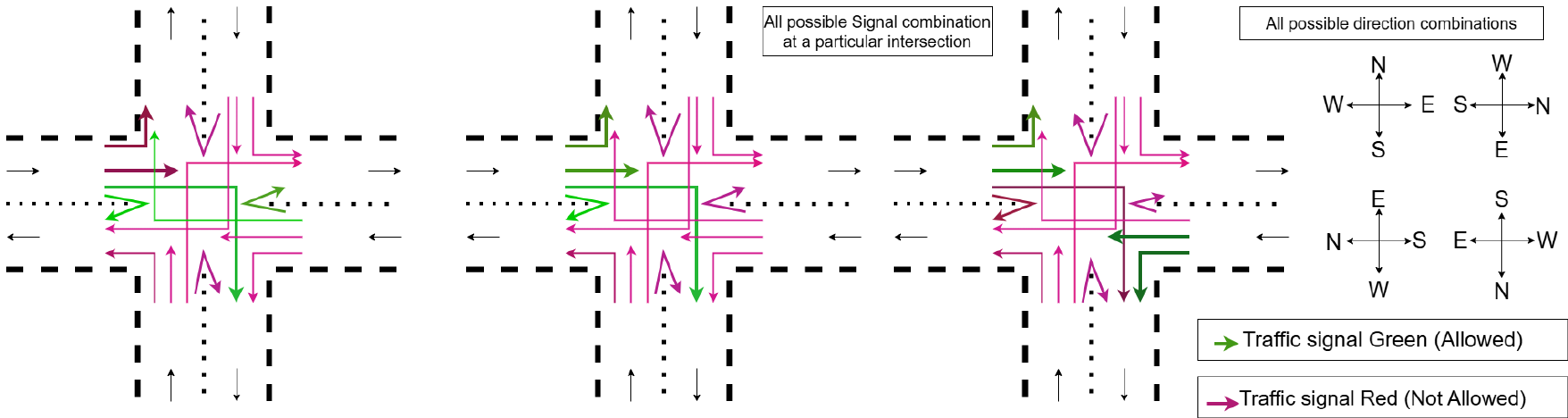}
    \caption{Signal phases are organized into three pattern groups. Type 1 phases ($G_1$) serve a single incoming approach. Type 2 phases ($G_2$) allow straight-through and left-turn movements from opposite approaches. Type 3 phases ($G_3$) allow right-turn and U-turn movements from opposite approaches. In total, eight admissible phases are defined.}
    \label{fig:overview}
\end{figure*}

Given the observed state $\mathbf{s}^t \in \mathcal{S}$, the controller selects an admissible signal phase $a^t \in \mathcal{A}$. The decision problem is inherently multi-faced, requiring the controller to balance six interrelated operational goals: (i) prioritise emergency movements while minimising disruption to regular traffic; (ii) reduce overall vehicle travel time; (iii) limit excessive individual waiting times to promote fair service across approaches; (iv) maintain smooth traffic progression by avoiding unnecessary phase changes; (v) preserve temporal consistency in signal-phase transitions, particularly during emergency priority operations; and (vi) reduce traffic-pressure imbalances by serving movements with high net pressure while accounting for downstream congestion and spillback.

These objectives may be competing; for example, aggressive emergency prioritisation can increase delays for regular traffic, while excessive emphasis on throughput may lead to prolonged waiting on lower-demand approaches. Traffic signal control is therefore formulated as a sequential multi-faced decision problem in which a policy $\pi\left(a^t \mid \mathbf{s}^t\right)$ must balance efficiency, fairness, stability, and emergency responsiveness over time. The specific utility functions through which these operational objectives are incorporated into SIGMA are introduced in Section~\ref{sec:res_met}.

\section{Proposed SIGMA Method}
\label{sec:res_met}

\begin{figure*}[!t]
    \centering
    \includegraphics[width=.96\linewidth]{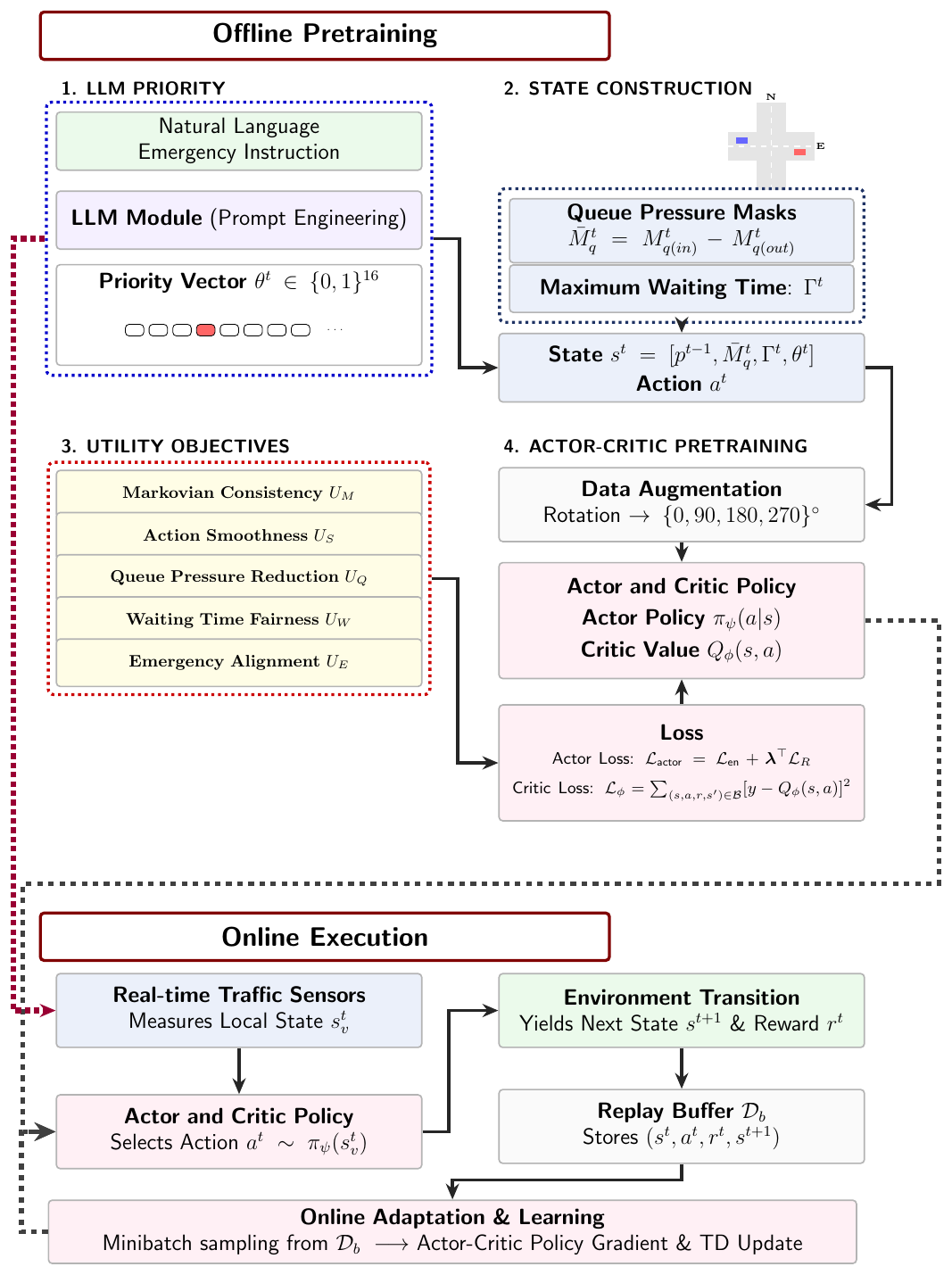  }
    \caption{Overview of the SIGMA architecture, including the offline pretraining and online execution phases.}
    \label{fig:arc}
\end{figure*}

SIGMA follows an offline-to-online learning architecture comprising two stages: offline pretraining and online adaptation, as illustrated in Figure~\ref{fig:arc}. During offline pretraining, the actor and critic networks are initialised using synthetic traffic trajectories that incorporate pressure-aware state representations and LLM-generated emergency-priority vectors. This stage exposes the controller to diverse combinations of traffic conditions and emergency scenarios before deployment, providing a stable initialisation for subsequent reinforcement learning. In particular, the movement-level pressure representation introduced in Section~\ref{sec:prob_for} captures the imbalance between upstream demand and downstream congestion, enabling the controller to account for both traffic demand and downstream spillback when learning signal control decisions.

During online operation, natural language traffic management instructions are translated by the LLM-based semantic module into structured emergency-priority vectors. These priorities are combined with real-time traffic measurements to construct the current state and are subsequently processed by the multi-faced actor-critic controller. Starting from the pretrained initialisation, the policy is refined through interaction with the traffic environment, allowing it to adapt to evolving traffic conditions while balancing emergency responsiveness, traffic efficiency, fairness, and signal stability. This separation between semantic interpretation and sequential control allows operational priorities to change dynamically without requiring manual reward redesign or retraining the controller from scratch.

\subsection{Action representation}
At each decision epoch $t$, the selected signal phase is represented at the movement level by a binary matrix
$${p}^t=\left(A_{i j}^t\right)_{4 \times 4} \in\{0,1\}^{4 \times 4}, \quad i, j \in\{\mathrm{E}, \mathrm{~N}, \mathrm{~W}, \mathrm{~S}\},$$
where $A_{i j}^t=1$ indicates that the movement from incoming approach $i$ to outgoing approach $j$ is permitted under the selected phase, and $A_{i j}^t=0$ otherwise. As defined in Section~\ref{sec:prob_for}, the action space $\mathcal{A}$ consists of eight admissible signal phases. Each phase is encoded by a one-hot vector
${a}^t \in\{0,1\}^8, \; \; \mathbf{1}^{\top} {a}^t=1.$

A fixed phase-to-movement transformation matrix
${A}_{\mathrm{TR}} \in\{0,1\}^{8 \times 16}$ maps the selected phase to its corresponding movement-level representation:
$$\mathbf{p}^t={a}^t {A}_{\mathrm{TR}} \in\{0,1\}^{16},$$
where each component of $\mathbf{p}^t$ indicates whether the corresponding origin-destination movement is activated by the selected phase. The explicit construction of ${A}_{\text {TR }}$ is provided in Appendix~\ref{subsec:action-encoding}.

\subsection{Emergency-priority generation}\label{par:llm}
For a homogeneous four-legged intersection, emergency priorities are represented by a binary matrix
$$
\boldsymbol{\Theta}^t=\left(\theta_{i j}^t\right)_{4 \times 4}, \quad \theta_{i j}^t \in\{0,1\},
$$
where $\theta_{i j}^t=1$ indicates that movement $i \rightarrow j$ requires emergency priority at time $t$. The matrix is vectorized using the operator $\mathcal{R} \mathcal{T}(\cdot)$ to obtain the sparse movement-level priority vector
$$
\boldsymbol{\theta}^t=\mathcal{R} \mathcal{T}\left(\boldsymbol{\Theta}^t\right) \in\{0,1\}^{16} .
$$
Here, we consider only movements associated with the requested emergency route that are activated, thereby limiting unnecessary disruption to regular traffic. SIGMA considers three instruction types. For an exact-path instruction, only the specified origin-to-destination movement is activated. For an incoming-only instruction, all admissible movements originating from the specified approach are assigned priority. For an outgoing-only instruction, priority is assigned to all admissible movements terminating at the specified outgoing approach. This representation provides a common structured interface between natural language instructions and the downstream traffic controller.

For each intersection $v$ and emergency instruction $m$, the LLM prompt is constructed as
$$
\mathcal{P} \mathcal{R}_v=S \oplus \mathcal{I} \mathcal{D} \oplus \mathcal{A} \mathcal{C} \oplus \mathcal{O} \mathcal{F} \oplus \mathcal{E} \mathcal{X} \oplus m \oplus \mathcal{G} \mathcal{R},
$$
where $S$ specifies the role of the LLM, $\mathcal{I D}$ describes the intersection configuration, $\mathcal{A} \mathcal{C}$ specifies the admissible phase patterns, $\mathcal{O F}$ defines the required 16-dimensional binary output format, $\mathcal{E} \mathcal{X}$ provides representative examples, and $\mathcal{G} \mathcal{R}$ imposes sparsity and safety requirements. The resulting priority vector is obtained as
$
\boldsymbol{\theta}^t=\mathcal{L} \mathcal{L} \mathcal{M}\left(\mathcal{P} \mathcal{R}_v\right),
$
and incorporated into the state representation supplied to the actor-critic controller. Thus, the LLM performs semantic priority extraction rather than direct signal control; the final phase decision remains governed by the multi-faced RL policy. 

To ensure operational continuity, SIGMA includes a fallback mechanism for LLM unavailability due to timeout, communication failure, or service interruption. In such cases, $\boldsymbol{\theta}^t=\mathbf{0}$, and the controller reduces to SIGMA-QW, which continues to operate using traffic-pressure and waiting-time information without explicit emergency-priority input. Failure of the semantic module does not disable the underlying traffic controller. More generally, the LLM is not intrinsic to the control architecture: any priority-extraction mechanism capable of producing a valid $\boldsymbol{\theta}^t \in\{0,1\}^{16}$ can replace it. This modular design permits lightweight rule-based or specialised models to be used when computational cost or latency is critical. The reliability implications of the fallback mechanism are analysed in Section~\ref{sec:reliability}.

\begin{figure*}[!t]
    \centering
    \includegraphics[width=\linewidth]{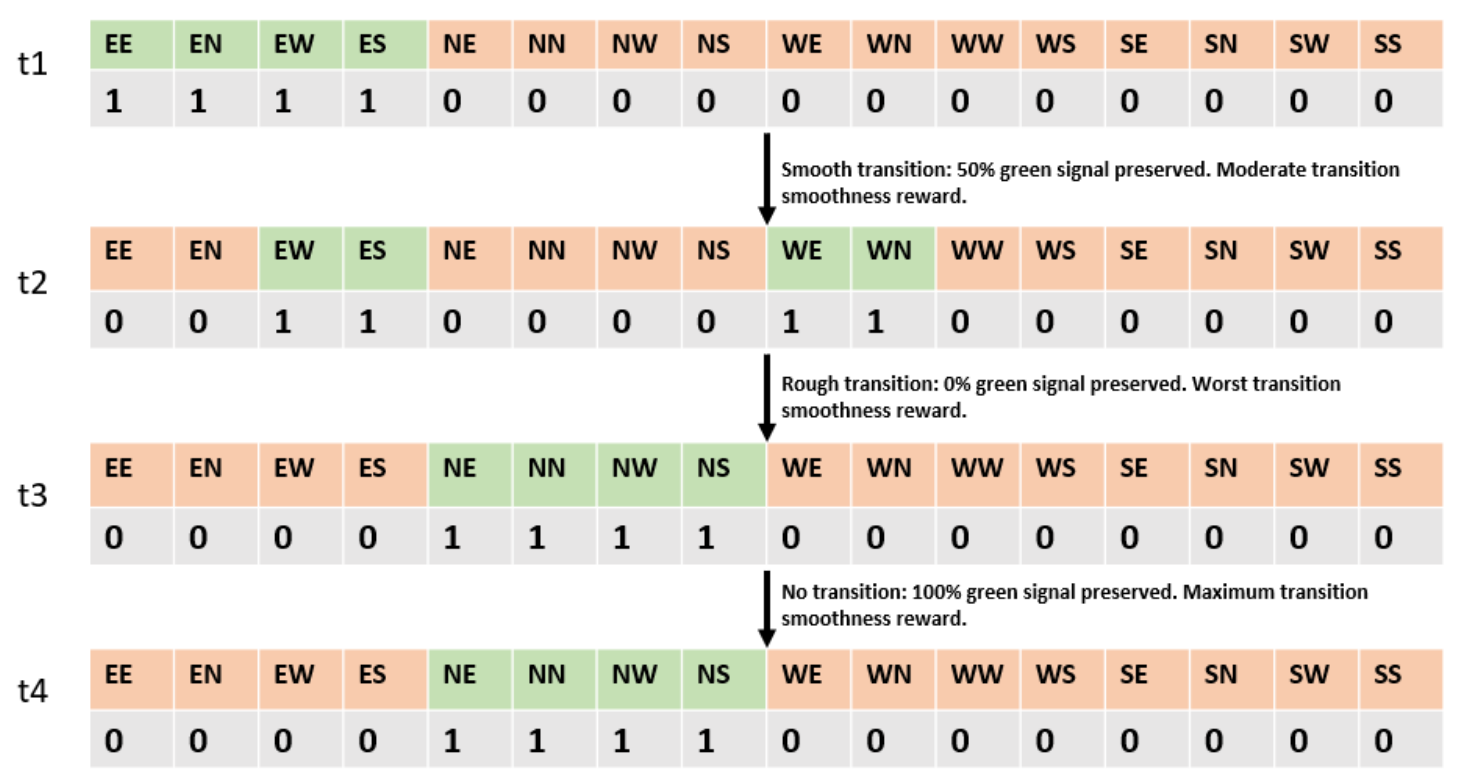}
    \caption{Examples of different types of action smoothness.}
    \label{fig:ac_smooth} 
\end{figure*}

\begin{table}[t]
\centering
\tiny
\caption{Utility functions and their loss and reward components. The actor outputs the distribution $\pi^{(t)}$, and the action is $a^t=\operatorname{softmax}(\pi^{(t)})$.}
\label{tab:lossreward}
\begin{tabularx}{\textwidth}{p{1.8cm} X p{2.5cm} p{1.6cm}}
\toprule
\textbf{Objective} & \textbf{Objective function} & \textbf{Actor loss} & \textbf{Reward} \\
\midrule

Markovian consistency &
$\displaystyle
\mathcal{U}_{\mathrm{M}}^{(t)}
=
\left\|
\boldsymbol{\pi}^{(t)\top}\mathcal{Q}
-
\boldsymbol{\pi}^{(t+1)\top}
\right\|_2^2
$
&
$\displaystyle
\mathcal{L}_{\mathrm{M}}
=
\frac{1}{T-1}
\sum_{t=1}^{T-1}
\mathcal{U}_{\mathrm{M}}^{(t)}
$
&
$\displaystyle
r_{\mathrm{M}}^{(t)}
=
-\mathcal{U}_{\mathrm{M}}^{(t)}
$
\\[2mm]

Action smoothness &
$\displaystyle
\mathcal{U}_{\mathrm{S}}^{(t)}
=
\left\|
\left(
\boldsymbol{\pi}^{(t+1)}
-
\boldsymbol{\pi}^{(t)}
\right)^{\top}
A_{\mathrm{TR}}
\right\|_2^2
$
&
$\displaystyle
\mathcal{L}_{\mathrm{S}}
=
\frac{1}{T-1}
\sum_{t=1}^{T-1}
\mathcal{U}_{\mathrm{S}}^{(t)}
$
&
$\displaystyle
r_{\mathrm{S}}^{(t)}
=
-\mathcal{U}_{\mathrm{S}}^{(t)}
$
\\[2mm]

Queue-pressure reduction &
$\displaystyle
\mathcal{U}_{\mathrm{Q}}^{(t)}
=
\left(
\frac{1}{4}
\boldsymbol{\pi}^{(t)\top}
A_{\mathrm{TR}}\mathbf{M}_{q}^{t}
-
\max_{a\in\mathcal{A}}
\frac{1}{4}
{a}^{\top}A_{\mathrm{TR}}\mathbf{M}_{q}^{t}
\right)^2
$
&
$\displaystyle
\mathcal{L}_{\mathrm{Q}}
=
\frac{1}{T}
\sum_{t=1}^{T}
\mathcal{U}_{\mathrm{Q}}^{(t)}
$
&
$\displaystyle
r_{\mathrm{Q}}^{(t)}
=
-\mathcal{U}_{\mathrm{Q}}^{(t)}
$
\\[2mm]

Waiting-time fairness &
$\displaystyle
\mathcal{U}_{\mathrm{W}}^{(t)}
=
\left(
\frac{1}{4}
\boldsymbol{\pi}^{(t)\top}
A_{\mathrm{TR}}\mathbf{M}_{w}^{t}
-
\max_{a\in\mathcal{A}}
\frac{1}{4}
{a}^{\top}A_{\mathrm{TR}}\mathbf{M}_{w}^{t}
\right)^2
$
&
$\displaystyle
\mathcal{L}_{\mathrm{W}}
=
\frac{1}{T}
\sum_{t=1}^{T}
\mathcal{U}_{\mathrm{W}}^{(t)}
$
&
$\displaystyle
r_{\mathrm{W}}^{(t)}
=
-\mathcal{U}_{\mathrm{W}}^{(t)}
$
\\[2mm]

Emergency alignment &
$\displaystyle
\mathcal{U}_{\mathrm{E}}^{(t)}
=
\left\|
\boldsymbol{\pi}^{(t)\top}A_{\mathrm{TR}}
-
\boldsymbol{\theta}^{t\top}
\right\|_2^2
$
&
$\displaystyle
\mathcal{L}_{\mathrm{E}}
=
\frac{1}{T}
\sum_{t=1}^{T}
\mathcal{U}_{\mathrm{E}}^{(t)}
$
&
$\displaystyle
r_{\mathrm{E}}^{(t)}
=
-\mathcal{U}_{\mathrm{E}}^{(t)}
$
\\

\bottomrule
\end{tabularx}
\end{table}

\subsection{Utility functions}
SIGMA employs five differentiable utility functions that encode complementary operational objectives for offline pretraining and online actor-critic execution, as summarised in Table~\ref{tab:lossreward}. Together, these utilities provide a multi-faced representation of signal-transition consistency, action smoothness, traffic pressure, waiting-time fairness, and emergency responsiveness.

The Markovian-consistency utility $\left(\mathcal{U}_{\mathrm{M}}\right)$ promotes temporal consistency between consecutive signal phases according to a prior transition matrix $\mathcal{Q} \in \mathbb{R}^{8 \times 8}$, thereby discouraging implausible or undesirable phase transitions. The action-smoothness utility $\left(\mathcal{U}_{\mathrm{S}}\right)$ penalises abrupt changes between consecutive action distributions, reducing unnecessary switching and promoting stable and predictable signal operation. An illustrative example is provided in Figure~\ref{fig:ac_smooth}. The queue-pressure utility $\left(\mathcal{U}_{\mathrm{Q}}\right)$ incorporates the max-pressure principle by favouring actions that serve movements with large upstream-downstream pressure imbalances. By accounting for both incoming demand and downstream congestion, this utility promotes efficient queue discharge while limiting spillback. The waiting-time utility $\left(\mathcal{U}_{\mathrm{W}}\right)$ complements pressure-based control by penalizing excessive waiting, thereby reducing the risk that lower-pressure approaches remain underserved and promoting fairer service across competing traffic streams. Finally, the emergency-alignment utility $\left(\mathcal{U}_{\mathrm{E}}\right)$ measures consistency between the selected movements and the emergency-priority vector $\boldsymbol{\theta}^t$, enabling emergency responsiveness without relying on fixed preemption rules.

The five utilities are jointly incorporated into the learning objective through a weighted scalarisation, allowing SIGMA to balance competing operational criteria within a unified differentiable optimisation framework. The corresponding utility definitions, actor losses, and reward specifications are given in Table~\ref{tab:lossreward}.

\subsection{Rotation-based data augmentation}
SIGMA employs rotational data augmentation during offline pretraining to reduce orientation-specific dependence and improve policy transferability across geometrically equivalent intersections. We assume that the four approaches share the same lane configuration, turning permissions, and movement structure, although traffic demand may vary across directions. Without such augmentation, a policy trained predominantly under one directional traffic pattern may overfit to absolute cardinal directions rather than learning relative traffic configurations. To mitigate this effect, each training scenario is transformed under rotations of 0°, 90°, 180°, and 270°.

Let $
\mathcal{D}^{\prime}=\left\{\left({a}^t, \mathbf{p}^{t-1}, \mathbf{M}_q^t, \boldsymbol{\Gamma}^t, \boldsymbol{\theta}^t\right)\right\}_{t=1}^T
$
denote the original offline dataset. Rotations are applied consistently to every component using separate permutation operators: $\rho_a$ for the 8-dimensional action representation, $\rho$ for the 16-dimensional movement-level quantities $\mathbf{p}^{t-1}, \mathbf{M}_{q^{\prime}}^t$ and $\boldsymbol{\theta}^t$, and $\rho^{\prime}$ for the 4-dimensional approach-level vector $\boldsymbol{\Gamma}^t$. The corresponding operators and direction mappings are provided in Appendix~\ref{subsec:rotation-operators}.

The augmented dataset is defined as \(\mathcal{D}_F=\bigcup_{k=0}^3 g^{k \pi / 2}\left(\mathcal{D}^{\prime}\right)\),
where $g^{k \pi / 2}$ applies the corresponding rotation operators jointly to the action and all state components. Because each transformation preserves the relative geometry and movement semantics of the intersection, the augmented data encourage the policy to learn representations that are less dependent on absolute orientation. The procedure increases the effective training set by up to a factor of four and improves robustness across cardinal rotations.

\subsection{Actor and critic pretraining}
Offline pretraining provides informative initialisations for the actor and critic before online interaction with the traffic environment. The actor $\pi_\psi\left(\mathbf{s}^{(t)}\right)$ maps the current state to a probability distribution over the eight admissible signal phases using feed-forward layers with ReLU activations and a softmax output. Its parameters $\psi$ are learned from the rotation-augmented dataset $\mathcal{D}_F$ using the reference actions together with the five objective-specific loss components.

The actor objective is $\mathcal{L}_{\text {actor }}=\mathcal{L}_{\text {en }}+\boldsymbol{\lambda}^{\top} \mathcal{L}_R,$ where
$\boldsymbol{\lambda}=\left(\lambda_M, \lambda_S, \lambda_Q, \lambda_W, \lambda_E\right)^{\top}$ contains the objective weights and
$\mathcal{L}_R=\left(\mathcal{L}_M, \mathcal{L}_S, \mathcal{L}_Q, \mathcal{L}_W, \mathcal{L}_E\right)^{\top}.$ The supervised component is given by the cross-entropy loss,
$$
\mathcal{L}_{\mathrm{en}}=-\frac{1}{\left|\mathcal{D}_F\right|} \sum_{\left(\mathbf{s}^{(t)}, {a}^{(t)}\right) \in \mathcal{D}_F}\left({a}^{(t)}\right)^{\top} \log \pi_\psi\left(\cdot \mid \mathbf{s}^{(t)}\right),
$$
where ${a}^{(t)}$ denotes the one-hot reference phase. The actor parameters are optimised using mini-batches sampled from $\mathcal{D}_F$. The critic $Q_\phi(\mathbf{s}, {a})$ estimates the discounted value associated with a state-action pair and receives the concatenated representation $[\mathbf{s}; {a}]$ as input. Because the offline dataset does not provide ground-truth $Q$ values, the critic is initialised using fitted $Q$ iteration. The scalar reward for the sample $i$ is defined as $r_i=\boldsymbol{\alpha}^{\top} \mathbf{r}_i^V,$
where $\boldsymbol{\alpha}$ contains the reward weights and
$\mathbf{r}_i^V=\left(r_{i, M}, r_{i, S}, r_{i, Q}, r_{i, W}, r_{i, E}\right)^{\top}$ collects the five objective-specific reward components.

Because the offline samples do not form fully continuous trajectories, the successor state associated with $\left(\mathbf{s}_i, {a}_i\right)$ is approximated locally using $K$ nearest neighbours sharing the same action. Specifically,
$$
\hat{\mathbf{s}}_i^{\prime}=\frac{1}{K} \sum_{j \in \mathcal{N}_K\left(\mathbf{s}_i, {a}_i\right)} \mathbf{s}_{j+1},
$$
where $\mathcal{N}_K\left(\mathbf{s}_i, {a}_i\right)$ denotes the set of $K$ nearest state neighbours associated with action ${a}_i$. This provides a local surrogate for the successor state required during offline Bellman updates. Starting from $Q_i^{(0)}=r_i$, fitted $Q$-iteration proceeds according to
$$
Q_i^{(m+1)}=r_i+\gamma \max _{{a}^{\prime} \in \mathcal{A}} Q_{\phi^{(m)}}\left(\hat{\mathbf{s}}_i^{\prime}, {a}^{\prime}\right),
$$
with the critic parameters updated by minimising the squared error between the network predictions and the current fitted targets. Iteration continues until
$$
\left\|Q^{(m+1)}-Q^{(m)}\right\|_2<\epsilon
$$
or a prescribed maximum number of iterations is reached. The complete critic-pretraining procedure is provided in Section~\ref{subsec:critic_pretraining}.

\subsection{Online execution and adaptation}
During online operation, the pretrained actor--critic controller interacts sequentially with the traffic environment. At decision epoch $t$, intersection $v$ observes the local state
$\mathbf{s}_v^{(t)}=\left[\mathbf{p}_v^{t-1}, \mathbf{M}_{v, q}^t, \boldsymbol{\Gamma}_v^t, \boldsymbol{\theta}_v^t\right],$
where the emergency-priority vector $\boldsymbol{\theta}_v^t$ is incorporated when an emergency instruction is available. Given the current state, the actor samples an admissible phase according to
${a}_v^t \sim \pi_\psi\left(\cdot \mid \mathbf{s}_v^{(t)}\right).$

The selected phase is executed in the traffic environment, after which the resulting transition $\left(\mathbf{s}_v^{(t)}, {a}_v^t, r_v^t, \mathbf{s}_v^{(t+1)}\right)$ is stored in the replay buffer $\mathcal{D}_b$. The scalar reward is obtained by aggregating the five objective-specific reward components: $r_v^t=\boldsymbol{\alpha}^{\top} \mathbf{r}_v^{V, t},$ where $\mathbf{r}_v^{V, t}=\left(r_{v, M}^t, r_{v, S}^t, r_{v, Q}^t, r_{v, W}^t, r_{v, E}^t\right)^{\top}$.

Once sufficient transitions have accumulated, mini-batches $\mathcal{B} \subset \mathcal{D}_b$ are sampled for online network updates. For each transition $\left(\mathbf{s}, \mathbf{a}, r, \mathbf{s}^{\prime}\right) \in \mathcal{B}$, a successor action is sampled according to
$\mathbf{a}^{\prime} \sim \pi_\psi\left(\cdot \mid \mathbf{s}^{\prime}\right),$
and the critic target is defined as
$y = r + \gamma Q_{\phi_{\text{target}}}\left(\mathbf{s}^{\prime}, \mathbf{a}^{\prime}\right).$
The critic parameters are updated by minimizing the temporal-difference loss
$$
\mathcal{L}_\phi =
\sum_{\left(\mathbf{s}, \mathbf{a}, r, \mathbf{s}^{\prime}\right) \in \mathcal{B}}
\left[y - Q_\phi(\mathbf{s}, \mathbf{a})\right]^2.
$$

The actor is subsequently refined using a policy-gradient update. Defining the temporal-difference error as
$$
\delta =
r + \gamma Q_\phi\left(\mathbf{s}^{\prime}, \mathbf{a}^{\prime}\right)
- Q_\phi(\mathbf{s}, \mathbf{a}),
$$
the policy-gradient estimator is
$$
\nabla_\psi J =
\sum_{\left(\mathbf{s}, \mathbf{a}, r, \mathbf{s}^{\prime}\right) \in \mathcal{B}}
\delta \nabla_\psi \log \pi_\psi(\mathbf{a} \mid \mathbf{s}).
$$

The critic and actor parameters are updated as
$$
\phi \leftarrow \phi - \eta_Q \nabla_\phi \mathcal{L}_\phi,
\quad
\psi \leftarrow \psi + \eta_\pi \nabla_\psi J,
$$
where $\eta_Q$ and $\eta_\pi$ denote the corresponding learning rates. To stabilize value estimation, the target critic is periodically updated using
$$
\phi_{\text{target}}
\leftarrow
\tau_o \phi +
\left(1-\tau_o\right)\phi_{\text{target}},
$$
where $\tau_o \in (0,1)$ controls the rate at which the target parameters track the online critic. The complete online execution and adaptation procedure is summarized in Algorithm~\ref{subsec:online_execution}.

\section{Structural Properties}
\label{sec:reliability}
We establish several structural properties of SIGMA that characterize its behaviour under changes in operational priorities, geometric transformations, traffic conditions, and component availability. These results provide theoretical support for the reliability and robustness of the proposed control framework. Throughout this section, we impose the following assumptions.

\begin{assumption}[Bounded Rewards]
\label{assum:bounded_rewards}
For each utility $k \in \{M,S,Q,W,E\}$, the reward function $r_k: \mathcal{S} \times \mathcal{A} \rightarrow \mathbb{R}$ is bounded: 
$|r_k(s,a)| \leq R_{\max}$ for all $(s,a) \in \mathcal{S} \times \mathcal{A}$.
\end{assumption}

\begin{assumption}[Regularity of Policy Class]
\label{assum:regularity}
The policy $\pi_\psi$ is differentiable in $\psi$ and satisfies:
\begin{enumerate}
    \item The score function $\nabla_\psi \log \pi_\psi(a|s)$ is bounded in norm by $L_{\text{score}}$.
    \item The policy is Lipschitz in $\psi$ with constant $L_{\text{softmax}}$:
    \[
    \|\pi_{\psi_1}(\cdot|s) - \pi_{\psi_2}(\cdot|s)\|_{\text{TV}} \leq L_{\text{softmax}} \|\psi_1 - \psi_2\|_2
    \]
\end{enumerate}
\end{assumption}

\begin{assumption}[Stationary Convergence]
\label{assum:stationary}
The policy gradient algorithm converges to a stationary point $\psi^*(\alpha)$ satisfying $\nabla_\psi J(\psi^*;\alpha) = 0$ in a neighborhood of the operating point.
\end{assumption}

Assumption~\ref{assum:bounded_rewards} ensures that no individual reward component can grow without bound, while Assumption~\ref{assum:regularity} guarantees that the policy changes smoothly with respect to its parameters and admits stable policy-gradient updates. Assumption~\ref{assum:stationary} ensures that, locally, the policy-gradient procedure converges to a well-defined stationary solution for a given choice of objective weights. These assumptions impose regularity conditions needed for the subsequent theoretical analysis.

\begin{theorem}
\label{thm:graceful_degradation}
Let the total reward be
$
r(s,a;\alpha)=\sum_{k\in\{M,S,Q,W,E\}} \alpha_k r_k(s,a)
$
with $|r_k(s,a)|\leq R_{\max}$ for all $(s,a)$. Let $\alpha\in\mathbb{R}_{>0}^5$ and let $\alpha^{(k)}$ denote the same vector with $\alpha_k=0$. Let $\Pi$ be the set of stationary stochastic policies and define
\[
J(\pi;\alpha)=\mathbb{E}_\pi\left[\sum_{t=0}^\infty \gamma^t r(s_t,a_t;\alpha)\right], \qquad
J^*(\alpha)=\sup_{\pi\in\Pi} J(\pi;\alpha).
\]
Then for every component $k$,
\[
J^*(\alpha)-J^*(\alpha^{(k)})\leq \frac{\alpha_k R_{\max}}{1-\gamma}.
\]
moreover, if $\pi_\alpha^*\in \arg\max_\pi J(\pi;\alpha)$, then
\[
J(\pi_\alpha^*;\alpha^{(k)})\geq J^*(\alpha)-\frac{\alpha_k R_{\max}}{1-\gamma}.
\]
\end{theorem}

Theorem~\ref{thm:graceful_degradation} quantifies the sensitivity of SIGMA's optimal discounted return to the removal of an individual objective component. In particular, deactivating the emergency-alignment reward yields
\[
\left|
J^*(\alpha)-J^*(\alpha^{(E)})
\right|
\leq
\frac{\alpha_E R_{\max}}{1-\gamma}.
\]
Thus, the contribution to performance variation arising solely from
removing the emergency objective is uniformly bounded. When the LLM module is unavailable, SIGMA additionally sets $\theta^t=\mathbf{0}$ and operates using the reduced SIGMA-QW
representation. This introduces an additional information and policy-class restriction that is not covered by the above bound. The practical degradation associated with this complete fallback mechanism is therefore evaluated empirically in the reliability experiments.

\begin{theorem}
\label{thm:c4_equivariance}
Let $G=C_4=\{g_0,g_1,g_2,g_3\}$ act on the state space by
\[
g_k\cdot s=(\rho^k p_{t-1}, \rho^k M_q^t, (\rho')^k \Gamma_t, \rho^k \theta_t),
\]
and on the action space by $g_k\cdot a=\rho_a^k a$, where $\rho,\rho',\rho_a$ are the rotation operators defined in Appendix ~\ref{sec:rotation-framework}. Assume the intersection is homogeneous: all four approaches have identical lane configurations, turning permissions, and traffic distributions. Then the transition kernel and reward satisfy
\[
P(g\cdot s'\mid g\cdot s,g\cdot a)=P(s'\mid s,a), \qquad r(g\cdot s,g\cdot a)=r(s,a) \qquad \forall g\in G.
\]
Let $Q^*$ be the unique fixed point of the Bellman optimality operator. Then $Q^*$ is $G$-invariant:
\[
Q^*(g\cdot s,g\cdot a)=Q^*(s,a) \qquad \forall g\in G,
\]
and the optimal deterministic policy $\pi^*(s)=\arg\max_a Q^*(s,a)$ is $G$-equivariant:
\[
\pi^*(g\cdot s)=g\cdot\pi^*(s) \qquad \forall g\in G.
\]
\end{theorem}

Theorem~\ref{thm:c4_equivariance} shows that, under exact rotational symmetry of the intersection dynamics and reward, the optimal control problem is itself invariant to the choice of cardinal orientation. Accordingly, rotated versions of a traffic scenario correspond to symmetry-equivalent decision problems and should induce correspondingly rotated optimal actions. The rotational augmentation described in Section~\ref{sec:res_met} and Appendix~\ref{sec:rotation-framework},
\[
\mathcal D_F
=
\bigcup_{k=0}^{3}g_{k\pi/2}(\mathcal D'),
\]
therefore, provides the actor with symmetry-consistent training examples across all four orientations. For the unconstrained neural architecture used in SIGMA, such augmentation encourages, but does not mathematically guarantee, exact $G$-equivariance. Exact equivariance would follow if the actor architecture itself were constrained to be $G$-equivariant.

\begin{theorem}
\label{thm:pressure_stability}
Let $q_t=(q_t^e,q_t^n,q_t^w,q_t^s)^\top$ denote the incoming queue pressures at time $t$. The queue dynamics are
\[
q_{t+1}=[q_t+A_t-S(p_t)]^+,
\]
where $A_t$ are arrivals with mean $\nu$ and $S(p)\in\mathbb{R}^4$ is the service vector under phase $p\in\mathcal{A}$. Let the capacity region $\mathcal{C}$ be the set of arrival rates for which there exists a randomised phase-selection rule with mean service strictly dominating $\nu$. If $\nu\in\operatorname{int}(\mathcal{C})$ and the controller selects phases by the max-pressure rule
\[
p_t\in\arg\max_{p\in\mathcal{A}} (A_{\text{TR}}\, a_t)^\top M q_t,
\]
where $a_t\in\{0,1\}^{|\mathcal{A}|}$ is the one-hot encoding of phase $p_t$, then the queues are strongly stable:
\[
\sup_{t\geq 0}\mathbb{E}[\|q_t\|_1]<\infty.
\]
\end{theorem}

Theorem~\ref{thm:pressure_stability} provides the theoretical motivation for incorporating the queue-pressure objective into SIGMA. The utility $\mathcal{U}_{\mathrm{Q}}$ encourages the actor to assign greater preference to phases serving movements with large upstream-downstream pressure imbalances, thereby promoting behaviour consistent with the max-pressure principle. In the limiting case where the learned policy selects the same phase as the max-pressure rule, the queue-stability guarantee of Theorem~\ref{thm:pressure_stability} applies. More generally, it $\mathcal{U}_{\mathrm{Q}}$ acts as a pressure-aware regularising objective that steers the multi-faced controller toward congestion-stabilising decisions while allowing the remaining objectives to influence the final signal action.

\begin{theorem}[Smooth Controllability under Regularity Conditions]
\label{thm:smooth_controllability}
Let $\mathcal{P}=\{\pi_{\boldsymbol{\psi}}(\cdot;\boldsymbol{\alpha}):\boldsymbol{\alpha}\in\mathbb{R}_{>0}^5,\|\boldsymbol{\alpha}\|_1=A\}$ be the set of stationary policies converged under policy gradient with reward $r(\cdot;\boldsymbol{\alpha})$. Assume:
\begin{enumerate}
    \item The critic $Q_{\boldsymbol{\phi}}$ is an unbiased estimator of the true value function.
    \item The policy gradient objective $J(\boldsymbol{\psi};\boldsymbol{\alpha})$ is locally strongly concave in a neighbourhood of a stationary point $\boldsymbol{\psi}^*(\boldsymbol{\alpha})$.
\end{enumerate}
Then the mapping $\Phi:\boldsymbol{\alpha}\mapsto\pi_{\boldsymbol{\psi}}(\cdot;\boldsymbol{\alpha})$ is locally Lipschitz continuous. Moreover, if the per-utility advantages $\{A^{(k)}(s,a)\}_{k=1}^5$ are linearly independent as functions over the state-action space, then there exists a neighbourhood $\mathcal{N}(\boldsymbol{\alpha})$ such that $\Phi$ is locally injective.
\end{theorem}
Theorem~\ref{thm:smooth_controllability} shows that, under the stated regularity conditions, small changes in the objective-weight vector $\boldsymbol{\alpha}$ lead to proportionally small changes in the converged SIGMA policy. Traffic operators can adjust the relative importance of competing objectives, such as queue reduction, waiting-time fairness, and emergency responsiveness, without inducing
arbitrarily abrupt changes in the control policy. Under the additional local-rank condition, distinct admissible perturbations of the objective weights also induce locally distinct policies, ensuring that changes in operational priorities are reflected
in the resulting control behaviour.

\begin{theorem}[Markovian Consistency-Driven Phase Selection]
\label{thm:markovian_selection_1}
Let $\mathcal{Q} \in \mathbb{R}^{|\mathcal{A}| \times |\mathcal{A}|}$ be the expected phase transition matrix, and let $\mathcal{D}_p$ denote the aggregate demand pressure for phase $p \in \{M,S,Q,W,E\}$. Define the weight vector
\[
\mathbf{w} = \left(\frac{\boldsymbol{\alpha}}{\mathbf{1}^\top \boldsymbol{\alpha}}\right)
\quad \text{with} \quad \sum_{k \in \{M,S,Q,W,E\}} w_k = 1,
\]
and specifically set
\[
w_M = \frac{1}{2} - \epsilon, \quad
w_S = \frac{1}{2} - \epsilon, \quad
w_W = 2\epsilon, \quad
w_E = w_Q = 0,
\]
where $0 < \epsilon \ll 1$. As $\epsilon \to 0^+$, the control objective emphasises consistency and smoothness of phase transitions. Then the optimal phase satisfies:
\[
p^{t+1} = 
\begin{cases} 
p^t, & \text{if } \mathcal{D}^\dagger_{\text{current}} \geq \mathcal{D}^\dagger_{\text{next}}, \\[6pt]
\operatorname{next}(p^t), & \text{if } \mathcal{D}^\dagger_{\text{next}} > \mathcal{D}^\dagger_{\text{current}},
\end{cases}
\]
where $\operatorname{next}(p^t)$ denotes the successor phase in the cyclic sequence, and $p^t \in \mathcal{A}$ is the selected phase at time $t$.
\end{theorem}

Theorem~\ref{thm:markovian_selection_1} characterizes the behaviour of SIGMA in the consistency-dominated operating regime. When the Markovian-consistency and action-smoothness objectives dominate, the controller restricts its decisions to temporally consistent phase transitions, while the smaller waiting-time component determines whether the current phase should be retained or the cyclic successor should be activated. The resulting controller exhibits predictable phase progression while retaining limited demand responsiveness through the waiting-time objective.

The proofs of all the theoretical results are deferred to Appendix~\ref{proof}.

\section{Experimental Analysis}
\label{sec:data_gen}

The experimental evaluation is conducted in two complementary stages to assess both the complete functionality of SIGMA and its applicability to conventional traffic-signal-control settings. First, the full SIGMA framework, including its emergency-aware decision mechanism, is evaluated under a controlled simulation environment with traffic demand calibrated from empirical observations. This stage examines the complete behaviour of SIGMA using synthetic traffic realisations generated from the calibrated demand model and evaluates its performance in terms of emergency responsiveness, traffic efficiency, queueing behaviour, and phase-transition characteristics. Second, to enable a fair comparison with conventional benchmark controllers that do not incorporate explicit emergency-priority information, an emergency-free variant of SIGMA, denoted as SIGMA-QW, is evaluated using real-world traffic data from four signalised intersections in Kolkata. In this second stage, the emergency-specific priority mechanism is disabled, ensuring that SIGMA-QW and the benchmark controllers operate using comparable traffic-state information and action spaces.

\subsection{Synthetic Traffic Demand Generation}
\label{subsec:datagensim}

Traffic demand is synthesised over the simulation horizon $T$ using non-homogeneous Poisson processes with harmonically varying intensities. For each incoming approach $i$, the arrival intensity is defined as
\[
\nu_i(t)=\nu_0+\sum_{r=1}^{R}k_r\sin(\omega_r t+\phi_r),
\]
with the resulting rate truncated at zero to ensure non-negativity. Identical intensity parameters are used across approaches to isolate the effects of the control policies. Departures are generated analogously using time-varying service rates $\mu_j(t)$, with departure times sampled according to an exponential distribution with rate $\mu_j(t)$. At each decision interval $t$, the number of arrivals is sampled as $N_i(t)\sim\operatorname{Poisson}(\nu_i(t)\Delta t)$.

The resulting per-approach incoming and outgoing cumulative counts, $\mathbf{L}_{\mathrm{in}}^t,\mathbf{L}_{\mathrm{out}}^t\in\mathbb{R}^{4}$, are recorded at each decision step. The baseline traffic state comprises the net pressure mask $\mathbf{M}_q^t$, obtained from the difference between incoming and outgoing flow and shifted to ensure non-negative entries, the vector of maximum waiting times $\boldsymbol{\Gamma}^t$, and the previous phase representation $\mathbf{p}^{t-1}$. Emergency events are introduced as rare episodic scenarios drawn from a predefined library specifying the vehicle type, origin, destination, and departure time. Each incident is subsequently encoded into a priority vector $\boldsymbol{\theta}^t$ through LLM-based interpretation of its natural-language description. The complete observation is therefore given by $\mathbf{s}^t=[\mathbf{p}^{t-1},\mathbf{M}_q^t,\boldsymbol{\Gamma}^t,\boldsymbol{\theta}^t]$.

For each decision step, the feasible actions are evaluated according to an instantaneous reward $r^t(a)$ that incorporates penalties associated with emergency delay, traffic congestion, and excessive phase switching. The action maximizing the instantaneous reward, $a^{t*}=\arg\max_{a\in\mathcal{A}}r^t(a)$, is selected and the resulting state--action information is recorded. Repeating this procedure over the simulation horizon produces the offline dataset $\mathcal{D'}$ used for actor-critic pretraining. Appendix~\ref{subsec:data_generation} provides the complete data-generation procedure.

\subsection{Baseline methods} \label{sec:baselines}

We evaluate SIGMA by comparing it with three baseline controllers representing fixed-time, traffic-responsive, and deep reinforcement learning. All baselines run under the same traffic environment and permissible action space as SIGMA, and we match the reward structure and state representation where applicable to ensure a fair comparison.

\paragraph{Fixed-Time Control:}
Fixed-time control uses a pre-computed cyclic signal schedule, where the $K$ phases are run in a fixed order with prescribed green durations $d_1,\ldots,d_K$. The appropriate cycle length is $C=\sum_{k=1}^{K}d_k$. Because the phase sequencing and durations are established a priori, the controller does not respond to real-time variations in traffic demand or emergency conditions. We provide the full implementation and parameterization in Appendix~\ref{subsec:fixed_time}.

\paragraph{Actuated Control:}
Actuated control dynamically chooses when to end a phase based on real-time detection of vehicles while enforcing minimum and maximum green-time limitations. The phase-extension mechanism is governed by three parameters: minimum green time $G_{\min}$, ensuring minimum service time; maximum green time $G_{\max}$, limiting long service of a single phase; and gap threshold $G_{\mathrm{gap}}$, detecting a drop in demand. For a phase $p$ that starts at time $t_0$, the end time is \[ t_1= \min\left( t_0+G_{\max}, \inf\left\{ t\geq t_0+G_{\min}: \max_i t_l^{(i)}(t)>G_{\mathrm{gap}} \right\} \right), \] where $t_l^{(i)}(t)$ is the time passed since the last detection of a vehicle on approach $i$. The next phase is chosen as the one that maximizes the total net pressure after it ends, \[ p_{\mathrm{next}} = \arg\max_{p'} \sum_{i\in S(p')} \sum_{j\in T(p')} M_{ij}^{q}(t_1), \] where $S(p')$ and $T(p')$ denote the corresponding incoming and outgoing approaches served by phase $p'$. The full control process is described in Appendix~\ref{subsec:actuated_control}.

\paragraph{Deep Q-Network (DQN):}
DQN~\cite{b35} offers a deep reinforcement learning baseline by learning the state-action value function $Q_{\phi}(\mathbf{s},{a})$ using experience replay and a target network. For the sake of comparability with SIGMA, DQN adopts the same state representation $\mathbf{s}^{t}=[\mathbf{p}^{t-1},\mathbf{M}_{q}^{t},\boldsymbol{\Gamma}^{t},\boldsymbol{\theta}^{t}]$, the same admissible action space and the utility-based reward structure. An $\epsilon$-greedy policy selects actions. More specifically, the greedy action $\arg\max_{a\in\mathcal{A}}Q_{\phi}(\mathbf{s}^{t},a)$ is selected with probability $1-\epsilon+\epsilon/|\mathcal{A}|$ and the exploration part takes the probability of $\epsilon/|\mathcal{A}|$ on each action. After performing $a^{t*}$, the transition is employed to generate the temporal-difference target \[ y=r^{t}(a^{t*}) +\max_{a'\in\mathcal{A}} Q_{\phi^{-}}(\mathbf{s}^{t+1},a'), \] where $\phi^{-}$ are the parameters of target-network. The DQN parameters are trained by minimizing \[ \mathcal{L}_{\mathrm{DQN}}= \left( y-Q_{\phi}(\mathbf{s}^{t},a^{t*}) \right)^2, \]
utilizing gradient-based optimisation, and the goal network is synchronized with the online network periodically. The reward is built with constant utility weights, allowing a consistent comparison to SIGMA without dynamic objective-weight adaptation. The whole DQN technique is described in Appendix~\ref{subsec:dqn}.

\begin{table*}[!t]
\centering
\tiny
\caption{Evaluation Metrics for Traffic Signal Control}
\label{tab:metrics_definition}
\begin{tabularx}{\textwidth}{p{1.2cm} X p{5.8cm}}
\toprule
\textbf{Metric} & \textbf{Significance} & \textbf{Formulae} \\
\midrule

AEWT$_{\Delta t}$ & Emergency vehicle waiting time averaged over each evaluation interval $\Delta t$. This metric is critical for assessing how well emergency vehicles are prioritised. & 
$\displaystyle
\text{AEWT}_{\Delta t} = \frac{1}{k} \sum_{i=1}^{k} \bigl( t_{\text{em(dep),i}} - t_{\text{em(arr),i}} \bigr)
$ \\
\midrule

AEWT & Overall emergency waiting time averaged over the full simulation period $T$. This is the primary indicator of emergency response performance. & 
$\displaystyle
\text{AEWT} = \frac{1}{T/\Delta t} \sum \text{AEWT}_{\Delta t}
$ \\
\midrule

AWT$_{\Delta t}$ & Waiting time of all vehicles (regular and emergency) averaged over each interval $\Delta t$. This reflects general user delay. & 
$\displaystyle
\text{AWT}_{\Delta t} = \frac{1}{k} \sum_{i=1}^{k} \bigl( t_{\text{dep},i} - t_{\text{arr},i} \bigr)
$ \\
\midrule

AWT & Overall waiting time of all vehicles averaged over the entire evaluation period. This is a key measure of average traffic efficiency. & 
$\displaystyle
\text{AWT} = \frac{1}{T/\Delta t} \sum \text{AWT}_{\Delta t}
$ \\
\midrule

AMWT$_{\Delta t}$ & Per-direction maximum waiting times averaged over the four approaches within interval $\Delta t$. This captures the worst-case delay on any approach. & 
$\displaystyle
\text{AMWT}_{\Delta t} = \frac{1}{4} \sum_{j=1}^{4} \left( \max_{i}^{k_j} \bigl( t_{\text{dep},i,j} - t_{\text{arr},i,j} \bigr) \right)
$ \\
\midrule

PC$_t$ & Magnitude of phase changes between consecutive decision steps. Lower values indicate smoother transitions, which reduce driver confusion and jerkiness. & 
$\displaystyle
\text{PC}_t = \frac{\| a_t \cdot A_{\text{TR}} - a_{t+1} \cdot A_{\text{TR}} \|_2^2}{8}
$ \\
\midrule

APC & Average Phase Change: Mean of PC across all decision steps. Overall indicator of transition smoothness. & 
$\displaystyle
\text{APC} = \frac{1}{T/\Delta t} \sum_t \text{PC}_t
$ \\
\midrule

TC$_t$ & How closely the selected action follows the expected transition matrix $Q_p$ at each decision point. Higher consistency improves predictability. & 
$\displaystyle
\text{TC}_t = \frac{\bigl| \| a_t \cdot Q_p - a_{t+1} \|_2^2 - 2 \bigr|}{2}
$ \\
\midrule

TC & Average transition consistency across all decision steps. This is an overall indicator of phase sequence stability. & 
$\displaystyle
\text{TC} = \frac{1}{T/\Delta t} \sum_t \text{TC}_t
$ \\
\midrule

ATP & Average number of vehicles passing through the intersection per time interval $\Delta t$. This is a direct measure of throughput and capacity utilization. & 
$\displaystyle
\text{ATP} = \frac{1}{T/\Delta t} \sum_{\Delta t} \bigl( \text{Number of vehicles released in } \Delta t \bigr)
$ \\

\bottomrule
\end{tabularx}
\end{table*}

\begin{table*}[t]
\centering
\tiny
\caption{Performance comparison of traffic signal control methods over 1000 episodes (After convergence of DQN and SIGMA from 2001-3000 th episode) (mean $\pm$ standard deviation)}
\label{tab:comp_performance}
\begin{tabularx}{\textwidth}{p{0.8cm} X X X X X X X}
\toprule
\textbf{Method} & \textbf{AMWT (s) $\downarrow$} & \textbf{AWT (s) $\downarrow$} & \textbf{AEWT (s) $\downarrow$} & \textbf{AQL $\downarrow$} & \textbf{APC $\downarrow$} & \textbf{TC $\uparrow$} & \textbf{ATP $\uparrow$} \\
\midrule

Fixed Time Control & 
$231.47 \pm 27.8$ & 
$108.23 \pm 13.0$ & 
$91.15 \pm 10.9$ & 
$18.0 \pm 1.8$ & 
$0.50 \pm 0.00$ & 
$\mathbf{1.00 \pm 0.00}$ & 
$4.23 \pm 0.48$ \\
\midrule

Actuated Control & 
$187.32 \pm 22.5$ & 
$99.61 \pm 11.9$ & 
$87.44 \pm 10.5$ & 
$16.0 \pm 1.6$ & 
$\mathbf{0.17 \pm 0.02}$ & 
$0.21 \pm 0.03$ & 
$6.31 \pm 0.72$ \\
\midrule

DQN & 
$\mathbf{121.78 \pm 14.6}$ & 
$86.29 \pm 10.4$ & 
$29.55 \pm 3.5$ & 
$13.0 \pm 1.3$ & 
$0.38 \pm 0.05$ & 
$0.32 \pm 0.04$ & 
$8.07 \pm 0.10$ \\
\midrule

SIGMA (Proposed) & 
$133.27 \pm 16.3$ &
$\mathbf{77.54 \pm 9.3}$ & 
$\mathbf{23.47 \pm 2.8}$ & 
$\mathbf{11.38 \pm 1.4}$ & 
$0.47 \pm 0.05$ & 
$0.39 \pm 0.05$ & 
$\mathbf{9.26} \pm \mathbf{1.08}$ \\
\bottomrule
\end{tabularx}
\end{table*}

\begin{figure*}[!ht]
\centering
\begin{minipage}[b]{0.32\textwidth}
\centering
\includegraphics[width=\textwidth]{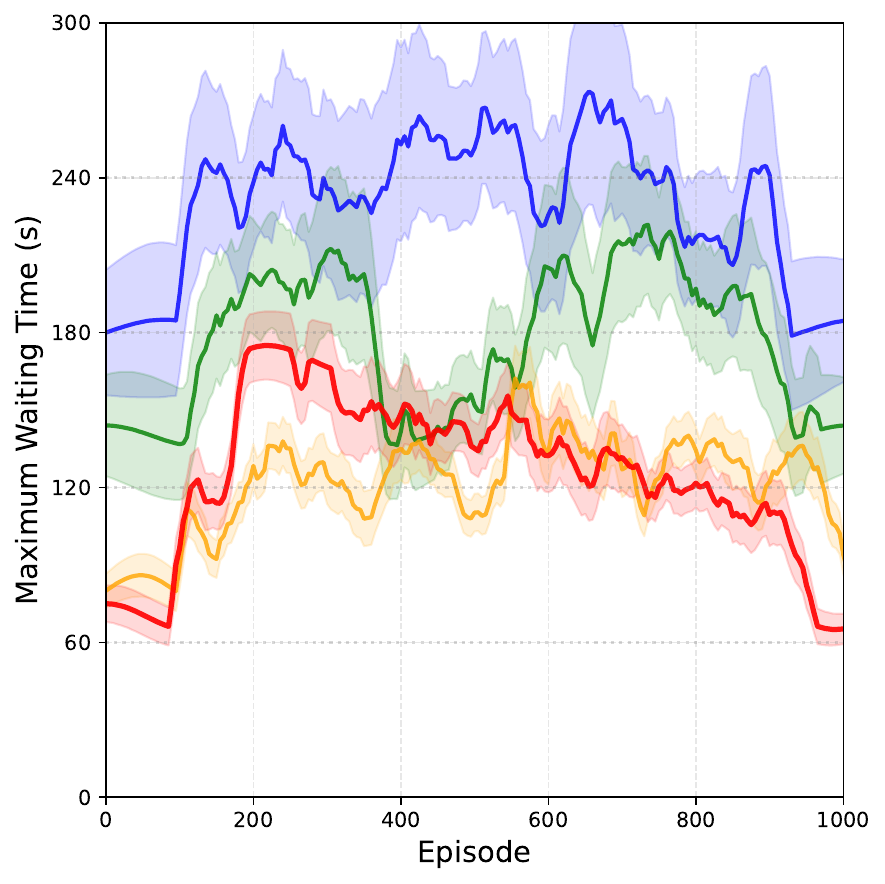}
\caption*{Maximum Waiting Time}
\end{minipage}
\hfill
\begin{minipage}[b]{0.32\textwidth}
\centering
\includegraphics[width=\textwidth]{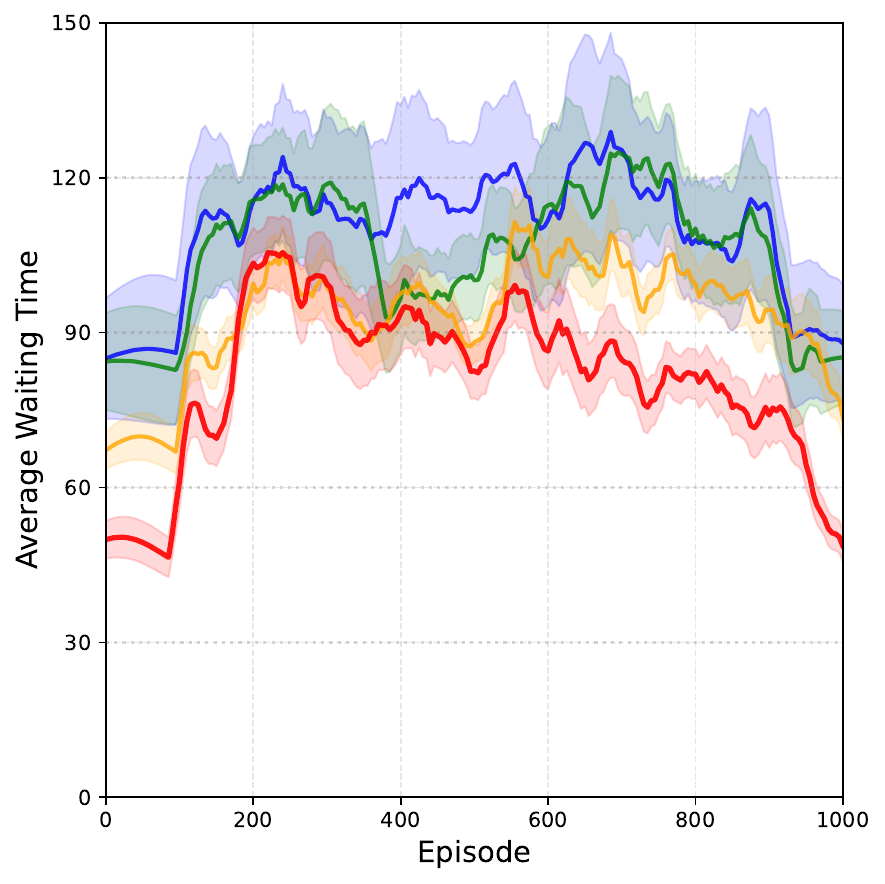}
\caption*{Average Waiting Time}
\end{minipage}
\hfill
\begin{minipage}[b]{0.32\textwidth}
\centering
\includegraphics[width=\textwidth]{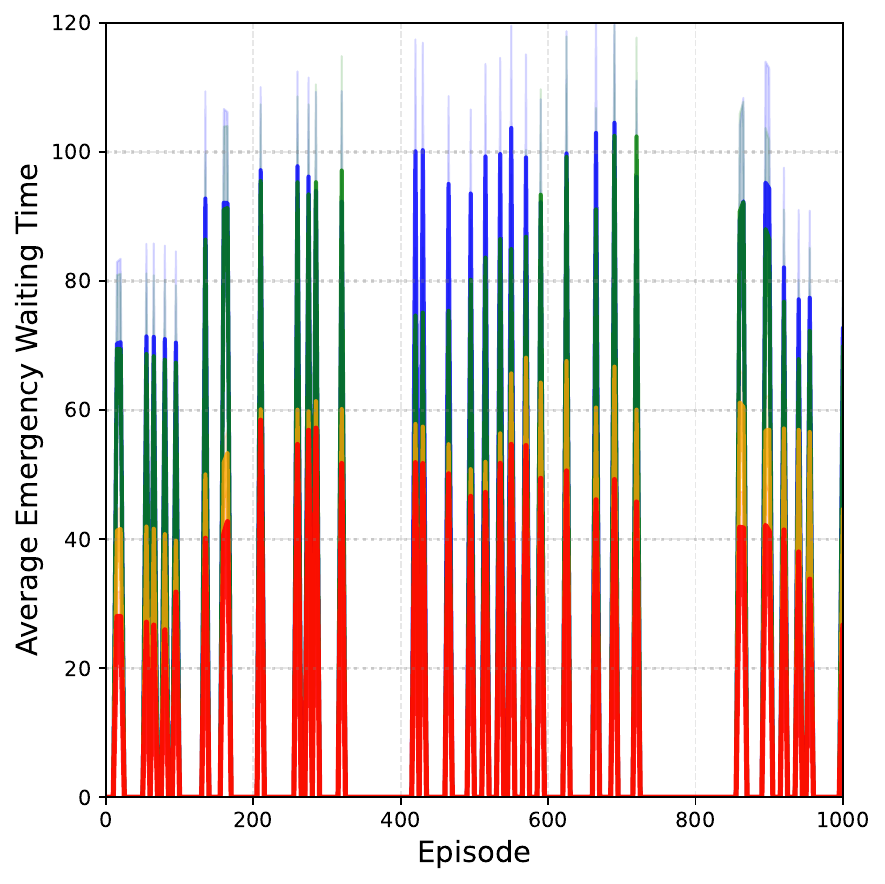}
\caption*{Emergency-Waiting Time}
\end{minipage}

\begin{minipage}[b]{0.32\textwidth}
\centering
\includegraphics[width=\textwidth]{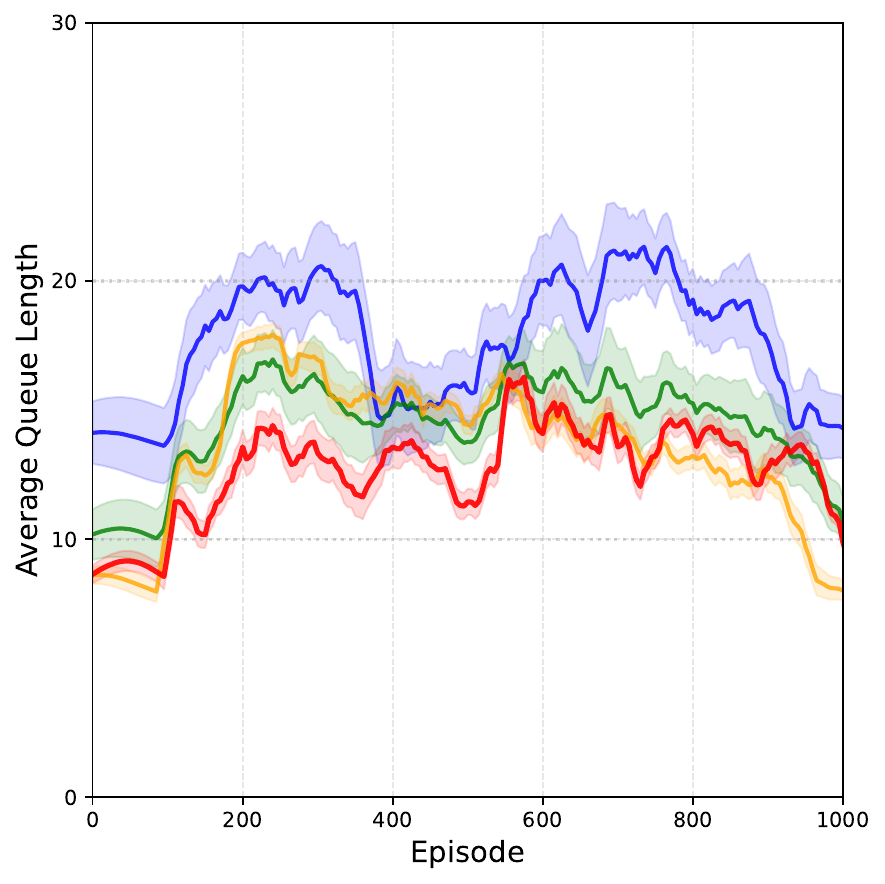}
\caption*{Average Queue Length}
\end{minipage}
\hfill
\begin{minipage}[b]{0.32\textwidth}
\centering
\includegraphics[width=\textwidth]{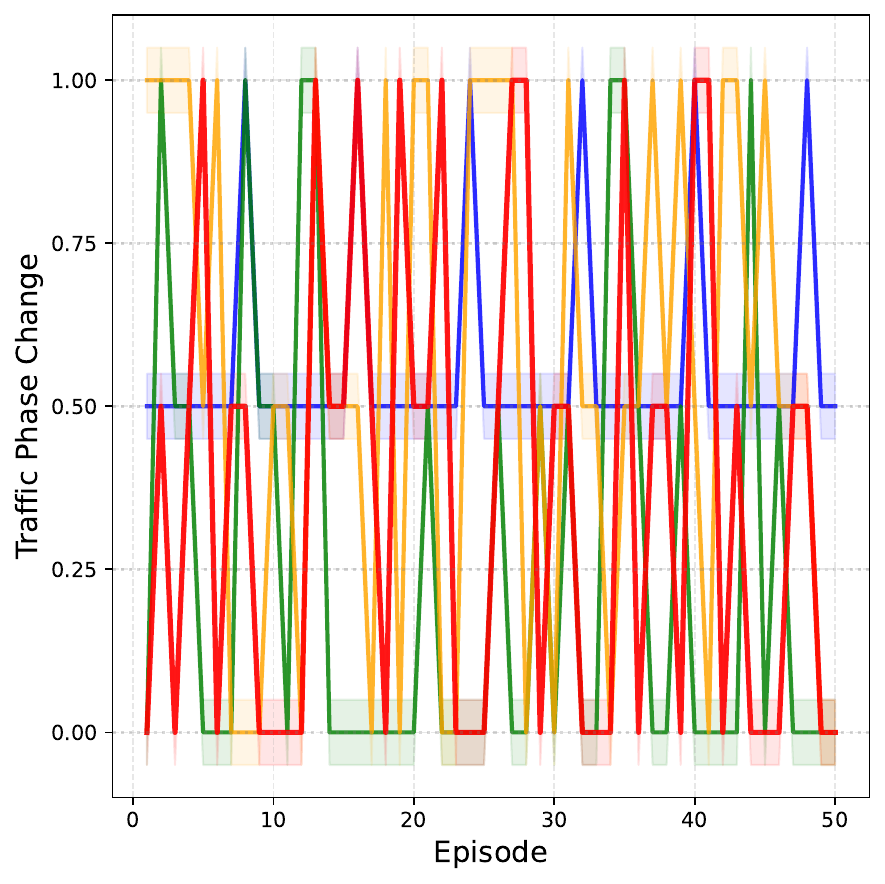}
\caption*{Traffic Phase Change}
\end{minipage}
\hfill
\begin{minipage}[b]{0.32\textwidth}
\centering
\includegraphics[width=\textwidth]{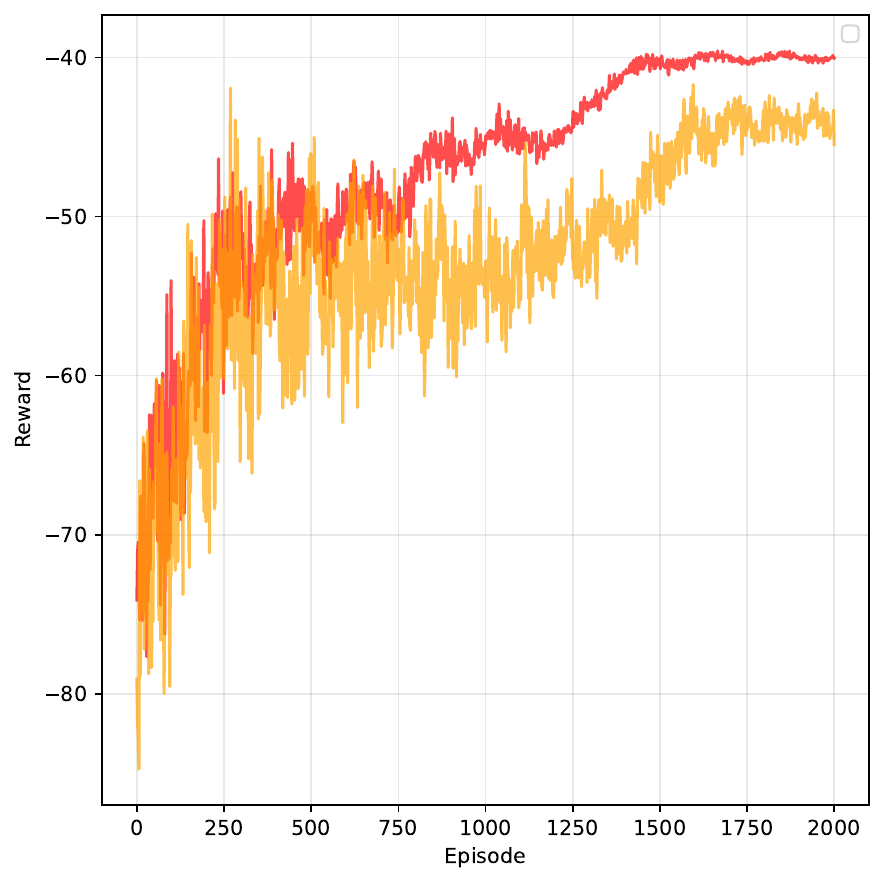}
\caption*{Reward}
\end{minipage}

\begin{minipage}{\textwidth}
\centering
\includegraphics[width=1.0\textwidth]{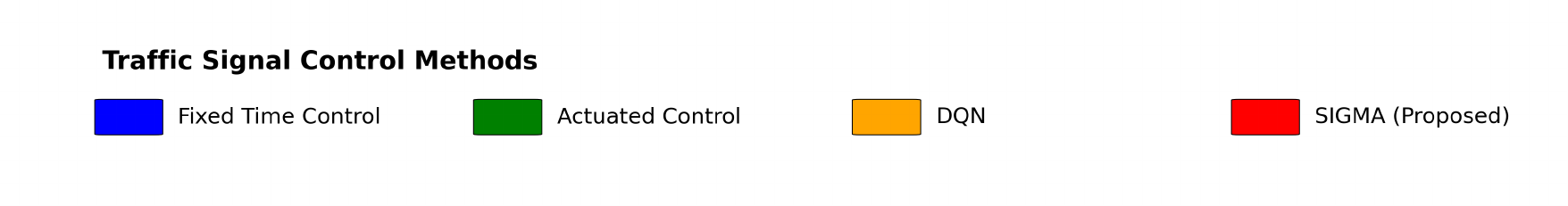}
\end{minipage}

\caption{Performance comparison across six metrics. Proposed method (red) shows improved emergency handling with minimal impact on general traffic. (a)--(d) are evaluated after 2,000 episodes, (e) over Episodes 2001--2050, and (f) shows convergence over the first 2,000 episodes comparing SIGMA and DQN.}
\label{fig:metrics}
\end{figure*}

\subsection{Experimental Setup and Evaluation Metrics} \label{subsec:experimental_setup}

The suggested control strategy is assessed on five aspects, namely, emergency-vehicle prioritisation, average delay, worst-case per-approach time, phase-transition smoothness, and traffic throughput. We present formal definitions of the evaluation metrics in Table~\ref{tab:metrics_definition}, which are interval-aggregated measures over $\Delta t$ and full-horizon measures over the simulation horizon $T$, respectively. The performance of SIGMA is compared with three baseline controllers: fixed-time, actuated, and DQN-based control. The comparative results are provided in Table~\ref{tab:comp_performance}.

The evaluation is designed to examine the ability of each controller to satisfy simultaneously conflicting operational objectives. The results show a clear separation between the control paradigms considered. Controllers based on fixed or relatively constrained decision criteria show larger performance trade-offs in the presence of multiple, competing objectives. Conversely, SIGMA maintains a more balanced performance profile using its multi-faced, utility-based formulation.

\subsection{Performance Comparison}\label{subsec:performance_comparison}
We next compare SIGMA with the baseline controllers across three key operational dimensions: emergency responsiveness, overall traffic efficiency, and temporal consistency of signal transitions.
\begin{itemize}
    \item Emergency Vehicle prioritisation: SIGMA achieves the lowest emergency waiting time among all examined controllers, with a relatively low variance. This is because explicitly including the LLM-derived priority representation $\boldsymbol{\theta}^t$ in the state representation gives the policy information to distinguish emergency movements from normal traffic demand. The emergency-related information thus directly influences phase selection. By contrast, DQN does not perform as well as other baselines. This limitation is especially relevant when emergency movements contend with heavily congested approaches. Neither fixed-time nor actuated controllers includes an explicit emergency-prioritising mechanism. Consequently, their pre-programmed or detector-based decision rules allow little flexibility for preferred treatment of time-critical movements.
    \item General Traffic Efficiency: SIGMA achieves the best overall performance in terms of average delay, queue length, and throughput, which demonstrates that the application of emergency priority does not necessarily lead to a significant reduction in the efficiency of ordinary traffic. This is due to the multi-faced utility formulation, where emergency response is optimized alongside congestion-, waiting-time-, and control-related objectives. DQN achieves competitive performance in terms of maximum waiting time, but it does not explicitly enforce a simultaneous balance between competing operational objectives due to its value-based decision methodology. Thus, actions that favor extremely congested approaches may result in less fair service across approaches, or may become suboptimal when emergency priorities are added. Fixed-time and actuated control have relatively steady operations but limited reactivity to temporal variations in demand, since their decision methods do not fully leverage the observed multi-dimensional traffic state.
    \item Smoothness of Phase Transitions and Consistency of Control: Fixed-time control has the greatest structural consistency in phase transitions, since the phase sequence and phase durations are determined in advance. Such constancy, however, is not matched by similar flexibility to real-time traffic situations. While DQN allows more flexible actions, there is no explicit regularisation on the consistency or smoothness of phase transitions, which may lead to less temporally coherent phase sequences over dynamically shifting states. We incorporate the Markovian consistency and smoothness utility components, $\mathcal{U}_M$ and $\mathcal{U}_S$, respectively, in SIGMA, which directly model the temporal structure of phase selection. These components provide consistent transitions while maintaining the policy's ability to adjust to changing traffic conditions. The resulting behaviour balances the structural regularity of fixed-time control with the flexibility required for responsive traffic management.
\end{itemize}

The overall comparative results clearly indicate the limitations of the baseline paradigms considered. SIGMA has explicit multi-faced and emergency-aware mechanisms. Fixed-time and actuated controllers are limited in their capacity to react to dynamically fluctuating demand and do not clearly incorporate the prioritisation of emergency vehicles. Conversely, SIGMA models phase selection under a single integrated utility-based framework, which considers emergency prioritisation, Markovian consistency, transition smoothness, queue reduction, and waiting time reduction simultaneously. This approach allows the optimisation of competing operational objectives within a shared decision framework, rather than through individually imposed control rules.

The relative performance on individual evaluation metrics still depends on the reward-weight configuration choice $\boldsymbol{\alpha}$. The sensitivity analysis in Section~\ref{app:sensitivity_full} shows that increasing the weight of a given objective can improve the associated performance measure, but at the expense of others. Specifically, for some weight configurations, gains in average maximum waiting time (AMWT) may be associated with relatively less desirable performance on other system-level objectives. However, SIGMA exhibits high performance across most evaluation dimensions across the studied configurations, which supports the resilience and flexibility of the proposed multi-faced control formulation.

\subsection{Real-World Traffic Signal Control Validation in SUMO}

To assess the practical effectiveness and generalisability of the proposed control framework, we conduct a simulation-based validation using four signalised intersections in Kolkata, India, reconstructed in the Simulation of Urban Mobility (SUMO) environment~\cite{b61}. The selected intersections represent heterogeneous urban traffic configurations and operating conditions, including differences in intersection geometry, approach characteristics, turning movements, and traffic-demand patterns. The corresponding OpenStreetMap (OSM) representations are shown in Figure~\ref{fig:kolkata-maps}.

\begin{figure*}[!ht]
\centering
\begin{minipage}[b]{0.22\textwidth}
\centering
\includegraphics[height=3.9cm, keepaspectratio]{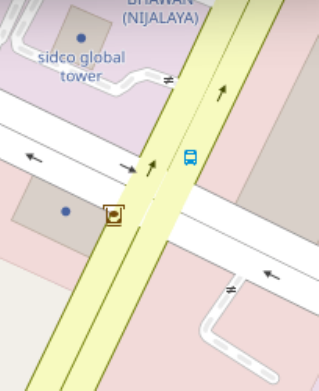}
\caption*{KOL-1: Saltlake CC}
\end{minipage}
\hfill
\begin{minipage}[b]{0.23\textwidth}
\centering
\includegraphics[height=3.9cm, keepaspectratio]{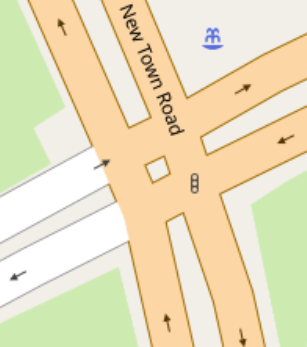}
\caption*{KOL-2: Newtown Gate 1}
\end{minipage}
\hfill
\begin{minipage}[b]{0.22\textwidth}
\centering
\includegraphics[height=3.9cm, keepaspectratio]{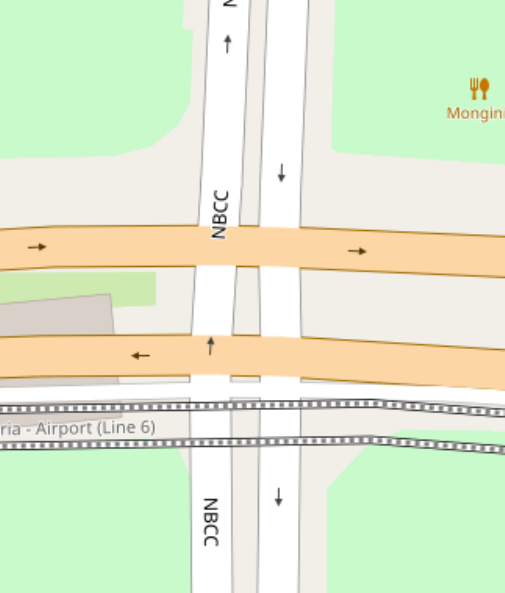}
\caption*{KOL-3: Newtown BS}
\end{minipage}
\hfill
\begin{minipage}[b]{0.24\textwidth}
\centering
\includegraphics[height=3.9cm, keepaspectratio]{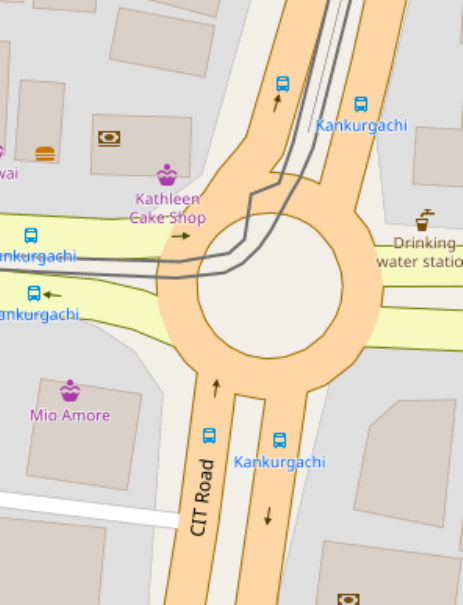}
\caption*{KOL-4: Kakurgachi (Central Kolkata)}
\end{minipage}
\caption{OSM-based representations of the four Kolkata intersections used for SUMO validation.}
\label{fig:kolkata-maps}
\end{figure*}

\begin{figure*}[!h]
\centering
\includegraphics[width=\linewidth]{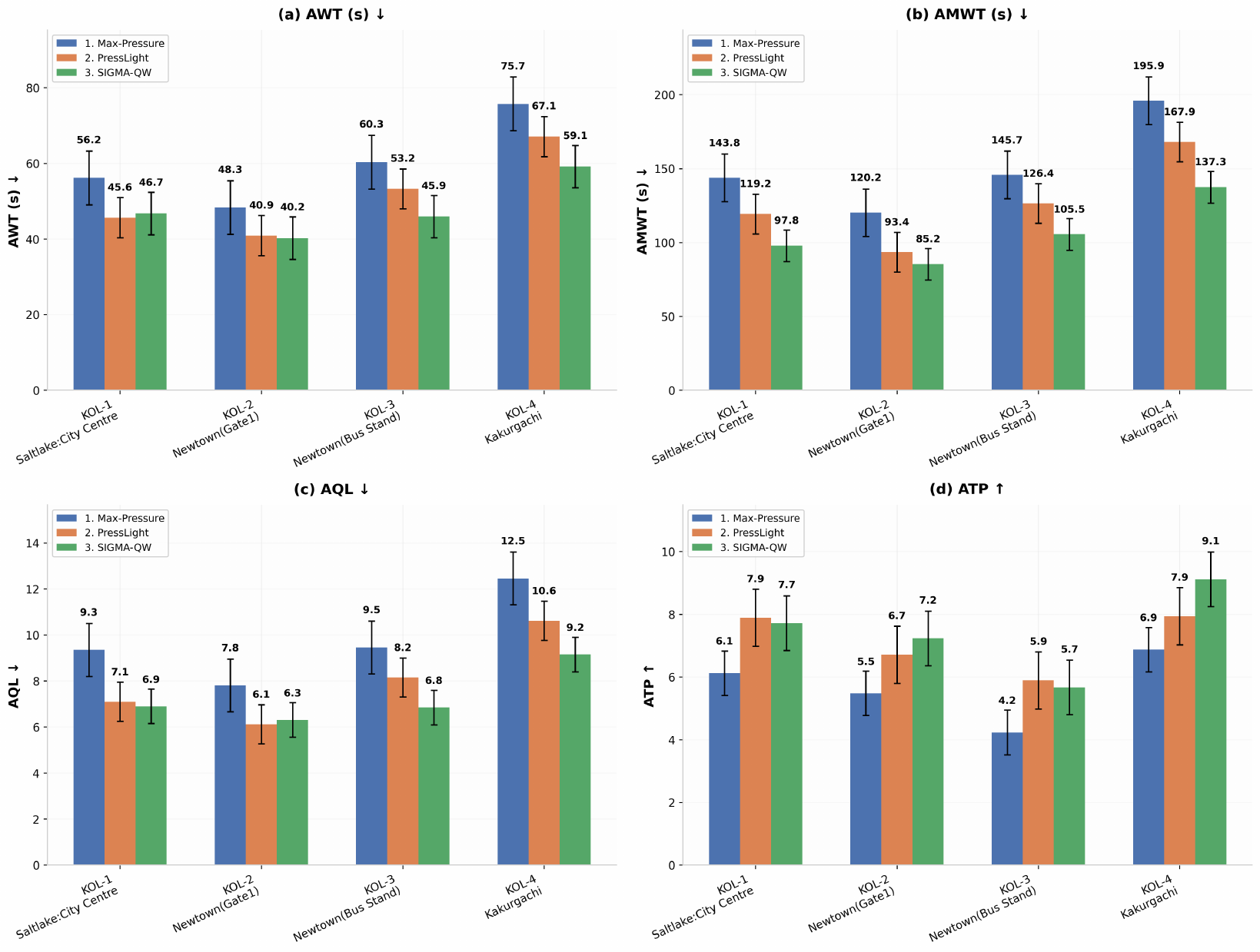}
\caption{Comparative SUMO performance of the evaluated signal-control methods across the four Kolkata intersections.}
\label{fig:kol_res}
\end{figure*}

The intersection networks are generated from OpenStreetMap data and subsequently validated to ensure consistency with the observed road geometry and signalised movements. Traffic demand is calibrated using representative municipal traffic counts and is generated using peak-period non-homogeneous Poisson processes to reproduce temporal variations in traffic intensity. Each simulation episode spans 3600 s, including a 300 s warm-up period, with control decisions updated at 30 s intervals. This protocol provides a common simulation environment in which all competing controllers are evaluated under identical traffic-demand realizations and network conditions.

To establish a fair comparison with conventional traffic-signal-control benchmarks, the emergency-prioritisation component of SIGMA is deliberately deactivated in this experiment. Specifically, the LLM-based emergency representation $\boldsymbol{\theta}^t$ is removed from the controller state, and no emergency-priority intervention is introduced during the evaluation. The resulting variant, denoted as SIGMA-QW, operates exclusively on the traffic-state information used by conventional non-emergency controllers. Its state is therefore given by
$
\mathbf{s}_{\mathrm{QW}}^{(t)}
=
[\mathbf{p}^{t-1},\mathbf{M}_q^t,\boldsymbol{\Gamma}^t]
\in\mathbb{R}^{36}.
$

This experimental design isolates the contribution of the underlying adaptive multi-faced control mechanism and prevents the emergency-specific representation from providing SIGMA with information unavailable to the standard benchmarks. For consistency across methods, the 16-dimensional pressure representation in SIGMA-QW is constructed deterministically from the same incoming and outgoing traffic information available to the baseline controllers. The maximum waiting-time vector $\boldsymbol{\Gamma}^t$ is retained only because it is required by the waiting-time utility component $\mathcal{U}_W$; it does not provide SIGMA-QW with additional information unavailable for constructing the corresponding benchmark state. Thus, the comparison evaluates the controllers primarily on their ability to adapt signal phases to ordinary traffic conditions rather than on their ability to prioritise emergency vehicles.

The benchmark set comprises fixed-time, actuated, and learning-based control methods, together with SIGMA-QW. All methods are evaluated using identical network configurations, traffic-demand realisations, simulation horizons, and control intervals. The resulting performance across the four intersections is summarised in Figure~\ref{fig:kol_res}. The complete simulation configuration, network construction procedure, demand-generation parameters, and evaluation protocol are provided in Appendix~\ref{app:sumo}.

\section{Ablation Study and Sensitivity Analysis}
\label{sec:ablation_sensitivity}
This section evaluates the contribution of the components of the proposed SIGMA framework and examines the sensitivity of its performance to alternative objective-weight configurations. The analysis is conducted from four complementary perspectives. First, a component-wise ablation study assesses the contribution of the individual objective terms in the multi-faced formulation. Second, the effect of rotation augmentation is examined to evaluate its contribution to policy robustness under directional transformations. Third, a sensitivity analysis investigates the performance implications of alternative objective-priority configurations. Finally, a Markovian consistency-dominated regime is analysed to characterise the limiting behaviour of SIGMA when phase consistency and transition smoothness receive substantially greater emphasis.

\begin{table*}[!ht]
\centering
\tiny
\caption{Configurations considered in the objective-component ablation study. A check mark ($\checkmark$) indicates that the corresponding objective is enabled, whereas a cross ($\times$) indicates that it is excluded from the optimisation. Case~5 represents a uniform-weight configuration.}
\label{tab:ablation_config}
\begin{tabularx}{\textwidth}{p{4.5cm} X X X X X}
\toprule
\textbf{Case} & 
\textbf{Emergency} &
\textbf{Waiting Time} &
\textbf{Queue Length} &
\textbf{Markovian} &
\textbf{Smoothness} \\
& $(\lambda_E,\alpha_E)$ & $(\lambda_W,\alpha_W)$ & $(\lambda_Q,\alpha_Q)$ & $(\lambda_M,\alpha_M)$ & $(\lambda_S,\alpha_S)$ \\
\midrule

\textbf{Case 1: Full Model} & 
$\checkmark$ & $\checkmark$ & $\checkmark$ & $\checkmark$ & $\checkmark$ \\
\textit{All components enabled} & & & & & \\
\midrule

\textbf{Case 2: No Emergency Priority} & 
$\times$ & $\checkmark$ & $\checkmark$ & $\checkmark$ & $\checkmark$ \\
\textit{Emergency objective disabled} & & & & & \\
\midrule

\textbf{Case 3: No Emergency + No MWT} & 
$\times$ & $\times$ & $\checkmark$ & $\checkmark$ & $\checkmark$ \\
\textit{Emergency and waiting-time objectives disabled} & & & & & \\
\midrule

\textbf{Case 4: No Emergency + No MWT + No QL} & 
$\times$ & $\times$ & $\times$ & $\checkmark$ & $\checkmark$ \\
\textit{Emergency, waiting-time, and queue objectives disabled} & & & & & \\
\midrule

\textbf{Case 5: Equal Weighting} & 
$1.0$ & $1.0$ & $1.0$ & $1.0$ & $1.0$ \\
\textit{All $\lambda$ and $\alpha$ parameters set uniformly} & & & & & \\
\bottomrule
\end{tabularx}

\vspace{0.2cm}
\footnotesize
\textbf{Note:} $\checkmark$ denotes an active objective component using its specified weight, whereas $\times$ denotes an excluded component with weight set to zero. Case~1 corresponds to the complete SIGMA formulation with the tuned objective weights. Cases~2--4 progressively remove selected objective components. Case~5 retains the complete objective set but assigns equal weights to all components.
\end{table*}

\begin{table*}[t]
\centering
\tiny
\caption{Performance of SIGMA under alternative objective-component configurations. Results are reported as mean $\pm$ standard deviation over 1000 simulation episodes. Arrows indicate the preferred direction of each metric.}
\label{tab:ablation_results}
\begin{tabularx}{\textwidth}{p{2.8cm} X X X X X X}
\toprule
\textbf{Case} &
\textbf{AEWT (s) $\downarrow$} &
\textbf{AMWT (s) $\downarrow$} &
\textbf{AWT (s) $\downarrow$} &
\textbf{AQL $\downarrow$} &
\textbf{TC $\uparrow$} &
\textbf{ATP $\uparrow$} \\
\midrule

\textbf{Case 1: Full Model} &
$\mathbf{23.47 \pm 2.8}$ &
$133.27 \pm 16.3$ &
$77.54 \pm 9.3$ &
$11.38 \pm 1.4$ &
$0.39 \pm 0.05$ &
$9.26 \pm 1.08$ \\
\textit{All components} & & & & & & \\
\midrule

\textbf{Case 2: No Emergency} &
$72.37 \pm 1.9$ &
$\mathbf{131.17 \pm 15.9}$ &
$75.13 \pm 8.1$ &
$10.97 \pm 1.3$ &
$0.41 \pm 0.07$ &
$9.31 \pm 1.06$ \\
\textit{$(\lambda_E,\alpha_E)$ disabled} & & & & & & \\
\midrule

\textbf{Case 3: No Emergency + No MWT} &
$71.21 \pm 1.7$ &
$177.27 \pm 21.23$ &
$\mathbf{70.21 \pm 6.9}$ &
$\mathbf{9.03 \pm 1.1}$ &
$0.43 \pm 0.04$ &
$\mathbf{10.18 \pm 0.93}$ \\
\textit{$(\lambda_E,\alpha_E)$ and $(\lambda_W,\alpha_W)$ disabled} & & & & & & \\
\midrule

\textbf{Case 4: No Emergency + No MWT + No QL} &
$83.32 \pm 2.1$ &
$191.49 \pm 23.97$ &
$92.26 \pm 9.3$ &
$16.13 \pm 2.1$ &
$\mathbf{0.59 \pm 0.08}$ &
$5.97 \pm 0.61$ \\
\textit{$(\lambda_E,\alpha_E)$, $(\lambda_W,\alpha_W)$, and $(\lambda_Q,\alpha_Q)$ disabled} & & & & & & \\
\midrule

\textbf{Case 5: Equal Weighting} &
$72.54 \pm 1.8$ &
$176.17 \pm 20.29$ &
$71.94 \pm 7.1$ &
$9.18 \pm 1.1$ &
$0.42 \pm 0.06$ &
$10.03 \pm 0.98$ \\
\textit{All $\lambda,\alpha=1.0$} & & & & & & \\
\bottomrule
\end{tabularx}

\vspace{0.2cm}
\footnotesize
\textbf{Note:} $\downarrow$ denotes that lower values are preferable and $\uparrow$ denotes that higher values are preferable. AEWT, AMWT, AWT, AQL, TC, and ATP denote average emergency waiting time, average maximum waiting time, average waiting time, average queue length, transition consistency, and average throughput, respectively. Boldface indicates the best value within each metric across the ablation configurations.
\end{table*}

\subsection{Ablation Study}
\label{sec:ablation}

The ablation study is designed to isolate the contribution of the principal objective components incorporated into SIGMA. The configurations considered are summarized in Table~\ref{tab:ablation_config}, and the corresponding performance results are reported in Table~\ref{tab:ablation_results}. All configurations are evaluated under identical network, demand, simulation-horizon, and action-space settings to ensure a consistent experimental basis. The hyperparameter configurations used for the individual cases are provided in Appendix~\ref{sec:algos}.

\subsubsection{Component-wise Ablation Study}
Five configurations are considered. Case~1 represents the complete SIGMA formulation, in which all five objectives – emergency prioritisation, waiting time, queue length, Markovian consistency, and phase-transition smoothness - are jointly incorporated. Cases~2--4 progressively remove the emergency, waiting-time, and queue-length objectives, respectively. Case~5 retains all objective components but assigns equal values to the corresponding $\lambda$ and $\alpha$ parameters, thereby providing a reference for evaluating the effect of objective-weight calibration. The results in Table~\ref{tab:ablation_results} demonstrate that the individual objective components influence different dimensions of control performance. 

The full formulation in Case~1 achieves the lowest AEWT among all configurations while maintaining competitive performance across the general traffic metrics. This result indicates that explicit incorporation of emergency prioritisation can substantially improve emergency responsiveness without requiring an equivalent deterioration in the broader traffic measures. Removing the emergency objective in Case~2 results in a pronounced increase in AEWT, whereas the changes in AWT, AQL, TC, and ATP are comparatively limited. This contrast provides evidence that the emergency-response improvement observed for the full formulation is closely associated with the explicit emergency objective. The result also indicates that emergency prioritisation can be introduced within the multi-faced formulation without substantially altering the controller's general traffic-management behaviour. Case~3 removes both the emergency and waiting-time objectives. Relative to Case~2, the configuration yields lower AWT and AQL and higher ATP, but AMWT increases considerably. This behaviour illustrates a trade-off between aggregate efficiency and delay equity: emphasising queue reduction and throughput without an explicit waiting-time criterion may improve system-level traffic processing while allowing disproportionately large delays to occur on individual approaches. The waiting-time objective therefore provides an additional mechanism for regulating the distribution of delay across competing movements. The removal of the queue-length objective in Case~4 produces a deterioration in AQL, AWT, AMWT, and ATP. Although TC increases relative to the other configurations, this gain in transition consistency is accompanied by substantially weaker traffic-processing performance. The result highlights the operational interaction between phase consistency and queue management and indicates that queue-length minimisation is an important component for maintaining effective traffic discharge. Case~5 examines the effect of uniform objective weighting. Despite retaining all objective components, the equal-weight formulation produces substantially higher AEWT than Case~1. Its relatively competitive AWT, AQL, and ATP values demonstrate that uniform weighting can still yield reasonable general traffic performance; however, it does not provide the same degree of emergency responsiveness as the calibrated formulation. This comparison supports the importance of appropriately assigning relative weights when the objectives have different operational priorities.

Overall, the component ablation indicates that the performance of SIGMA results from the interaction between multiple heterogeneous objectives rather than from any single objective term. Emergency prioritisation primarily affects emergency responsiveness, the waiting-time objective contributes to the regulation of delay disparities, and the queue-length objective supports efficient traffic processing. The results therefore provide empirical support for retaining the complete multi-faced formulation.

\begin{figure}[h]
    \centering
    \includegraphics[width=.45\linewidth]{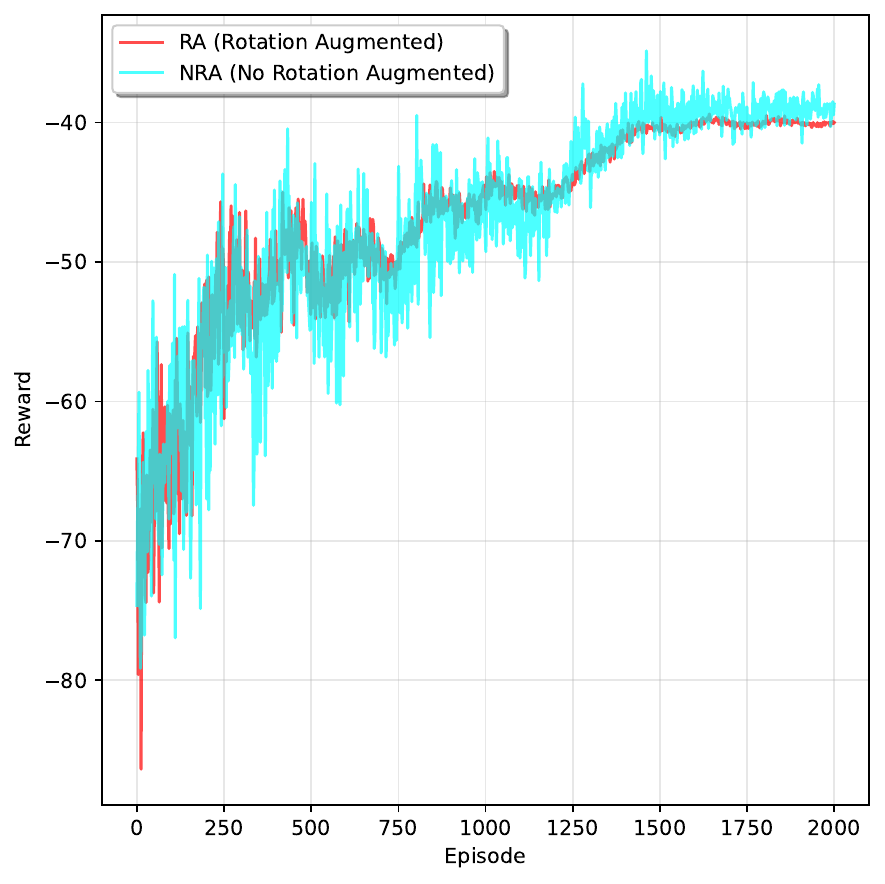}
    \caption{Training convergence of the RA and NRA configurations.}
    \label{fig:reward_rotationcomparison}
\end{figure}

\subsubsection{Rotation-Augmentation Ablation Study}

The contribution of rotation augmentation is evaluated by comparing the rotation-augmented (RA) and non-rotation-augmented (NRA) training configurations. The corresponding results are reported in Table~\ref{tab:rotation_ablation}. The experiment assumes directional symmetry in the traffic environment, such that equivalent traffic configurations obtained through rotations of $90^\circ$ should induce correspondingly transformed control decisions.

\begin{table*}[t]
\centering
\tiny
\caption{Ablation study of rotation augmentation under the assumed fourfold directional symmetry. Results are reported as mean $\pm$ standard deviation over the evaluation episodes.}
\label{tab:rotation_ablation}
\begin{tabularx}{\textwidth}{p{3.5cm} X X X X}
\toprule
\textbf{Training Configuration} &
\textbf{AEWT (s) $\downarrow$} &
\textbf{AMWT (s) $\downarrow$} &
\textbf{AWT (s) $\downarrow$} &
\textbf{AQL $\downarrow$} \\
\midrule

\textbf{RA (Rotation-Augmented)} &
$\mathbf{23.47 \pm 2.8}$ &
$\mathbf{133.27 \pm 16.3}$ &
$77.54 \pm 9.3$ &
$\mathbf{11.38 \pm 1.4}$ \\
\textit{Proposed training procedure with $C_4$ rotational augmentation} & & & & \\
\midrule

\textbf{NRA (No Rotation Augmentation)} &
$24.18 \pm 4.9$ &
$134.16 \pm 21.26$ &
$\mathbf{76.97 \pm 14.86}$ &
$12.21 \pm 2.4$ \\
\textit{Original orientation only} & & & & \\
\bottomrule
\end{tabularx}

\vspace{0.2cm}
\footnotesize
\textbf{Note:} RA denotes training with rotational augmentation over the four cardinal orientations $0^\circ$, $90^\circ$, $180^\circ$, and $270^\circ$. NRA denotes training without rotational augmentation. $\downarrow$ indicates that lower values are preferable. Boldface indicates the best value for each metric.
\end{table*}

The RA configuration achieves the lowest AEWT and AQL and exhibits substantially lower dispersion in several of the reported measures. Although NRA produces a marginally lower mean AWT, its larger standard deviation indicates greater variability across evaluation episodes. The comparison therefore suggests that rotation augmentation primarily contributes to the consistency and robustness of policy performance rather than uniformly improving every mean performance measure. The observed reduction in variability is consistent with the intended role of rotational augmentation. By exposing the policy to rotated instances of otherwise equivalent traffic configurations, the training procedure discourages the learning of orientation-specific patterns and encourages the extraction of structural relationships that remain valid under the assumed rotational symmetry. The convergence behaviour in Figure~\ref{fig:reward_rotationcomparison} further illustrates the corresponding difference in training dynamics.

However, the effectiveness of this strategy depends on the validity of the assumed symmetry. In environments with pronounced directional asymmetry arising from unequal demand, turning proportions, lane configurations, or network geometry, rotational transformations may generate states that do not represent the underlying traffic process. Consequently, the observed benefits should be interpreted as evidence supporting rotation augmentation under the stated symmetry assumption rather than as evidence of universal superiority across arbitrary traffic environments. These two ablation experiments provide evidence for the contribution of both the multi-faced formulation and the rotation-based training strategy. The former determines how competing operational objectives are reconciled, whereas the latter improves the robustness of the learned policy to equivalent directional configurations.

\subsection{Sensitivity Analysis}
\label{subsec:sensitivity_overview}
The sensitivity analysis evaluates the response of SIGMA to alternative objective-priority specifications. Let the normalised objective-weight vector be defined as
\[
\mathbf{w}
=
\frac{\boldsymbol{\alpha}}
{\mathbf{1}^{\top}\boldsymbol{\alpha}}
=
\frac{\boldsymbol{\lambda}}
{\mathbf{1}^{\top}\boldsymbol{\lambda}},
\qquad
\sum_{k\in\{M,S,Q,W,E\}}w_k=1.
\]
We consider seven configurations to represent distinct operational priorities: balanced control (C1), emergency-oriented control (C2), queue-oriented control (C3), fairness-oriented control (C4), stability-oriented control (C5), joint emergency-queue prioritisation (C6), and joint emergency-fairness prioritisation (C7). The complete weight specifications are provided in Table~\ref{app:sensitivity_full}, while the corresponding performance measures are summarised in Table~\ref{tab:sensitivity_summary}.

\begin{table}[t]
\centering
\tiny
\caption{Performance summary under alternative objective-weight configurations. Boldface indicates the best value for each metric.}
\label{tab:sensitivity_summary}
\begin{tabularx}{\textwidth}{p{2.5cm} X X X X X}
\toprule
\textbf{Configuration} &
\textbf{AEWT (s) $\downarrow$} &
\textbf{AWT (s) $\downarrow$} &
\textbf{AMWT (s) $\downarrow$} &
\textbf{TC $\uparrow$} &
\textbf{ATP $\uparrow$} \\
\midrule

C1: Balanced &
$72.54$ &
$\mathbf{71.94}$ &
$176.17$ &
$0.42$ &
$10.03$ \\
\midrule

C2: Emergency &
$\mathbf{18.23}$ &
$79.21$ &
$135.41$ &
$0.37$ &
$9.12$ \\
\midrule

C3: Queue &
$74.39$ &
$73.16$ &
$142.37$ &
$0.34$ &
$\mathbf{10.41}$ \\
\midrule

C4: Fairness &
$71.29$ &
$74.19$ &
$\mathbf{117.29}$ &
$0.36$ &
$9.53$ \\
\midrule

C5: Stability &
$82.17$ &
$82.56$ &
$148.71$ &
$\mathbf{0.48}$ &
$8.47$ \\
\midrule

C6: E+Queue &
$24.12$ &
$74.02$ &
$127.83$ &
$0.38$ &
$9.87$ \\
\midrule

C7: E+Fairness &
$19.87$ &
$75.91$ &
$121.45$ &
$0.36$ &
$9.34$ \\
\bottomrule
\end{tabularx}
\end{table}

The results in Table~\ref{tab:sensitivity_summary} reveal systematic trade-offs among the considered operational objectives. No single weight configuration simultaneously dominates all performance measures, which is consistent with the inherently multi-faced nature of signal control. Changes in the relative importance assigned to one objective alter the control policy in ways that can improve one dimension of performance while affecting others. Configuration C2, which emphasises emergency response, achieves the lowest AEWT among all configurations. This result demonstrates that increasing the relative weight assigned to emergency prioritisation can substantially improve emergency responsiveness. The corresponding increase in AWT and reduction in ATP relative to several other configurations illustrate the operational trade-off associated with allocating greater optimisation priority to emergency movements. Configuration C3, which emphasises queue management, achieves the highest ATP and maintains competitive queue-related performance. However, its AEWT is considerably higher than those of the emergency-oriented configurations. This behaviour indicates that prioritising aggregate queue reduction and traffic discharge does not necessarily produce equivalent improvements in emergency service. Configuration C4 produces the lowest AMWT among the evaluated configurations. This finding indicates that greater emphasis on fairness-oriented delay control can reduce excessive waiting times experienced by individual approaches, even when aggregate waiting time is not minimized. The distinction is important because average delay and maximum delay capture different dimensions of service quality. Configuration C5 attains the highest TC value, confirming that stronger emphasis on stability-related objectives leads to more consistent phase transitions. This improvement is accompanied by comparatively higher AWT and lower ATP, demonstrating the efficiency-related trade-off associated with prioritising transition consistency. The two combined-objective configurations provide intermediate operating points. C6, which jointly emphasises emergency response and queue management, substantially reduces AEWT relative to the queue-focused configuration while retaining comparatively high ATP. C7 similarly emphasises emergency response while incorporating fairness considerations and produces both low AEWT and low AMWT relative to most configurations. These results demonstrate that the objective-weight vector can be used to navigate different regions of the performance trade-off space.

Importantly, the sensitivity analysis does not establish a universally optimal configuration. Instead, it demonstrates that the preferred weight vector depends on the operational priorities of the target environment. An emergency-oriented application may reasonably favor C2, whereas an application emphasising throughput, delay equity, or phase stability may favor C3, C4, or C5, respectively. The combined configurations provide additional alternatives when multiple operational priorities must be accommodated simultaneously. Accordingly, the principal implication of the sensitivity analysis is that SIGMA provides an explicit mechanism for controlling the relative emphasis placed on competing traffic-management objectives. The underlying policy architecture remains unchanged, while the objective-weight vector determines the desired operating regime.

\subsection{Markovian Consistency-Dominated Regime}
\label{subsec:markovian_regime}

To further characterize the behaviour of SIGMA at an extreme point of the objective-weight space, we consider a specialized \textit{Markovian consistency-dominated regime} defined by

\[
(w_M,w_S,w_W)
=
\left(
\frac{1}{2}-\epsilon,
\frac{1}{2}-\epsilon,
2\epsilon
\right),
\qquad
\epsilon\rightarrow0^+,
\]

with the remaining objective weights set to zero. In this limiting regime, Markovian consistency and phase-transition smoothness dominate the objective, while the waiting-time component becomes progressively less influential. Under this configuration, the controller is restricted to a local phase-transition structure in which the admissible decision at each step is either to retain the current phase or to advance to the next phase in the prescribed cyclic order. The resulting decision rule can be expressed as

\[
a^{t+1}
=
\begin{cases}
a^t,
&
\text{if }
\mathcal{D}^\dagger_{\mathrm{current}}
\geq
\mathcal{D}^\dagger_{\mathrm{next}},
\\[4pt]
\operatorname{next}(a^t),
&
\text{if }
\mathcal{D}^\dagger_{\mathrm{next}}
>
\mathcal{D}^\dagger_{\mathrm{current}},
\end{cases}
\]

where $\mathcal{D}^\dagger_{\mathrm{current}}$ and $\mathcal{D}^\dagger_{\mathrm{next}}$ denote the corresponding demand measures. The resulting policy can therefore be interpreted as a demand-responsive cyclic scheduling mechanism in which demand determines whether the current phase is retained or the next phase is activated.

This regime provides a useful limiting interpretation of the role of Markovian consistency within SIGMA. Rather than prioritising instantaneous efficiency, the controller places greater emphasis on predictable phase evolution and transition regularity. Such a configuration may be relevant in environments where operational predictability is considered particularly important, although its suitability depends on the specific requirements and traffic characteristics of the deployment setting. The Markovian analysis therefore complements the general sensitivity analysis by illustrating the behaviour of SIGMA under a strongly consistency-oriented objective specification. It demonstrates that the proposed framework can represent a continuum of operational regimes, ranging from efficiency-oriented control to policies that place greater emphasis on phase predictability and transition regularity. The complete sensitivity analysis, including the detailed objective-weight configurations and associated results, is provided in Appendix~\ref{app:sensitivity_full}.

\subsection{LLM Selection for Emergency Instruction Interpretation}
\label{subsec:llm_selection}

Because emergency-priority information is derived from natural-language instructions, an additional experiment is conducted to identify an appropriate language-model backbone for instruction interpretation. Seven open-source small-scale language models are evaluated with respect to their zero-shot interpretation of emergency directives. The evaluation set comprises 200 natural-language emergency instructions spanning five predefined categories. Model performance is assessed using exact-match accuracy against the corresponding ground-truth priority vectors, together with per-query inference latency measured on an NVIDIA RTX 5060 GPU.

The resulting accuracy-latency comparison is presented in Figure~\ref{fig:llm_comp}. The evaluated models exhibit a trade-off between interpretation accuracy and computational latency. LLaMA-2-7B achieves the highest exact-match accuracy among the considered models while maintaining an inference latency compatible with the control interval. The smaller models generally exhibit lower inference latency but show comparatively lower instruction-interpretation accuracy. Based on this accuracy-latency trade-off, LLaMA-2-7B is selected as the language-model backbone for the subsequent SIGMA experiments. This selection reflects the joint requirements of reliable emergency-instruction interpretation and computational feasibility for real-time traffic-control operation.

\begin{figure*}[h]
    \centering
    \includegraphics[width=\linewidth]{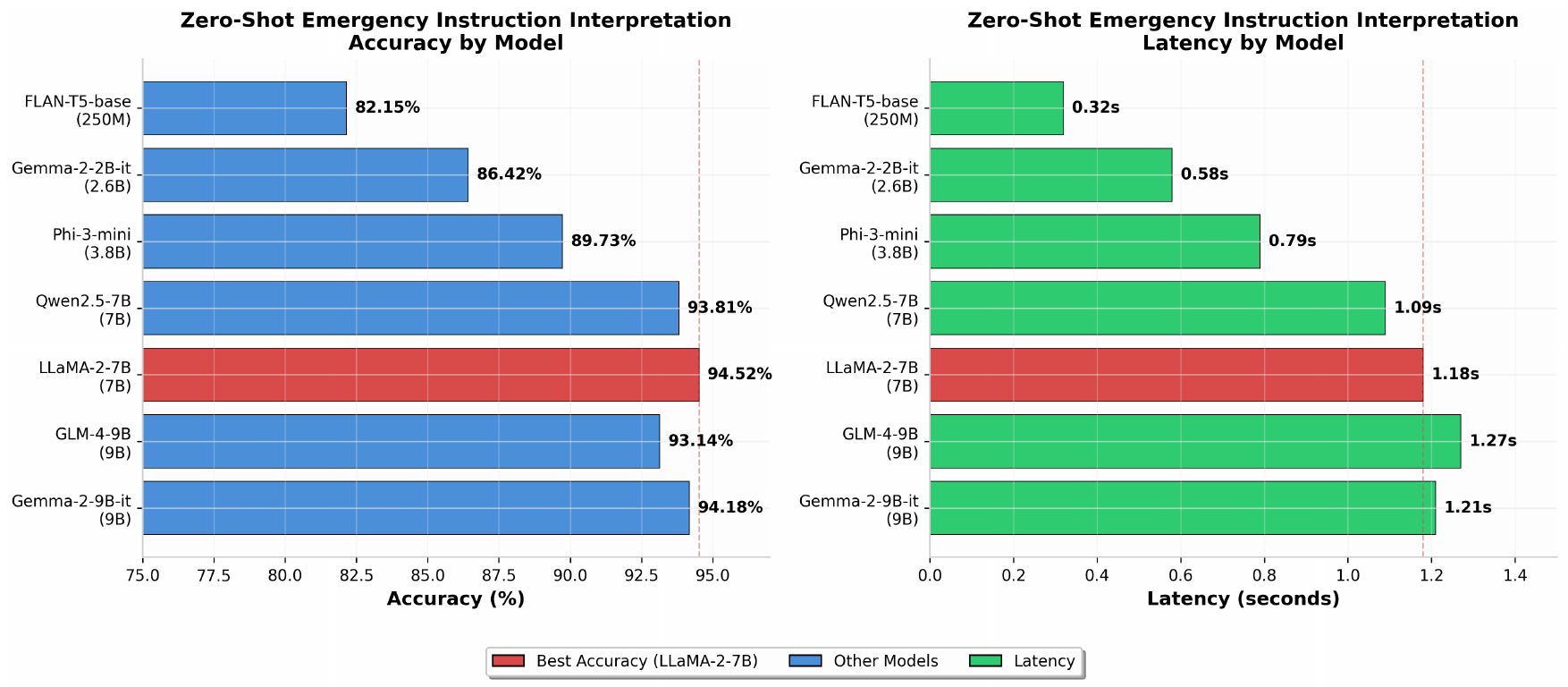}
    \caption{Accuracy-latency comparison of the evaluated language models for emergency-instruction interpretation.}
    \label{fig:llm_comp}
\end{figure*}

\section{Limitations and Future Work}
\label{sec:limfurwork}

Although SIGMA shows promising results, the current study has several limitations that will be addressed in future work.

(i) \textit{LLM inference latency:} The LLaMA-2-7B model provides high accuracy for interpreting emergency instructions, but its inference latency is relatively high for a real time traffic control system. The measured latency is approximately 1.18~s. Future work will therefore investigate smaller language models, knowledge distillation, and model compression to reduce inference time while maintaining reliable interpretation of emergency instructions.

(ii) \textit{Single intersection evaluation:} The current experiments consider individual four leg intersections. In a real traffic network, neighbouring intersections interact with each other through vehicle flow and queue spillback. Future work will extend SIGMA to multiple connected intersections and investigate coordinated signal control. Such an extension could also support coordinated emergency-vehicle movement across a sequence of intersections.

(iii) \textit{Dependence on fourfold symmetry:} The rotation-based training strategy assumes that the intersection has four geometrically similar approaches. This assumption may not hold for irregular intersections with different numbers of lanes, turning movements, or traffic patterns across approaches. Therefore, the proposed rotation augmentation should not be directly applied to heterogeneous intersections without appropriate modifications.

(iv) \textit{Extension to more general intersection geometries:} The current framework focuses on $C_4$ symmetry, corresponding to the four rotational orientations of a four-leg intersection. Future work will investigate a more general $C_k$-symmetric formulation for intersections with $k$ homogeneous approaches. This will require suitable state and action transformations for different intersection geometries. For intersections where such symmetry does not exist, graph-based representations or learned spatial embeddings may provide a more appropriate solution.

(v) \textit{Limited set of traffic objectives:} The current formulation considers waiting time, queue length, throughput, transition consistency, and emergency priority. Other practical considerations, such as vehicle emissions, fuel consumption, pedestrian waiting time, and public transport priority, are not included. Future studies will incorporate these objectives and examine the resulting trade-offs under different traffic conditions.

(vi) \textit{Simulation-based evaluation:} The results are obtained using SUMO simulations based on four Kolkata traffic intersections and the corresponding traffic flow data collected. Although simulation provides a controlled environment for comparing different controllers, real traffic contains additional sources of uncertainty, including imperfect traffic measurements, communication delays, unexpected driver behaviour, and changes in traffic demand. Future work will therefore evaluate SIGMA using real traffic live data and, where possible, real-world traffic control environments.

\section{Conclusion}
\label{sec:conclusion}
This paper proposed SIGMA, a multi-faced reinforcement learning framework for intelligent traffic signal control with emergency vehicle prioritisation. SIGMA combines natural-language instruction interpretation through an LLM with an actor-critic controller that jointly considers multiple traffic objectives. This enables emergency priorities to be communicated through natural-language instructions without redesigning the control formulation when operational priorities change. Rotation-based training further reduces dependence on the absolute orientation of homogeneous four-leg intersections and improves consistency across geometrically equivalent configurations. SUMO experiments on four Kolkata-based intersections show favourable overall performance compared with conventional controllers and DQN across several traffic-management measures. SIGMA reduces waiting times and queue lengths, improves emergency vehicle response, and increases traffic throughput, while maintaining competitive performance for regular traffic. The results highlight the importance of balancing emergency and non-emergency objectives within a unified control framework. Ablation and objective-weight experiments further demonstrate that removing the emergency component substantially reduces emergency responsiveness, whereas eliminating other traffic-related objectives deteriorates overall performance. Varying the objective weights changes the trade-off among emergency response, queue management, waiting-time fairness, and signal-transition behaviour. The rotational augmentation results additionally support exploiting geometric symmetry to reduce orientation-specific variability in the learned policy.

From a transportation-management and decision-making perspective, SIGMA treats objective weights as operational decision levers rather than fixed algorithmic parameters. Traffic authorities can therefore define operating regimes according to current priorities, such as emphasising queue reduction and throughput during peak periods, increasing emergency priority during ambulance or fire-service movements, or promoting smoother and more predictable phase transitions where operational consistency is important. The framework consequently separates the managerial decision of what should be prioritised from the lower-level decision of which signal phase to activate. This provides a practical means of translating policy preferences, service-level requirements, and event-specific priorities into real-time signal-control behaviour without redesigning the controller. The language-model experiments also demonstrate the potential of natural-language interfaces for communicating these changes. Among the evaluated models, LLaMA-2-7B achieved the highest interpretation accuracy, indicating that natural-language instructions can provide an intuitive interface between transportation operators and an adaptive control system. However, the observed inference latency highlights the need to consider computational efficiency, communication reliability, and appropriate fallback mechanisms during deployment. SIGMA's ability to operate without the LLM-based priority signal further contributes to its operational resilience.

\backmatter

\bmhead{Acknowledgements}

P. Payra would like to thank Mr. Purnendu Das (M.Tech. student, Indian Statistical Institute, Kolkata) for his critical inputs and suggestions.

\section*{Declarations}

\begin{itemize}
\item Authors have no conflict of interest.
\item Kolkata traffic volume data for intersections is avialable strictly on request and approval
\item The source code for the SIGMA framework is made
publicly available \url{https://github.com/pratham-payra/sigma_julia}
\item P.Payra conceived the study and performed data analysis and
code development. All authors contributed to result interpretation, manuscript
review, and approved the final version.
\end{itemize}

\begin{appendices}

\section{Training and Inference Procedures}
\label{sec:algos}

\subsection{Actor Network Pretraining}
\label{subsec:actor_pretraining}

The actor network $\pi_{\psi}(a \mid s)$ is first pre-trained on the rotation-augmented dataset $\mathcal{D}_F$ to obtain a stable and policy-aware initialization before being deployed online. The offline pretraining phase includes supervised action learning and multi-faced regularization to allow the actor to learn meaningful phase-selection behavior without direct interaction with the traffic environment. The actor is parameterized as a feed-forward neural network with three hidden layers with ReLU activations and a softmax output layer that maps a state representation $s$ to a probability distribution over the $|\mathcal{A}|$ admissible signal phases. Formally,
\[
\pi_{\psi}(a \mid s)
=
\operatorname{\sigma_{softmax}}
\left(
\mathbf{W}_4
\left[
\operatorname{ReLU}
\left(
\mathbf{W}_3
\left[
\operatorname{ReLU}
\left(
\mathbf{W}_2
\left[
\operatorname{ReLU}
\left(
\mathbf{W}_1s+\mathbf{b}_1
\right)
\right]
+\mathbf{b}_2
\right)
\right]
+\mathbf{b}_3
\right)
\right]
+\mathbf{b}_4
\right).
\]

The pretraining objective is designed to jointly promote action-level supervision and desirable traffic-control characteristics. Specifically, the actor loss $\mathcal{L}_{\text{actor}}$ consists of a cross-entropy term, which encourages agreement between the predicted policy and the target actions contained in $\mathcal{D}_F$, together with a weighted multi-faced regularisation term comprising five complementary utility objectives: Markovian consistency, action smoothness, queue-length reduction, waiting-time reduction, and emergency-vehicle alignment. Accordingly, the overall optimisation objective is defined as
\[
\mathcal{L}_{\text{actor}}
=
-\frac{1}{|\mathcal{D}_F|}
\sum_{(s,a)\in\mathcal{D}_F}
a^\top \log \pi_{\psi}(a\mid s)
+
\lambda^\top \mathcal{L}_R,
\]
where $\lambda=(\lambda_M,\lambda_S,\lambda_Q,\lambda_W,\lambda_E)^\top$ denotes the vector of objective-specific weighting coefficients, governing the relative contribution of each auxiliary objective, and
\[
\mathcal{L}_R
=
(\mathcal{L}_M,\mathcal{L}_S,\mathcal{L}_Q,\mathcal{L}_W,\mathcal{L}_E)^\top
\]
collects the corresponding regularisation losses. This formulation enables the actor to learn not only to reproduce desirable phase-selection decisions but also to internalise structural and operational properties conducive to effective traffic signal control.

The complete actor pretraining procedure is summarised in Algorithm~\ref{alg:actor_pretrain}. Starting from randomly initialized parameters $\psi$, the network is optimised iteratively via gradient-based updates over $E_{\text{actor}}$ epochs, yielding a well-conditioned policy initialisation for subsequent online learning and deployment.

\begin{algorithm}[ht]
\caption{Actor Pretraining}
\label{alg:actor_pretrain}
\tiny
\begin{algorithmic}

\State \textbf{Input:} Dataset $\mathcal{D}_F$
\State \textbf{Hyperparameters:} 
$\eta_{\pi}=10^{-3}$, 
$\lambda=(2.0,1.3,0.7,0.5,0.5)$, 
$E_{\text{actor}}=100$, 
$N_{\text{batch}}=256$

\State Initialize actor parameters $\psi$ randomly

\For{$epoch=1$ to $E_{\text{actor}}$}
    \State Sample mini-batch 
    $\mathcal{B}\sim\operatorname{Uniform}(\mathcal{D}_F,N_{\text{batch}})$
    
    \State Compute the actor loss:
    \[
    \mathcal{L}_{\text{actor}}
    =
    -\frac{1}{|\mathcal{B}|}
    \sum_{(s,a)\in\mathcal{B}}
    a^{\top}\log\pi_{\psi}(\cdot\mid s)
    +
    \lambda^{\top}\mathcal{L}_{R}
    \]
    
    \State Update actor parameters:
    \[
    \psi
    \leftarrow
    \psi-\eta_{\pi}
    \nabla_{\psi}\mathcal{L}_{\text{actor}}
    \]
\EndFor

\State \textbf{return} pretrained actor parameters $\psi$

\end{algorithmic}
\end{algorithm}

\subsection{Critic Network Pretraining}
\label{subsec:critic_pretraining}

The critic network, denoted by $Q_{\phi}(s,a)$, is designed to estimate the action-value function and thereby provide evaluative guidance for policy improvement. Unlike the actor, the critic is pretrained entirely in an offline setting using fitted Q-iteration, which integrates multi-faced reward construction, nearest neighbour-based next-state approximation, and iterative Bellman value refinement. This formulation enables the critic to learn informative action-value estimates despite the absence of directly observed and temporally continuous state transitions in the offline dataset.

The critic is parameterised as a feed-forward neural network that maps a state-action pair $[s;a]$ to a scalar action value. The architecture comprises two hidden layers with ReLU activation functions, followed by a linear output layer. Formally, the critic is expressed as
\[
Q_{\phi}(s,a)
=
\mathbf{W}_3^Q
\left[
\operatorname{ReLU}
\left(
\mathbf{W}_2^Q
\left[
\operatorname{ReLU}
\left(
\mathbf{W}_1^Q[s;a]
+
\mathbf{b}_1^Q
\right)
\right]
+
\mathbf{b}_2^Q
\right)
\right]
+
\mathbf{b}_3^Q.
\]

To ensure that the learned value function reflects the multiple operational objectives of the traffic control problem, a balanced reward signal is constructed by aggregating the five utility-specific reward components. For the $i$-th transition, the scalar reward is defined as $r_i=\alpha^\top \mathbf{r}_i^V$, where $\alpha=(\alpha_M,\alpha_S,\alpha_Q,\alpha_W,\alpha_E)^\top$ denotes the vector of reward-weighting coefficients and $\mathbf{r}_i^V=(r_{i,M},r_{i,S},r_{i,Q},r_{i,W},r_{i,E})^\top$ contains the corresponding rewards associated with Markovian consistency, action smoothness, queue reduction, waiting-time reduction, and emergency-vehicle alignment, respectively.

A fundamental challenge in the offline setting is that the available dataset does not provide complete temporal continuity between successive observations. The successor state associated with a given state-action pair is approximated using a $k$ nearest neighbor procedure. Specifically, for the transition associated with $(s_i,a_i)$, the approximate next state is obtained by averaging the successor states of its $k$ nearest state neighbors that share the same action:
\[
\hat{s}_i'
=
\frac{1}{k}
\sum_{j\in\mathcal{N}_k(s_i,a_i)}
s_{j+1},
\]
where $\mathcal{N}_k(s_i,a_i)$ denotes the set of the $k$ nearest neighbours of $s_i$ among the transitions associated with action $a_i$. This approximation provides a locally consistent estimate of the successor-state distribution and facilitates Bellman-based value propagation in the absence of explicitly observed transition trajectories.

The critic is subsequently optimised using fitted Q-iteration. The procedure is initialised by assigning the immediate reward as the initial target value, $Q_i^{(0)}=r_i$. At each iteration $t$, the critic parameters are obtained by minimising the squared discrepancy between the predicted action values and the current fitted targets:
\[
\phi^{(t+1)}
\leftarrow
\arg\min_{\phi}
\sum_{i=1}^{N}
\left[
Q_{\phi}(s_i,a_i)
-
Q_i^{(t)}
\right]^2.
\]

The target values are then updated according to the Bellman optimality principle:
\[
Q_i^{(t+1)}
\leftarrow
r_i
+
\gamma
\max_{a'\in\mathcal{A}}
Q_{\phi^{(t+1)}}(\hat{s}_i',a').
\]
This iterative approach is repeated for a maximum of $K$ epochs and terminates early if the change in the fitted target values is such that $\left\|Q^{(t+1)}-Q^{(t)}\right\|_2<\epsilon$.

The full critic pretraining technique is summarised in Algorithm~\ref{alg:critic_pretrain}. When the convergence is achieved or the stopping criterion is satisfied, the last fitted action-value estimates are used to construct the critic training dataset, $\mathcal{D}_{\text{critic}}=\left\{(s_i,a_i,Q_i^{(\text{final})})\right\}_{i=1}^{N}$, which is the pretrained value-function representation to be used for subsequent policy evaluation and improvement.

\begin{algorithm}[ht]
\caption{Critic Pretraining}
\label{alg:critic_pretrain}
\footnotesize
\begin{algorithmic}

\State \textbf{Input:} Offline dataset $\widetilde{\mathcal{D}}$, nearest neighbor parameter $k$, discount factor $\gamma$, convergence threshold $\epsilon$, number of iterations $K$

\State \textbf{Hyperparameters:} 
$\eta_{\phi}=10^{-3}$, 
$(\alpha_M,\alpha_S,\alpha_Q,\alpha_W,\alpha_E)=(0.5,0.5,0.7,1.3,2.0)$, 
$k=5$, 
$\gamma=0.95$, 
$\epsilon=10^{-4}$, 
$K=50$, 
architecture $(\dim(\mathcal{S}),|\mathcal{A}|,h_1,h_2)=(52,8,128,64)$

\State Initialize $Q_i^{(0)}\gets r_i$ for all $i\in\{1,\dots,N\}$

\State Initialize network $Q_{\phi}$ with $\phi^{(0)}$

\For{$t=0$ to $K-1$}

    \State $\displaystyle
    \phi^{(t+1)}
    \gets
    \arg\min_{\phi}
    \sum_{i=1}^{N}
    \left[
    Q_{\phi}(s_i,a_i)-Q_i^{(t)}
    \right]^2$

    \For{$i=1$ to $N$}
        \State $\displaystyle
        Q_i^{(t+1)}
        \gets
        r_i+
        \gamma
        \max_{a'}
        Q_{\phi^{(t+1)}}(\hat{s}'_i,a')$
    \EndFor

    \If{$\left\|Q^{(t+1)}-Q^{(t)}\right\|_2<\epsilon$}
        \State \textbf{break}
    \EndIf

\EndFor

\State Construct
$\mathcal{D}_{\text{critic}}
=
\left\{
(s_i,a_i,Q_i^{(\text{final})})
\right\}_{i=1}^{N}$

\State \textbf{return} $\mathcal{D}_{\text{critic}}$

\end{algorithmic}
\end{algorithm}

\subsection{Online Execution}
\label{subsec:online_execution}

During online deployment, the pretrained actor-critic framework operates in a closed-loop manner by continuously interacting with the live traffic environment. At each decision epoch, the system observes the current intersection state, optionally incorporates emergency-related instructions through the LLM module, selects and executes an appropriate signal phase, records the resulting transition, and periodically updates both the actor and critic networks. The local state representation at the intersection $v$ and time step $t$ is defined as $s_v^t=[p_v^{t-1},\mathbf{M}_{v,q}^t,\boldsymbol{\Gamma}_v^t,\boldsymbol{\theta}_v^t]$, where the constituent components encode the previously executed phase, queue-related information, maximum waiting times, and priority parameters, respectively.

When an emergency instruction $m$ is received, the LLM processes the instruction in conjunction with the relevant intersection context $\mathcal{PR}(m,v)$ and generates an emergency-priority representation, $\boldsymbol{\theta}_v^t=\operatorname{LLM}(\mathcal{PR}(m,v))$. The resulting priority vector is subsequently incorporated into the state representation before action selection, thereby allowing the control policy to adapt its decisions in accordance with dynamically specified priority requirements.

Given the current state $s_v^t$, the actor samples a control action according to $a_v^t\sim\pi_\psi(\cdot\mid s_v^t)$ and executes the corresponding signal phase. Following the interaction with the environment, the agent receives a composite reward that aggregates the five utility-specific components:
\[
r_v^t
=
\sum_{k\in\{M,S,Q,W,E\}}
\alpha_k \cdot r_{v,k}^t.
\]
The resulting transition $(s_v^t,a_v^t,r_v^t,s_v^{t+1})$ is subsequently stored in the replay buffer $\mathcal{D}_b$, facilitating experience-based learning and reducing the dependence of parameter updates on temporally correlated observations.

Once the replay buffer contains at least $N_{\text{batch}}$ transitions, a mini-batch $\mathcal{B}$ is sampled uniformly for network optimisation. The critic is updated using temporal-difference learning, with the target value defined as $y=r+\gamma Q_{\phi_{\text{target}}}(s',a')$, where the successor action is sampled according to $a'\sim\pi_\psi(\cdot\mid s')$. The corresponding critic objective is
\[
\mathcal{L}_\phi
=
\sum_{(s,a,r,s')\in\mathcal{B}}
\left[
y-Q_\phi(s,a)
\right]^2.
\]

Following the critic update, the actor is optimised using a policy-gradient objective. The policy gradient is estimated as
\[
\nabla_\psi J
=
\sum_{(s,a,r,s')\in\mathcal{B}}
\delta
\nabla_\psi
\log\pi_\psi(a\mid s),
\]
where $\delta=r+\gamma Q_\phi(s',a')-Q_\phi(s,a)$ denotes the temporal-difference error. The network parameters are subsequently updated according to $\phi\leftarrow\phi-\eta_Q\nabla_\phi\mathcal{L}_\phi$ and $\psi\leftarrow\psi+\eta_\pi\nabla_\psi J$, where $\eta_Q$ and $\eta_\pi$ denote the learning rates of the critic and actor, respectively.

To improve training stability and ensure gradual propagation of learned value and policy estimates, the target networks are periodically updated using soft parameter updates. Specifically,
\[
\phi_{\text{target}}
\leftarrow
\tau_o\phi
+
(1-\tau_o)\phi_{\text{target}},
\]
and
\[
\psi_{\text{target}}
\leftarrow
\tau_o\psi
+
(1-\tau_o)\psi_{\text{target}},
\]
where $\tau_o$ controls the rate at which the target parameters track their corresponding online networks. The complete closed-loop online execution and adaptation procedure is summarised in Algorithm~\ref{alg:online_execution}.

\begin{algorithm}[ht]
\caption{Online Execution}
\label{alg:online_execution}
\tiny
\begin{algorithmic}

\State \textbf{Input:} Pretrained actor $\pi_\psi$, pretrained critic $Q_\phi$, replay buffer capacity $N$, batch size $N_{\text{batch}}$, time step $\Delta t$, learning rates $\eta_Q$, $\eta_\pi$, target update frequency $K_{\text{target}}$, soft update rate $\tau_o$

\State \textbf{Hyperparameters:} 
$N_{\text{batch}}=512$, 
$N=128$, 
$\Delta t=30\,\mathrm{s}$, 
$\eta_Q=10^{-4}$, 
$\eta_\pi=10^{-4}$, 
$K_{\text{target}}=100$, 
$\tau_o=0.001$

\State Initialize replay buffer $\mathcal{D}_b\leftarrow\emptyset$, target networks, $t\leftarrow0$

\Loop

    \State Observe state $s_v^t$; inject LLM priority $\boldsymbol{\theta}_v^t$ if available

    \State Sample action $a_v^t\sim\pi_\psi(\cdot\mid s_v^t)$; execute; observe $r_v^t,s_v^{t+1}$

    \State Store $(s_v^t,a_v^t,r_v^t,s_v^{t+1})$ in $\mathcal{D}_b$

    \If{$|\mathcal{D}_b|\geq N_{\text{batch}}$}

        \State Sample batch $\mathcal{B}$; compute targets
        $y=r+\gamma Q_{\phi_{\text{target}}}(s',a')$
        and TD errors
        $\delta=r+\gamma Q_\phi(s',a')-Q_\phi(s,a)$

        \State Update critic:
        $\phi\leftarrow\phi-\eta_Q\nabla_\phi\mathcal{L}_\phi$
        where
        $\mathcal{L}_\phi=
        \sum_{(s,a,r,s')\in\mathcal{B}}
        \left(y-Q_\phi(s,a)\right)^2$

        \State Update actor:
        $\psi\leftarrow\psi+\eta_\pi\nabla_\psi J$
        where
        $\nabla_\psi J=
        \sum_{(s,a,r,s')\in\mathcal{B}}
        \delta\nabla_\psi\log\pi_\psi(a\mid s)$

        \If{$t\bmod K_{\text{target}}=0$}

            \State Soft-update target networks:
            $\phi_{\text{target}}\leftarrow
            \tau_o\phi+(1-\tau_o)\phi_{\text{target}}$

            \State $\psi_{\text{target}}\leftarrow
            \tau_o\psi+(1-\tau_o)\psi_{\text{target}}$

        \EndIf

    \EndIf

    \State $t\leftarrow t+\Delta t$

\EndLoop

\end{algorithmic}
\end{algorithm}

\subsection{Fixed-Time Control}
\label{subsec:fixed_time}

\textbf{Fixed-time traffic signal control} (Algorithm ~\ref{alg:fixed_time}) is based on a fixed cyclic schedule and hence does not depend on the real-time changes in traffic circumstances. The entire cycle length $C$ is given by the sum of the fixed green time duration for the $K$ preset signal phases, i.e., $C=\sum_{k=1}^{K}d_k$, where $d_k$ is the fixed green time duration for phase $k$. To maintain temporal consistency and compliance with the underlying signal-control resolution, each phase duration is limited to integer multiples of a preset base time unit $\Delta$, i.e., $d_k=\Delta\cdot m_k$, where $m_k\in\mathbb{Z}_{>0}$.

The stages are performed successively in a preset order, with each phase being maintained for its allocated time before the controller transitions to the next phase. Thus, fixed-time control offers a stable and deterministic baseline with predictable operating behaviour. However, such a control technique is fundamentally unable to cope with dynamic demand variations, unforeseen congestion patterns or emergency-priority needs, as neither the phase sequence nor the allocated durations are modified to the current traffic situation.

\begin{algorithm}[ht]
\caption{Fixed-Time Control}
\label{alg:fixed_time}
\footnotesize
\begin{algorithmic}

\State \textbf{Input:} Number of phases $K$, phase sequence $(1,2,\dots,K)$, base time unit $\Delta$, and phase multipliers $(m_1,m_2,\dots,m_K)$

\State \textbf{Hyperparameters:} 
$\Delta=30\,\mathrm{s}$, 
$m_k\in\mathbb{Z}^{+}$, 
$d_k=\Delta m_k=30m_k\,\mathrm{s}$, 
and cycle length
$C=\sum_{k=1}^{K}d_k
=30\sum_{k=1}^{K}m_k\,\mathrm{s}$

\State Initialize current phase $p\leftarrow1$ and simulation time $t\leftarrow0$

\Loop

    \For{$k=1$ to $K$}

        \State Set current phase $p\leftarrow k$

        \State Activate phase $p$ and hold the green signal for duration $d_k$

        \State Advance time: $t\leftarrow t+d_k$

    \EndFor

\EndLoop

\end{algorithmic}
\end{algorithm}

\subsection{Actuated Control}
\label{subsec:actuated_control}

Actuated traffic signal control (~\ref{alg:actuated_control}) modifies the phase duration according to the detected vehicles and guarantees operational safety with pre-defined temporal limitations. The control logic is driven by three main thresholds: minimum green duration ($G_{min}$) for safe phase operation, maximum green duration ($G_{max}$) to avoid excessive service to one approach and possible starvation of competing movements, and maximum time gap ($G_{gap}$) between successive vehicle detections to detect a reduction in active-phase demand.

For an active phase $p$ initiated at time $t_0$, the green interval is terminated at time
\[
t_1
=
\min\left(
t_0+G_{\max},
\inf\left\{
t\geq t_0+G_{\min}:
\max_i t_l^{(i)}(t)>G_{\text{gap}}
\right\}
\right),
\]
where $t_l^{(i)}(t)$ denotes the elapsed time since the most recent vehicle detection on approach $i$. This formulation ensures that the active phase remains green for at least $G_{\min}$ while allowing termination when the observed demand falls below the specified gap threshold, subject to the hard upper bound $G_{\max}$. Following phase termination, the next phase is selected according to the queue demand across its served approaches. Specifically, the successor phase is determined as $p_{\text{next}}=\arg\max_{p'}\sum_{i\in S(p')}l_i(t_1)$, where $S(p')$ denotes the set of approaches served by the candidate phase $p'$ and $l_i(t_1)$ represents the queue length on approach $i$ at the phase-transition time $t_1$.

In addition to phase-level actuation, downstream capacity constraints are incorporated to prevent the propagation of congestion into already saturated outgoing directions. When an outgoing direction reaches its maximum permissible capacity $Q_{\max}$, additional vehicles are prevented from entering that direction until sufficient downstream capacity becomes available. This capacity constraint is enforced consistently during both simulation-based evaluation and live traffic execution, thereby ensuring that the actuated controller accounts not only for upstream demand but also for downstream network congestion.

\begin{algorithm}[ht]
\caption{Actuated Control}
\label{alg:actuated_control}
\tiny
\begin{algorithmic}

\State \textbf{Input:} $\mathcal{P}$, vehicle detections, 
$G_{\min}=30\,\mathrm{s}$, 
$G_{\max}=180\,\mathrm{s}$, 
$G_{\mathrm{gap}}=3\,\mathrm{s}$; initialize $p\in\mathcal{P}$ and $t\leftarrow0$.

\Loop

    \State Activate $p$ and set $t_0\leftarrow t$; hold green until $t\geq t_0+G_{\min}$.

    \State Monitor detections and compute the maximum time gap; terminate at
    \[
    t_1\leftarrow
    \min\!\left(
    t_0+G_{\max},
    \inf\left\{
    t\geq t_0+G_{\min}:
    \max_i t_l^{(i)}(t)>G_{\mathrm{gap}}
    \right\}
    \right).
    \]

    \State Set $t\leftarrow t_1$ and select
    \[
    p_{\mathrm{next}}\leftarrow
    \arg\max_{p'\in\mathcal{P}}
    \sum_{i\in S(p')}
    \sum_{j\in T(p')}
    M_{ij}^{q}(t_1);
    \]
    set $p\leftarrow p_{\mathrm{next}}$.

\EndLoop

\end{algorithmic}
\end{algorithm}

\subsection{Deep Q-Network Baseline}
\label{subsec:dqn}

We compare the performance of SIGMA against the classic deep reinforcement learning baseline, the Deep Q-Network (DQN) (Algorithm~\ref{alg:dqn_baseline}). To provide a fair and controlled comparison, DQN is used on the same state space and reward formulation as the proposed framework. The method approximates the state-action value function $Q_\phi(s,a)$ using a neural network, utilizing experience replay with a separate target network to enhance training stability.

In each decision step the agent picks an action according to a $\epsilon$ greedy policy that mixes the exploitation of the current estimated optimal action with exploration behaviour. Specifically, the probability of taking the greedy action is $1-\epsilon+\epsilon/|\mathcal{A}|$, whereas the chance of taking any non-greedy action is $\epsilon/|\mathcal{A}|$, with $\epsilon$ being the exploration rate and $|\mathcal{A}|$ the size of the action space. The agent then observes the reward $r^t$ and the state $s^{t+1}$ resulting from doing action $a^t$ . The temporal-difference target is specified as $y=r^t+\gamma\max_{a'}Q_{\phi^-}(s^{t+1},a')$, where $\phi^-$ are the parameters of the target network.

The DQN parameters are optimised by minimising the squared temporal-difference error,
\[
\mathcal{L}_{\text{DQN}}
=
\left[
y-Q_\phi(s^t,a^t)
\right]^2,
\]
with the parameters updated by gradient descent as $\phi\leftarrow\phi-\eta\nabla_\phi\mathcal{L}_{\text{DQN}}$ where $\eta$ is the learning rate. The parameters of the target network, $\phi^-$, are regularly synchronised with the online-network parameters $\phi$, giving a slowly fluctuating target for value estimation, and so decreasing the instability due to the moving-target aspect of temporal-difference learning.

\begin{algorithm}[ht]
\caption{DQN Baseline}
\label{alg:dqn_baseline}
\tiny
\begin{algorithmic}

\State \textbf{Input:} $\mathcal{S},\mathcal{A},Q_\phi,Q_{\phi^-},\mathcal{D}_b$; 
$\epsilon=0.1,\gamma=0.95,\eta=10^{-4},\Delta t=30\,\mathrm{s},|\mathcal{A}|=8,
\dim(\mathcal{S})=52,(h_1,h_2)=(128,64),K_{\mathrm{target}}=100$.

\State Initialize $\phi$, set $\phi^-\leftarrow\phi$, $\mathcal{D}_b\leftarrow\emptyset$, $t\leftarrow0$.

\Loop

    \State Observe $s^t$; select $a^t\sim\epsilon$-greedy:
    $a^t=\text{random action}$ w.p. $\epsilon$, otherwise
    $a^t=\arg\max_{a\in\mathcal{A}}Q_\phi(s^t,a)$.

    \State Execute $a^t$; observe $(r^t,s^{t+1})$ and store
    $(s^t,a^t,r^t,s^{t+1})$ in $\mathcal{D}_b$; sample minibatch $\mathcal{B}\subset\mathcal{D}_b$.

    \State Compute $y=r+\gamma\max_{a'\in\mathcal{A}}Q_{\phi^-}(s',a')$ and
    $\mathcal{L}_{\mathrm{DQN}}=\frac{1}{|\mathcal{B}|}
    \sum_{(s,a,r,s')\in\mathcal{B}}[y-Q_\phi(s,a)]^2$.

    \State Update $\phi\leftarrow\phi-\eta\nabla_\phi\mathcal{L}_{\mathrm{DQN}}$;
    if $t\bmod K_{\mathrm{target}}=0$, set $\phi^-\leftarrow\phi$.

    \State $t\leftarrow t+\Delta t$.

\EndLoop

\end{algorithmic}
\end{algorithm}

\subsection{Synthetic Data Generation}
\label{subsec:data_generation}
Training and evaluation are accomplished with a realistic synthetic traffic-demand model with periodic temporal variations of vehicle flow and occasional emergency circumstances. To capture the stochastic and time-varying nature of urban traffic, the arrival and departure of vehicles are modelled as non-homogeneous Poisson processes with harmonic intensity functions. This formulation allows for the controlled modification of demand, while preserving the probabilistic nature of the traffic flow.

The instantaneous intensity at the arrival of each incoming approach $i$ is defined as $\nu_i(t)=\nu_0+\sum_{r=1}^{R}k_r\sin(\omega_r t+\phi_r)$ The departure intensity for each outgoing direction $j$ is $\mu_j(t)=\mu_0+\sum_{r=1}^{R}\kappa_r\sin(\omega_r t+\psi_r)$. The resultant intensities are truncated at zero to get physically meaningful non-negative flow rates. At each time step $\Delta t$ the number of incoming vehicles $N_i^{\mathrm{(arr)}}(t)$ and outgoing vehicles $N_j^{\mathrm{(dep)}}(t)$ are drawn from the distributions $N_i^{\mathrm{(arr)}}(t)\sim\operatorname{Poisson}(\nu_i(t)\Delta t)$ and $N_j^{\mathrm{(dep)}}(t)\sim\operatorname{Poisson}(\mu_j(t)\Delta t)$, and the numbers of incoming and outgoing vehicles are updated.

The net pressure representation is generated by comparing the incoming and outgoing queue data to characterise the instantaneous traffic imbalance at the intersection. The intermediate pressure vector is explicitly given by $\mathbf{M'}_q^t=\mathbf{M}_{q(\mathrm{in})}^t-\mathbf{M}_{q(\mathrm{out})}^t$. Then a shifted representation $\mathbf{M}_q^t=\mathbf{M'}_q^t-\min(\mathbf{M'}_q^t)$ is constructed to guarranty non-negative state components. Besides the pressure representation, the greatest waiting durations throughout the four approaches $\boldsymbol{\Gamma}^t=(\tau_e,\tau_n,\tau_w,\tau_s)^{\top}$ and the final completed phase $p^{t-1}$ are also kept. The set of variables provides an overview of present traffic conditions and recent history of control.

Furthermore, the presence of emergency vehicles is viewed as unusual episodic events to test the capacity of the control framework to respond to atypical and high-priority traffic situations. An event with a low chance of occurrence is sampled at the beginning of an emergency episode, and the vehicle type, origin, destination, and departure time are specified. The LLM processes these semantic properties to produce a priority representation $\boldsymbol{\theta}^t$, which is then added to the traffic state. Thus, the state representation is $s^t=[p^{t-1},\mathbf{M}_q^t,\boldsymbol{\Gamma}^t,\boldsymbol{\theta}^t]$.

Then, for each created traffic state, all possible signal control actions are assessed with the matching multi-faced reward function. The action with the largest reward is selected according to $a^{t*}=\arg\max_{a\in\mathcal{A}}R(s^t,a)$. The resultant training tuple, containing the observed state, the reward of each viable action, and the selected action, is stored as $(s^t,\{r^t(a)\}_{a\in\mathcal{A}},a^{t*})$ for further use in actor pretraining.

To avoid injecting subjective bias by manually specifying objective preferences for synthetic data generation, the reward-weighting vector $\boldsymbol{\alpha}=(\alpha_M,\alpha_S,\alpha_Q,\alpha_W,\alpha_E)^{\top}$ is randomly sampled subject to the constraints $\alpha_i>0$ and $\sum_{i\in\{M,S,Q,W,E\}}\alpha_i=5$. Hence, different synthetic traffic circumstances might correspond to different trade-offs among the underlying control goals with the same overall reward scale. The randomised weighting allows the offline dataset $\mathcal{\tilde{D}}$ to represent a variety of objective preferences, which improves resilience and reduces dependency of the pretrained policy on any particular manually supplied reward configuration.

The harmonic parameters for arrivals and departures are detailed in Tables~\ref{alg:data_generation}.

\begin{algorithm}[ht]
\caption{Synthetic Data Generation}
\label{alg:data_generation}
\tiny
\begin{algorithmic}

\State \textbf{Input:} $T,\Delta t,\mathcal{A},\nu_0,\mu_0,\{k_r,\kappa_r,\omega_r,\phi_r,\psi_r\}_{r=1}^{R},p_e,\{\alpha_i\}$; 
$\Delta t=30\,\mathrm{s},|\mathcal{A}|=8,\nu_0=\mu_0=4,R=12,\alpha_i>0,\sum_i\alpha_i=5$.

\State Set $\omega_r=2\pi r/2880$ and harmonic parameters 
$\{(k_r,\kappa_r,\phi_r,\psi_r)\}_{r=1}^{12}=
\{(4.5,5.5,0,0),(3.6,3.8,\pi/6,\pi/6),$
$(2.7,4.1,\pi/4,\pi/4),(2.2,3.7,\pi/3,\pi/3),(1.8,3.4,\pi/2,\pi/2),(1.4,2.41,2\pi/3,2\pi/3),(1.3,2.0,3\pi/4,3\pi/4),$\\$(1.1,1.8,5\pi/6,5\pi/6),(0.9,1.1,\pi,\pi),(0.7,0.9,7\pi/6,7\pi/6),$\\$(0.5,0.7,5\pi/4,5\pi/4),(0.3,0.4,4\pi/3,4\pi/3)\}$.

\State Initialize $t\leftarrow0$, demand parameters, and initial queues.

\Loop

    \State Compute $\nu_i(t),\mu_j(t)$ from the harmonic demand model; sample
    $N_i^{\mathrm{(arr)}}(t)\sim\operatorname{Poisson}(\nu_i(t)\Delta t)$ and
    $N_j^{\mathrm{(dep)}}(t)\sim\operatorname{Poisson}(\mu_j(t)\Delta t)$.

    \State Update $l_i^{\mathrm{in}},l_j^{\mathrm{out}}$ and construct
    $\mathbf{M}_{q(\mathrm{in})}^{t},\mathbf{M}_{q(\mathrm{out})}^{t},\mathbf{M}_{q}^{t}$; record $\boldsymbol{\Gamma}^{t}$ and $p^{t-1}$.

    \State With probability $p_e$, inject an emergency event and construct priority encoding $\boldsymbol{\theta}^{t}$; otherwise set $\boldsymbol{\theta}^{t}=\mathbf{0}$.

    \State Construct $s^t=[p^{t-1},\mathbf{M}_{q}^{t},\boldsymbol{\Gamma}^{t},\boldsymbol{\theta}^{t}]$ and evaluate $r^t(a)=R(s^t,a)$ for all $a\in\mathcal{A}$.

    \State Select $a^{t*}\leftarrow\arg\max_{a\in\mathcal{A}}r^t(a)$ and store
    $\left(s^t,\{r^t(a)\}_{a\in\mathcal{A}},a^{t*}\right)$.

    \State $t\leftarrow t+\Delta t$.

\EndLoop

\end{algorithmic}
\end{algorithm}

\subsection{Emergency Incident Generation and Encoding}
\label{subsec:emergency_generation}

The emergency-incident generation and encoding process (Algorithm~\ref{alg:emergency_generation}) combines probabilistic event activation and semantic representation to embed uncommon context-dependent priority needs into the traffic-control environment. In a first step, the emergency situations are defined in a finite scenario library, which is probabilistically triggered throughout the simulation. These scenarios are then converted into vector representations, which may be integrated into the control policy state.

A finite collection of emergency situations $\mathcal{E}=\{e_1,\ldots,e_{40}\}$ is specified, where each scenario defines the vehicle type, origin approach, destination approach, departure time, and a corresponding natural language description. For example, a scenario may be stated as \textit{``Ambulance approaching from North and going toward South.''} At the start of each simulation episode, the occurrence of an emergency event is sampled using a Bernoulli random variable $Z\sim\operatorname{Bernoulli}(p_e)$, with $p_e=1/20$ denoting the emergency-event probability. If $Z=1$, an emergency scenario is uniformly sampled from the predefined library, i.e., $e^t\sim\operatorname{Uniform}(\mathcal{E})$, and its natural-language incident description is transformed into a semantic priority representation $\boldsymbol{\theta}^t=\operatorname{LLM}(\text{incident description})$. If $Z=0$, no emergency event is introduced and the related priority representation is inactive.

This technique guarantees that emergency incidents are rare enough to represent their extraordinary nature, yet the diversity of vehicle kinds, travel routes, and temporal features is properly regulated. Thus, the generated semantic priority vector allows the control strategy to take into account, besides the standard traffic-state characteristics, the high-level emergency information to support adaptive signal choices for both regular and priority-sensitive operation scenarios.

\begin{algorithm}[ht]
\caption{Emergency Incident Generation}
\label{alg:emergency_generation}
\tiny
\begin{algorithmic}

\State \textbf{Input:} Scenario library $\mathcal{E}=\{e_1,\ldots,e_{40}\}$, occurrence probability $p_e=1/20$; each $e_j$ specifies vehicle type, origin approach, destination approach, and departure time.

\State Sample $Z\sim\operatorname{Bernoulli}(p_e)$.

\If{$Z=1$}
    \State Sample $e^t\sim\operatorname{Uniform}(\mathcal{E})$, retrieve $m=\operatorname{description}(e^t)$, and encode $\boldsymbol{\theta}^t\leftarrow\operatorname{LLM}(m)$.
\Else
    \State Set $\boldsymbol{\theta}^t=\mathbf{0}$ (no emergency event).
\EndIf

\end{algorithmic}
\end{algorithm}

\section{Rotational Augmentation Framework}
\label{sec:rotation-framework}

To achieve orientation invariance and enable a single policy to generalise effectively across arbitrarily rotated intersection layouts, we develop a systematic $C_4$ rotational augmentation framework. This section details the rotation operators, encoding schemes, and group structures that underpin our symmetry-aware training methodology.

\begin{table*}[!ht]
\centering
\tiny
\caption{Fundamental pattern groups for traffic signal actions under $C_4$ symmetry.}
\label{tab:pattern-groups}
\begin{tabularx}{\textwidth}{p{2.0cm} X p{0.4cm} p{2.5cm} X}
\toprule
\textbf{Pattern Group} & \textbf{Rotation Group} & \textbf{Size} & \textbf{Condition} & \textbf{Description} \\
\midrule
Single-Pivot Dominant & 
$G_1 = \{A_1^{E}, A_1^{N}, A_1^{W}, A_1^{S}\}$ Cyclic under $\rho_a$ & 
4 & 
Pivot $\mathbf{ic} \in \mathcal{IC}$ & 
All movements from a single incoming direction. \\
\midrule
Complementary Straight and Left-Turn & 
$G_2 = \{A_2^{N-S}, A_2^{E-W}\}$ Complements under $\rho_a$ & 
2 & 
$\mathbf{ic}, \mathbf{og}$ opposite and opposite and left directions & 
Straight and left-turn movements for perpendicular pairs. \\
\midrule
Complementary Right and U-Turn & 
$G_3 = \{A_3^{N-S}, A_3^{E-W}\}$ Complements under $\rho_a$ & 
2 & 
$\mathbf{ic}, \mathbf{og}$ same and right directions & 
Right-turn and U-turn movements for perpendicular pairs. \\
\bottomrule
\end{tabularx}
\end{table*}

\subsection{Rotation Operators for State Components}
\label{subsec:rotation-operators}

The intersection state contains representations of different dimensionalities, each of which requires a dedicated rotation operator to preserve the semantic correspondence among directional and movement-specific features under spatial transformation. For four-dimensional vectors representing per-approach quantities, such as queue lengths $\mathbf{L}^t=[l_e,l_n,l_w,l_s]^\top$ and maximum waiting times $\boldsymbol{\Gamma}^t=[\tau_e,\tau_n,\tau_w,\tau_s]^\top$, an anticlockwise rotation operator $\rho'\in\mathbb{R}^{4\times4}$ is defined as
\[
\rho'
=
\begin{bmatrix}
0 & 1 & 0 & 0\\
0 & 0 & 1 & 0\\
0 & 0 & 0 & 1\\
1 & 0 & 0 & 0
\end{bmatrix}.
\]
For a directional vector $Y=[Y_e,Y_n,Y_w,Y_s]^\top$, application of this operator yields $\rho'Y=[Y_n,Y_w,Y_s,Y_e]^\top$, corresponding to the directional mapping East $\rightarrow$ North, North $\rightarrow$ West, West $\rightarrow$ South, and South $\rightarrow$ East. The resulting transformation preserves the relative directional structure of per-approach features and is summarised in Table~\ref{tab:rotation-operators}.

For 16-dimensional representations encoding movement-specific information, including the priority vector $\boldsymbol{\theta}^t\in\mathbb{R}^{16}$, phase representation $\mathbf{p}^{t}\in\mathbb{R}^{16}$, and queue-pressure mask $\mathbf{M}_q^{t}\in\mathbb{R}^{16}$, the representation is constructed from a $4\times4$ movement matrix $\mathcal{X}=((X_{ij}))_{4\times4}\in\mathbb{R}^{4\times4}$, where $i,j\in\{E,N,S,W\}$ denote the incoming and outgoing directions, respectively. To preserve the directional structure during vectorisation, the matrix is flattened using the row-wise transformation $\mathcal{RT}(\cdot)$:
\[
X
=
\mathcal{RT}(\mathcal{X})
=
\left(
\mathcal{R}_1(\mathcal{X}),
\ \rho'\mathcal{R}_2(\mathcal{X}),
\ \rho'^2\mathcal{R}_3(\mathcal{X}),
\ \rho'^3\mathcal{R}_4(\mathcal{X})
\right)^{\top},
\]
where $\mathcal{R}_i(\mathcal{X})$ extracts the $i$-th row of $\mathcal{X}$. This transformation aligns the movement-specific components according to their corresponding incoming directions before concatenation, thereby maintaining consistent spatial semantics in the resulting 16-dimensional representation.

The corresponding anticlockwise rotation operator for the 16-dimensional representation is defined as
\[
\rho
=
\begin{bmatrix}
\mathbf{0}_{12\times4} & I_{12}\\
I_4 & \mathbf{0}_{4\times12}
\end{bmatrix},
\]
which cyclically permutes the four direction-aligned blocks of the transformed representation. Repeated applications of $\rho$ generate successive spatial rotations while preserving the internal directional structure of each movement block. The resulting transformations and their associated directional mappings are summarised in Table~\ref{tab:rotation-operators}.

\begin{table}[!ht]
\centering
\tiny
\caption{Rotation operators for $X\in\mathbb{R}^{16}$ and $Y\in\mathbb{R}^{4}$ under $C_4$ symmetry.}
\label{tab:rotation-operators}
\begin{tabularx}{\textwidth}{p{1.2cm} X p{1.2cm} X X}
\toprule
\multicolumn{2}{c}{$X\in\mathbb{R}^{16}$} &
\multicolumn{2}{c}{$Y\in\mathbb{R}^{4}$} &
\textbf{Direction Mapping} \\
\cmidrule(lr){1-2}\cmidrule(lr){3-4}\cmidrule(lr){5-5}
\textbf{Rotation} & \textbf{Expression} &
\textbf{Rotation} & \textbf{Expression} &
\\
\midrule

$0^\circ$ &
$X^{(0)}=I_{16}X$ &
$0^\circ$ &
$Y^{(0)}=I_4Y$ &
$E,N,W,S\rightarrow E,N,W,S$ \\
\midrule

$90^\circ$ &
$X^{(90)}=\rho X$ &
$90^\circ$ &
$Y^{(90)}=\rho'Y$ &
$E,N,W,S\rightarrow N,W,S,E$ \\
\midrule

$180^\circ$ &
$X^{(180)}=\rho^2X$ &
$180^\circ$ &
$Y^{(180)}=\rho'^2Y$ &
$E,N,W,S\rightarrow W,S,E,N$ \\
\midrule

$270^\circ$ &
$X^{(270)}=\rho^3X$ &
$270^\circ$ &
$Y^{(270)}=\rho'^3Y$ &
$E,N,W,S\rightarrow S,E,N,W$ \\
\bottomrule
\end{tabularx}
\end{table}

\subsection{Action Space Encoding and Rotation}
\label{subsec:action-encoding}

The action space comprises eight admissible signal-phase patterns, organized into three rotation groups according to the underlying traffic-engineering constraints. To characterize the rotational structure of these patterns, the incoming and outgoing directions are represented in the complex plane as $\mathcal{IC}=\{\alpha' e^{jk\pi/2}:k\in\{0,1,2,3\}\}\subset\mathbb{C}$ and $\mathcal{OG}=\{\beta e^{jk\pi/2}:k\in\{0,1,2,3\}\}\subset\mathbb{C}$, where $\alpha',\beta>0$ are scaling constants. The four directional orientations are represented by $k=0$: East ($\mathbf{e}=\alpha'$), $k=1$: North ($\mathbf{n}=j\alpha'$), $k=2$: West ($\mathbf{w}=-\alpha'$), and $k=3$: South ($\mathbf{s}=-j\alpha'$). The complete movement space is then obtained through the Minkowski sum $\mathcal{M}=\mathcal{IC}\oplus\mathcal{OG}=\{\mathbf{ic}+\mathbf{og}:\mathbf{ic}\in\mathcal{IC},\,\mathbf{og}\in\mathcal{OG}\}$.

Based on the traffic-engineering constraints, three fundamental pattern groups are defined as summarized in Table~\ref{tab:pattern-groups}. These groups generate exactly eight admissible signal-phase combinations, denoted by $\mathcal{P}=\{p_1^{E},p_1^{N},p_1^{W},p_1^{S},p_2^{N-S},p_2^{E-W},p_3^{N-S},p_3^{E-W}\}$. Each admissible phase $p_i \in \mathcal{P}$ is represented by an 8-dimensional one-hot vector ${a}_i=[0,\ldots,0,\underbrace{1}_{\text{position }i},0,\ldots,0]^\top\in\{0,1\}^{8}$. The corresponding 16-dimensional movement representation is recovered through a fixed transition matrix $A_{TR}\in\{0,1\}^{8\times16}$ according to $\mathbf{p}_i = {a}_i A_{TR}$, where $\mathbf{p}_i \in \{0,1\}^{16}$ is the movement-level representation of phase $p_i$, and $A_{TR}=[B_1\;B_2\;B_3\;B_4]$ with

{\footnotesize
\[
B_1=
\begin{bmatrix}
1&1&1&1\\
0&0&1&1\\
1&1&0&0\\
0&0&0&0\\
0&0&0&0\\
0&0&0&0\\
0&0&0&0\\
0&0&0&0
\end{bmatrix},
\qquad
B_2=
\begin{bmatrix}
0&0&0&0\\
0&0&0&0\\
0&0&0&0\\
0&0&0&0\\
1&1&1&1\\
0&0&1&1\\
1&1&0&0\\
0&0&0&0
\end{bmatrix},
\]

\[
B_3=
\begin{bmatrix}
0&0&0&0\\
1&1&0&0\\
0&0&1&1\\
1&1&1&1\\
0&0&0&0\\
0&0&0&0\\
0&0&0&0\\
0&0&0&0
\end{bmatrix},
\qquad
B_4=
\begin{bmatrix}
0&0&0&0\\
0&0&0&0\\
0&0&0&0\\
0&0&0&0\\
0&0&0&0\\
1&1&0&0\\
0&0&1&1\\
1&1&1&1
\end{bmatrix}.
\]
}

The columns of $A_{TR}$ follow the movement-space ordering $\mathcal{MS}=(\mathrm{EE},\mathrm{EN},\mathrm{EW},\mathrm{ES},\mathrm{NE},\mathrm{NN},\mathrm{NW},\mathrm{NS},\mathrm{WE},\mathrm{WN},\\\mathrm{WW},\mathrm{WS},\mathrm{SE},\mathrm{SN},\mathrm{SW},\mathrm{SS})$, while its rows follow the state-space ordering $\mathcal{SO}=$$(A_1^E,A_2^{E-W},A_3^{E-W},A_1^W,A_1^N,A_2^{N-S},A_3^{N-S},A_1^S)$. The corresponding one-hot representations are denoted by $\mathcal{SO}_{oh}=(a_1^E,a_2^{E-W},a_3^{E-W},a_1^W,a_1^N,a_2^{N-S},a_3^{N-S},a_1^S)$, with the complete one-hot action matrix given by $\mathcal{A}_{oh}=I_8$.

The resulting action space possesses $C_4$ rotational symmetry. In particular, for any rotation angle $\theta=r\pi/2$, where $r\in\{0,1,2,3\}$, the rotated movement space is given by
\[
\mathcal{M}_{\mathrm{rot}}
=
e^{j\theta}\cdot\mathcal{M}
=
\left\{
e^{j\theta}(\mathbf{ic}+\mathbf{og})
:
\mathbf{ic}\in\mathcal{IC},\,
\mathbf{og}\in\mathcal{OG}
\right\}.
\]
This symmetry is represented directly in the one-hot action space through an $8\times8$ block-diagonal rotation operator
\[
\rho_a
=
\begin{pmatrix}
P & \mathbf{0}_{4\times4}\\
\mathbf{0}_{4\times4} & S
\end{pmatrix},
\]
where $P$ governs the rotational transformation of the single-pivot group $G_1$ and is defined as
{\footnotesize
\[
P=
\begin{bmatrix}
0&1&0&0\\
0&0&1&0\\
0&0&0&1\\
1&0&0&0
\end{bmatrix},
\]
}
while $S$ governs the complementary groups $G_2$ and $G_3$ and is given by
{\footnotesize
\[
S=
\begin{bmatrix}
0&1&0&0\\
1&0&0&0\\
0&0&0&1\\
0&0&1&0
\end{bmatrix}.
\]}
Thus, $\rho_a$ preserves the admissible phase structure while cyclically transforming the directional components according to the $C_4$ symmetry of the intersection. This construction provides a consistent algebraic basis for generating rotationally equivalent state-action representations and for enforcing symmetry-aware learning across spatially equivalent traffic configurations.

\subsection{Transition Operator}

This operator cyclically permutes the admissible single-pivot actions while appropriately exchanging the complementary phase pairs, thereby preserving the structural constraints associated with valid signal-phase transitions. In addition to its role in defining the rotational structure of the action space, the operator is used to characterize transition smoothness between consecutive signal phases through the phase-transition matrix
{\footnotesize
\[
\mathcal{Q}_p
=
\begin{bmatrix}
0 & 1 & 0 & 0 & 0 & 0 & 0 & 0\\
0 & 0 & 1 & 0 & 0 & 0 & 0 & 0\\
0 & 0 & 0 & 1 & 0 & 0 & 0 & 0\\
0 & 0 & 0 & 0 & 1 & 0 & 0 & 0\\
0 & 0 & 0 & 0 & 0 & 1 & 0 & 0\\
0 & 0 & 0 & 0 & 0 & 0 & 1 & 0\\
0 & 0 & 0 & 0 & 0 & 0 & 0 & 1\\
1 & 0 & 0 & 0 & 0 & 0 & 0 & 0
\end{bmatrix}.
\]}
The matrix $\mathcal{Q}_p$ specifies the preferred cyclic ordering among the admissible signal phases and is used to quantify the consistency and smoothness of successive phase selections. The corresponding chronological ordering of the action states is defined as $\mathcal{SO}=(A_{1}^{E},A_{2}^{E-W},A_{3}^{E-W},A_{1}^{W},A_{1}^{N},A_{2}^{N-S},A_{3}^{N-S},A_{1}^{S})$.

During actor-network training, it $\mathcal{Q}_p$ is incorporated into a smoothed transition matrix $\mathcal{Q}$ to avoid assigning zero probability to non-preferred transitions. Specifically, the transition structure is regularised as $\mathcal{Q}=(1-\epsilon)\mathcal{Q}_p+((\epsilon/8))_{8\times8}$, where $\epsilon$ controls the degree of smoothing. This formulation retains the dominant cyclic structure encoded by $\mathcal{Q}_p$ while introducing a small uniform probability mass across all possible transitions. The resulting transition matrix provides a differentiable and numerically stable representation of phase-transition preferences.

\subsection{Rotational Data Augmentation}
\label{subsec:rotation-augmentation}

Using the operators defined above, we apply rotational augmentation during offline pretraining. Starting from the baseline dataset $\mathcal{D}' = \bigcup_{t=0}^{N} (a^t,\ \mathbf{p}^{t-1},\ \mathbf{L}^{t},\ \mathbf{\Gamma}^{t},\ \boldsymbol{\theta}^{t})$, four augmented versions are generated by applying successive $90^\circ$ anticlockwise rotations as $\mathcal{D}_{F} = \bigcup_{k=0}^{3} g^{k \cdot \pi/2}(\mathcal{D}')$, where $g^{k \cdot \pi/2}(\cdot)$ applies the appropriate rotation operators to each component via $g^{k \cdot \pi/2}(a^t, \mathbf{p}^{t-1}, \mathbf{L}^t, \mathbf{\Gamma}^t, \boldsymbol{\theta}^t) = (\rho_a^k a^t,\ \rho^k \mathbf{p}^{t-1},\ \rho'^k \mathbf{L}^t,\ \rho'^k \mathbf{\Gamma}^t,\ \rho^k \boldsymbol{\theta}^t)$.

\section{Empirical Reliability Analysis}
\label{app}

This section of the appendix provides supplementary empirical analyses of the reliability-related properties of SIGMA discussed in the main paper. The analyses are designed to complement the theoretical results and primary experimental evaluation without introducing additional theoretical assumptions. 

\subsection{Bootstrap Confidence Intervals}
\label{app:boot}

The emergency service level at a prescribed waiting-time threshold $\tau$ is estimated as the empirical proportion of emergency vehicles whose waiting time does not exceed the specified threshold. Given emergency waiting times $\{T_{em,i}\}_{i=1}^{n}$, the empirical service-level estimator is defined as
\[
\widehat{SL}(\tau)
=
\frac{1}{n}
\sum_{i=1}^{n}
\mathbb{I}\!\left(T_{em,i}<\tau\right),
\]
where $\mathbb{I}(\cdot)$ denotes the indicator function. To quantify the sampling uncertainty associated with this estimate, a nonparametric percentile bootstrap procedure is employed. Specifically, $B$ bootstrap samples of size $n$ are generated independently with replacement from the observed emergency waiting times, and the service-level estimator is recomputed for each resampled dataset. The empirical $2.5$th and $97.5$th percentiles of the resulting bootstrap distribution are then used to construct the two-sided $95\%$ percentile confidence interval. The complete procedure is summarized in Algorithm~\ref{alg:bootstrap_sl}.

\begin{algorithm}[h]
\caption{Percentile Bootstrap Confidence Interval for Emergency Service Level}
\label{alg:bootstrap_sl}
\tiny
\begin{algorithmic}

\State \textbf{Input:} Emergency waiting times $\{T_{em,i}\}_{i=1}^{n}$, threshold $\tau$, number of bootstrap replicates $B$

\State \textbf{Output:} $\widehat{SL}(\tau)$ and its $95\%$ percentile bootstrap confidence interval

\State Compute the empirical service-level estimate:
\[
\widehat{SL}(\tau)
=
\frac{1}{n}
\sum_{i=1}^{n}
\mathbb{I}\!\left(T_{em,i}<\tau\right)
\]

\For{$b=1$ to $B$}

    \State Draw a bootstrap sample
    $\{T_{em,i}^{(b)}\}_{i=1}^{n}$
    with replacement from $\{T_{em,i}\}_{i=1}^{n}$

    \State Compute
    \[
    \widehat{SL}^{(b)}(\tau)
    =
    \frac{1}{n}
    \sum_{i=1}^{n}
    \mathbb{I}\!\left(T_{em,i}^{(b)}<\tau\right)
    \]

\EndFor

\State Let $\widehat{SL}^{(0.025)}(\tau)$ and
$\widehat{SL}^{(0.975)}(\tau)$ denote the empirical $2.5$th and
$97.5$th percentiles of
$\{\widehat{SL}^{(b)}(\tau)\}_{b=1}^{B}$

\State \textbf{return}
\[
\left[
\widehat{SL}^{(0.025)}(\tau),
\widehat{SL}^{(0.975)}(\tau)
\right]
\]

\end{algorithmic}
\end{algorithm}

The bootstrap analysis is conducted using $n=50$ emergency observations sampled from $N=1000$ simulation episodes, with $B=10{,}000$ bootstrap replicates. The resulting service-level estimates and corresponding $95\%$ percentile confidence intervals at the $\tau=60$\,s threshold are reported in Table~\ref{tab:boot}. SIGMA (Full) attains the highest estimated emergency service level at the 60-s threshold and exhibits a comparatively narrow confidence interval, indicating low sampling variability in the observed emergency-service performance.

\begin{table}[!ht]
\centering
\tiny
\caption{Bootstrap Estimates of Emergency Service Level at a 60-s Threshold}
\label{tab:boot}
\begin{tabularx}{\textwidth}{l X X X}
\toprule
\textbf{Method} & $\widehat{SL}(60)$ & \textbf{95\% CI} & \textbf{CI Width} \\
\midrule
SIGMA (Full) & $0.992$ & $[0.984, 0.998]$ & $0.014$ \\
\midrule
SIGMA-QW & $0.847$ & $[0.812, 0.881]$ & $0.069$ \\
\midrule
DQN & $0.712$ & $[0.624, 0.794]$ & $0.170$ \\
\midrule
Fixed-Time & $0.341$ & $[0.258, 0.424]$ & $0.166$ \\
\bottomrule
\end{tabularx}
\end{table}

The bootstrap analysis complements the performance-loss characterization established in Theorem~\ref{thm:graceful_degradation} by providing an empirical assessment of sampling uncertainty in the emergency service-level estimates. In particular, the comparison between SIGMA (Full) and SIGMA-QW provides empirical evidence of the effect associated with removing the emergency-alignment component, while the corresponding confidence intervals quantify the variability of the estimated service levels arising from the finite sample of emergency observations.

\subsection{Graceful Degradation Under LLM Failure}
\label{app:deg}

Theorem ~\ref{thm:graceful_degradation} characterizes the performance implications of removing an objective component from the multi-faced formulation. To empirically examine the corresponding fallback mechanism, we simulate LLM failures independently at the episode level using a Bernoulli failure process with probability $p_{\mathrm{fail}}$. Whenever an LLM failure occurs, the controller automatically falls back to the SIGMA-QW configuration.

\[
(1-p_{\mathrm{fail}})J_{\mathrm{full}}
+
p_{\mathrm{fail}}J_{\mathrm{QW}},
\]

where $J_{\mathrm{full}}$ and $J_{\mathrm{QW}}$ denote the corresponding performance measures under the full and fallback configurations, respectively. This expression represents the expected performance under the simulated failure process and is used only for empirical evaluation.

Table~\ref{tab:perf_u} reports the resulting performance for several failure probabilities. As expected, increasing $p_{\mathrm{fail}}$ progressively shifts system performance toward the fallback configuration. The results demonstrate that the controller continues to operate under LLM failure without requiring an alternative control architecture, although performance degradation becomes more pronounced as the failure probability increases.

\begin{table}[!ht]
\centering
\tiny
\caption{Performance Under Simulated LLM Failure}
\label{tab:perf_u}
\begin{tabularx}{\textwidth}{X X X X}
\toprule
$p_{\mathrm{fail}}$ & \textbf{MPD (AEWT)} & \textbf{MPD (AWT)} & \textbf{MPD (ATP)} \\
\midrule
$0.0$ & $23.5$ & $77.5$ & $9.26$ \\
\midrule
$0.1$ & $28.4$ & $76.8$ & $9.98$ \\
\midrule
$0.5$ & $47.9$ & $76.2$ & $10.15$ \\
\midrule
$1.0$ & $72.4$ & $74.9$ & $10.31$ \\
\bottomrule
\end{tabularx}
\end{table}

This experiment provides empirical evidence for the fallback mechanism associated with Theorem ~\ref{thm:graceful_degradation}. Importantly, the observed degradation should not be interpreted as a separate theoretical bound; rather, it demonstrates empirically how the system behaves as the probability of losing the LLM-generated priority information increases.

\subsection{Sensitivity to Traffic-Demand Perturbations}
\label{app:sens}

To quantify the sensitivity of emergency-service performance to perturbations in traffic demand, we define the normalized demand sensitivity as
\[
S_{\delta}
=
\frac{\Delta\mathrm{AEWT}/\mathrm{AEWT}_{0}}{\delta},
\]
where $\mathrm{AEWT}_{0}$ denotes the baseline emergency average effective waiting time, $\Delta\mathrm{AEWT}=\mathrm{AEWT}_{\delta}-\mathrm{AEWT}_{0}$ denotes the corresponding change under a relative traffic-demand perturbation of magnitude $\delta$, and $S_{\delta}$ represents the normalized change in emergency waiting time per unit demand perturbation. Lower values of $S_{\delta}$ indicate lower sensitivity and greater robustness to variations in traffic demand.

\begin{table}[!ht]
\centering
\tiny
\caption{Sensitivity of Emergency Average Effective Waiting Time to Traffic-Demand Perturbations}
\label{tab:demand_sensitivity}
\begin{tabularx}{\textwidth}{X X X}
\toprule
\textbf{Method} & $\mathbf{S_{0.1}}$ & $\mathbf{S_{0.2}}$ \\
\midrule
SIGMA & $0.18$ & $0.35$ \\
\midrule
DQN & $0.41$ & $0.82$ \\
\midrule
Fixed-Time & $0.05$ & $0.11$ \\
\bottomrule
\end{tabularx}
\end{table}

As reported in Table~\ref{tab:demand_sensitivity}, SIGMA exhibits lower sensitivity than DQN at both perturbation levels, indicating greater robustness to variations in traffic demand. Although sensitivity increases with perturbation magnitude, SIGMA remains consistently less sensitive than DQN. The low sensitivity of Fixed-Time should be interpreted alongside its absolute emergency-service performance, as limited adaptability does not imply superior responsiveness.

This analysis complements Theorem~\ref{thm:smooth_controllability}, which characterises sensitivity to objective-weight perturbations. Here, the perturbation arises from external traffic demand while the objective weights remain fixed; hence, the results provide complementary empirical evidence of operational robustness rather than a direct consequence of the theorem.

\section{Generalized $C_k$ Rotational Augmentation Framework and Transfer Learning Direction}
\label{sec:ck-rotation-framework}

The $C_4$ framework developed in Section~\ref{sec:rotation-framework} exploits the four-fold rotational symmetry inherent to standard four-legged intersections. We now generalize this methodology to intersections with $k$ homogeneous approaches ($k \geq 3$), where the geometry exhibits $C_k$ cyclic symmetry: that is, invariance under rotations by $2\pi/k$. This extension is applicable to three-legged junctions ($k=3$), five-way intersections ($k=5$), six-way roundabouts ($k=6$), and symmetric configurations, provided all approaches exhibit identical lane configurations, turning movements, and traffic patterns.

\subsection{Group-Theoretic Foundation}
\label{subsec:ck-group-theory}

The cyclic group of order $k$ is defined as $C_k = \{ \omega^r : r \in \{0, 1, \ldots, k-1\} \}$, where $\omega = e^{j2\pi/k}$ with group operation given by complex multiplication and the primitive element $\omega$ satisfying $\omega^k = 1$. For a $k$-legged intersection, direction $i$ is represented as a complex number $\mathbf{d}_i = \alpha \omega^i = \alpha e^{j2\pi i/k}$ for $i \in \{0, 1, \ldots, k-1\}$ where $\alpha > 0$ is a scaling constant and the angle between adjacent approaches is $2\pi/k$, with examples including $k=3$: directions at $0^\circ, 120^\circ, 240^\circ$ (T-junction, Y-junction), $k=4$: directions at $0^\circ, 90^\circ, 180^\circ, 270^\circ$ (standard four-way), $k=5$: directions at $0^\circ, 72^\circ, 144^\circ, 216^\circ, 288^\circ$ (five-way), and $k=6$: directions at $0^\circ, 60^\circ, 120^\circ, 180^\circ, 240^\circ, 300^\circ$ (six-way).

\subsection{Generalized Rotation Operators for State Components}
\label{subsec:ck-state-operators}

The intersection state comprises feature representations of different dimensionalities, each requiring an appropriate rotation operator to preserve semantic consistency under the cyclic symmetry group $C_k$. For $k$-dimensional vectors representing per-approach quantities, such as queue lengths $\mathbf{L}^t=[l_0,l_1,\ldots,l_{k-1}]^\top$ and maximum waiting times $\boldsymbol{\Gamma}^t=[\tau_0,\tau_1,\ldots,\tau_{k-1}]^\top$, a generalized cyclic permutation operator $\rho'_k\in\mathbb{R}^{k\times k}$ is defined as

{\footnotesize
\[
\rho'_k
=
\begin{bmatrix}
0 & 1 & 0 & \cdots & 0\\
0 & 0 & 1 & \cdots & 0\\
\vdots & \vdots & \vdots & \ddots & \vdots\\
0 & 0 & 0 & \cdots & 1\\
1 & 0 & 0 & \cdots & 0
\end{bmatrix}.
\]}
For a directional vector $Y=[Y_0,Y_1,\ldots,Y_{k-1}]^\top$, application of $\rho'_k$ yields $\rho'_kY=[Y_1,Y_2,\ldots,Y_{k-1},Y_0]^\top$, corresponding to the cyclic directional mapping $i\rightarrow(i+1)\bmod k$. The operator therefore generates successive rotations of the directional representation and satisfies $(\rho'_k)^k=I_k$, $(\rho'_k)^r(\rho'_k)^s=(\rho'_k)^{(r+s)\bmod k}$, and $((\rho'_k)^r)^{-1}=(\rho'_k)^{k-r}$. The set $\{I_k,\rho'_k,(\rho'_k)^2,\ldots,(\rho'_k)^{k-1}\}$ forms a representation isomorphic to the cyclic group $C_k$, as summarized in Table~\ref{tab:general-rotation-operators}.

For movement-specific features, the priority representation $\boldsymbol{\theta}^t\in\mathbb{R}^{k^2}$ and the previous-phase representation $\mathbf{p}^{t-1}\in\mathbb{R}^{k^2}$ are derived from $k\times k$ movement matrices $\mathcal{X}=((X_{ij}))_{k\times k}\in\mathbb{R}^{k\times k}$, where $X_{ij}$ denotes the movement or flow associated with incoming direction $i$ and outgoing direction $j$, with $i,j\in\{0,1,\ldots,k-1\}$. To obtain a vector representation while maintaining directional consistency, the matrix is flattened using the generalized row-vector transformation
\[
X
=
\mathcal{RT}_k(\mathcal{X})
=
\left(
\mathcal{R}_1(\mathcal{X}),
\ \rho'_k\mathcal{R}_2(\mathcal{X}),
\ (\rho'_k)^2\mathcal{R}_3(\mathcal{X}),
\ \ldots,
\ (\rho'_k)^{k-1}\mathcal{R}_k(\mathcal{X})
\right)^\top
\in\mathbb{R}^{k^2},
\]
where $\mathcal{R}_i(\mathcal{X})$ extracts the $i$-th row of $\mathcal{X}$. This transformation aligns movement-specific components relative to their associated incoming directions before concatenation, thereby ensuring that equivalent traffic movements retain consistent semantic positions under cyclic rotations.

The corresponding anticlockwise rotation operator for the resulting $k^2$-dimensional representation is defined as the block cyclic permutation matrix $\rho_k\in\mathbb{R}^{k^2\times k^2}$:
\[
\rho_k
=
\begin{bmatrix}
\mathbf{0}_{k(k-1)\times k} & I_{k(k-1)}\\
I_k & \mathbf{0}_{k\times k(k-1)}
\end{bmatrix}.
\]
Equivalently, $\rho_k$ can be expressed as a $k\times k$ block matrix, with each block having dimension $k\times k$:
{\footnotesize
\[
\rho_k
=
\begin{bmatrix}
\mathbf{0} & I_k & \mathbf{0} & \cdots & \mathbf{0}\\
\mathbf{0} & \mathbf{0} & I_k & \cdots & \mathbf{0}\\
\vdots & \vdots & \vdots & \ddots & \vdots\\
\mathbf{0} & \mathbf{0} & \mathbf{0} & \cdots & I_k\\
I_k & \mathbf{0} & \mathbf{0} & \cdots & \mathbf{0}
\end{bmatrix}.
\]}
This operator cyclically permutes the $k$ directional blocks, each of dimension $k$, while preserving the internal movement structure associated with each block. It satisfies $(\rho_k)^k=I_{k^2}$, and successive powers $(\rho_k)^r$ cyclically shift the directional blocks by $r$ positions. The operator preserves the semantic correspondence between incoming and outgoing directions under arbitrary cyclic rotations, providing a consistent $k^2$-dimensional representation for symmetry-aware state transformations. The resulting rotational mappings are summarised in Table~\ref{tab:general-rotation-operators}.

\begin{table*}[t]
\centering
\tiny
\caption{General rotation operators for $k^2$- and $k$-dimensional vectors under $C_k$ symmetry.}
\label{tab:general-rotation-operators}
\begin{tabularx}{\textwidth}{p{1.5cm} X p{2.5cm} p{1.5cm} X p{2.5cm}}
\toprule
\multicolumn{3}{c}{$X\in\mathbb{R}^{k^2}$} &
\multicolumn{3}{c}{$Y\in\mathbb{R}^{k}$} \\
\cmidrule(lr){1-3}\cmidrule(lr){4-6}
\textbf{Rotation} & \textbf{Expression} & \textbf{Block Mapping} &
\textbf{Rotation} & \textbf{Expression} & \textbf{Direction Mapping} \\
\midrule

$0$ &
$X^{(0)} = I_{k^2}X$ &
$0,1,\ldots,k-1 \rightarrow 0,1,\ldots,k-1$ &
$0$ &
$Y^{(0)} = I_kY$ &
$0,1,\ldots,k-1 \rightarrow 0,1,\ldots,k-1$ \\
\midrule

$\frac{2\pi}{k}$ &
$X^{(1)} = \rho_kX$ &
$0,1,\ldots,k-1 \rightarrow 1,2,\ldots,k-1,0$ &
$\frac{2\pi}{k}$ &
$Y^{(1)} = \rho'_kY$ &
$0,1,\ldots,k-1 \rightarrow 1,2,\ldots,k-1,0$ \\
\midrule

$\frac{4\pi}{k}$ &
$X^{(2)} = (\rho_k)^2X$ &
$0,1,\ldots,k-1 \rightarrow 2,3,\ldots,k-1,0,1$ &
$\frac{4\pi}{k}$ &
$Y^{(2)} = (\rho'_k)^2Y$ &
$0,1,\ldots,k-1 \rightarrow 2,3,\ldots,k-1,0,1$ \\
\midrule

$\vdots$ &
$\vdots$ &
$\vdots$ &
$\vdots$ &
$\vdots$ &
$\vdots$ \\
\midrule

$\frac{2\pi r}{k}$ &
$X^{(r)} = (\rho_k)^rX$ &
cyclic block shift by $r$ &
$\frac{2\pi r}{k}$ &
$Y^{(r)} = (\rho'_k)^rY$ &
cyclic shift by $r$ \\
\midrule

$\vdots$ &
$\vdots$ &
$\vdots$ &
$\vdots$ &
$\vdots$ &
$\vdots$ \\
\midrule

$\frac{2\pi(k-1)}{k}$ &
$X^{(k-1)} = (\rho_k)^{k-1}X$ &
cyclic block shift by $k-1$ &
$\frac{2\pi(k-1)}{k}$ &
$Y^{(k-1)} = (\rho'_k)^{k-1}Y$ &
cyclic shift by $k-1$ \\
\bottomrule
\end{tabularx}
\end{table*}

\subsection{Generalized Action Space Encoding and Rotation}
\label{subsec:ck-action-encoding}

To characterize rotational symmetry for an arbitrary number of intersection legs $k$, the incoming and outgoing directions are represented as uniformly distributed points on the complex unit circle. Specifically, the incoming-direction set and outgoing-direction set are defined as $\mathcal{IC}_k=\{\alpha e^{j2\pi i/k}:i\in\{0,1,\ldots,k-1\}\}\subset\mathbb{C}$ and $\mathcal{OG}_k=\{\beta e^{j2\pi i/k}:i\in\{0,1,\ldots,k-1\}\}\subset\mathbb{C}$, respectively, where $\alpha,\beta>0$ are scaling constants. The corresponding movement space is obtained through the Minkowski sum $\mathcal{M}_k=\mathcal{IC}_k\oplus\mathcal{OG}_k=\{\mathbf{ic}+\mathbf{og}:\mathbf{ic}\in\mathcal{IC}_k,\mathbf{og}\in\mathcal{OG}_k\}$.

For a $k$-legged intersection, the cardinality and structure of the admissible signal-phase space depend on both $k$ and the associated movement-conflict graph. Accordingly, the pattern groups are defined according to the underlying traffic-engineering constraints, as summarized in Table~\ref{tab:ck-pattern-groups}. Several representative cases illustrate this dependence.

Let $m_k=|\mathbb{A}_k|$ denote the number of admissible signal phases for a $k$-legged intersection. Each phase $A_i\in\mathbb{A}_k$ is represented by an $m_k$-dimensional one-hot vector ${a}_i=[0,\ldots,0,\underbrace{1}_{\text{position }i},0,\ldots,0]^\top\in\{0,1\}^{m_k}$. The corresponding $k^2$-dimensional movement representation is then recovered using a fixed transition matrix $A_{TR}^{(k)}\in\{0,1\}^{m_k\times k^2}$ according to $A_i={a}_iA_{TR}^{(k)}$. This construction separates the compact phase-level action representation from its movement-level realization while preserving a deterministic mapping between the two spaces.

The resulting action space possesses $C_k$ rotational symmetry. For any rotation angle $\theta=2\pi r/k$, where $r\in\{0,1,\ldots,k-1\}$, the rotated movement space is given by
\[
\mathcal{M}_{k,\mathrm{rot}}
=
e^{j\theta}\cdot\mathcal{M}_k
=
\left\{
e^{j\theta}(\mathbf{ic}+\mathbf{og})
:
\mathbf{ic}\in\mathcal{IC}_k,\,
\mathbf{og}\in\mathcal{OG}_k
\right\}.
\]
The corresponding rotation operator on the $m_k$-dimensional one-hot action space is constructed as a block-diagonal matrix $\rho_k^a\in\mathbb{R}^{m_k\times m_k}$, with each block describing the permutation induced by rotation within a particular pattern group:
\[
\rho_k^a
=
\operatorname{diag}
\left(
P_k^{(1)},P_k^{(2)},\ldots,P_k^{(g)}
\right),
\]
where $g$ denotes the number of pattern groups. For a cyclic pattern group of size $s$, the corresponding permutation matrix is
{\footnotesize
\[
P_k^{(i)}
=
\begin{bmatrix}
0&1&0&\cdots&0\\
0&0&1&\cdots&0\\
\vdots&\vdots&\vdots&\ddots&\vdots\\
0&0&0&\cdots&1\\
1&0&0&\cdots&0
\end{bmatrix}.
\]}
For a complementary pair of size $2$, the corresponding permutation matrix reduces to $P_k^{(i)}=\begin{bmatrix}0&1\\1&0\end{bmatrix}$, whereas a fixed pattern of size $1$ is represented by $P_k^{(i)}=[1]$. Thus, the action-space rotation operator is constructed directly from the orbit structure of the admissible phase groups, allowing the same symmetry framework to be applied across intersections with different numbers of legs.

As example,  for $k=3$, assuming only the single-pivot group of size $3$, the corresponding rotation operator reduces to
{\footnotesize\[
\rho_3^a
=
\begin{bmatrix}
0&1&0\\
0&0&1\\
1&0&0
\end{bmatrix}.
\]}
Hence, the proposed construction provides a unified representation of rotationally equivalent actions for arbitrary intersection geometries while preserving the structural constraints imposed by their respective movement-conflict configurations.

\begin{table*}[t]
\centering
\tiny
\caption{Fundamental pattern groups for traffic signal actions under $C_k$ symmetry.}
\label{tab:ck-pattern-groups}
\begin{tabularx}{\textwidth}{p{2.5cm} X p{0.8cm} p{2.6cm} X}
\toprule
\textbf{Pattern Group} & \textbf{Rotation Behavior} & \textbf{Size} & \textbf{Condition} & \textbf{Description} \\
\midrule
Single-Pivot Dominant & 
$G_1^{(k)} = \{A_1^{(0)}, A_1^{(1)}, \ldots, A_1^{(k-1)}\}$ Cyclic under $\rho_k^a$ & 
$k$ & 
Pivot $\mathbf{ic} \in \mathcal{IC}_k$ & 
All movements from a single incoming direction $i$. \\
\midrule
Complementary Pairs (if $k$ even) & 
$G_2^{(k)} = \{A_2^{(i)}, A_2^{(i+k/2)}\}$ Complements under $\rho_k^a$ & 
$k/2$ & 
$\mathbf{ic}, \mathbf{og}$ opposite directions & 
Straight and turning movements for antipodal pairs. \\
\midrule
Additional Groups & 
$G_m^{(k)}$ (context-dependent) & 
Variable & 
Conflict graph dependent & 
Custom groups for non-conflicting movement sets. \\
\bottomrule
\end{tabularx}
\end{table*}

\subsection{Generalized Transition Operator}
\label{subsec:ck-transition}

The transition operator $\mathcal{Q}_p^{(k)} \in \mathbb{R}^{m_k \times m_k}$ encodes valid phase transitions for $k$-legged intersections, defined as $(\mathcal{Q}_p^{(k)})_{ij} = 1$ if phase $j$ can follow phase $i$ without conflict and $0$ otherwise, which for $C_k$ symmetry must satisfy the equivariance condition $\mathcal{Q}_p^{(k)} = \rho_k^a \cdot \mathcal{Q}_p^{(k)} \cdot (\rho_k^a)^\top$, meaning transition validity is invariant under rotation, and during training the exploration-augmented operator is $\mathcal{Q}^{(k)} = (1-\epsilon) \cdot \mathcal{Q}_p^{(k)} + \frac{\epsilon}{m_k} \mathbf{1}_{m_k \times m_k}$.

\subsection{Generalized Rotational Data Augmentation}
\label{subsec:ck-rotation-augmentation}

Using the operators defined above, we apply rotational augmentation during offline pretraining for arbitrary $k$-legged intersections. Starting from the baseline dataset $\mathcal{D}' = \bigcup_{t=0}^{N} (a^t,\ \mathbf{p}^{t-1},\ \mathbf{M_q}^{t},\ \mathbf{\Gamma}^{t},\ \boldsymbol{\theta}^{t})$, $k$ augmented versions are generated by applying successive $2\pi/k$ rotations as $\mathcal{D}_{F}^{(k)} = \bigcup_{r=0}^{k-1} g^{r \cdot 2\pi/k}(\mathcal{D}')$, where the generalized augmentation operator $g^{r \cdot 2\pi/k}(\cdot)$ applies the appropriate rotation operators to each component via $g^{r \cdot 2\pi/k}(a^t, \mathbf{p}^{t-1}, \mathbf{M_q}^t, \mathbf{\Gamma}^t, \boldsymbol{\theta}^t) = ((\rho_k^a)^r a^t,\ (\rho_k)^r \mathbf{p}^{t-1},\ (\rho_k)^r \mathbf{M_q}^t,\ (\rho'_k)^r \mathbf{\Gamma}^t,\ (\rho_k)^r \boldsymbol{\theta}^t)$.

\subsection{Structural Requirements for $C_k$ Preservation}
\label{subsec:ck-requirements}

Three requirements must be satisfied for $C_k$ equivariant training. First, the action space $\mathbb{A}_k$ must be closed under $C_k$ rotation, requiring $\forall a \in \mathbb{A}_k,\ \forall r \in \{0, \ldots, k-1\}: (\rho_k^a)^r a \in \mathbb{A}_k$, which can be verified by constructing the conflict graph $\mathcal{C}_k$ where vertices are movements and edges represent conflicts, with $\mathbb{A}_k = \{\text{maximal independent sets of } \mathcal{C}_k\}$, and checking that if $S$ is a maximal independent set, then its rotation $(\rho_k)^r S$ is also one. Second, the operator $\rho_k \in \mathbb{R}^{k^2 \times k^2}$ must cyclically permute the $k$ direction blocks while preserving the incoming$\rightarrow$outgoing semantic structure. Third, the policy must satisfy $\pi_\psi(a \mid s) = \pi_\psi((\rho_k^a)^r a \mid (\rho_k)^r s)$ for all $r$.

\subsection{Practical Applicability and Limitations}
\label{subsec:ck-applicability}

The $C_k$ framework is applicable to various intersection types as summarized in Table~\ref{tab:ck-applicability}, with $k=3$ (T-junction, Y-junction) being very common and valid under homogeneous conditions, $k=4$ (standard four-way) being most common and valid under the original framework, $k=5$ (five-way) being rare but valid if homogeneous, $k=6$ (six-way roundabout) being moderate and valid if homogeneous, and $k>6$ (complex multi-way) being very rare but theoretically valid. A critical constraint is homogeneity, requiring all $k$ approaches to have identical lane configurations (number of lanes, turn bays), permitted turning movements, and traffic demand patterns (in distribution). For $k$ even, antipodal directions exist (separated by $\pi$), enabling complementary pair groups $G_2^{(k)}$ with opposite-direction movements; for $k$ odd, no antipodal pairs exist, all groups are cyclic, and phase design requires conflict-angle-based analysis rather than perpendicular-pair intuition. A comparison of rotation framework components between $C_4$ and generalized $C_k$ is provided in Table~\ref{tab:c4-vs-ck}.

\begin{table}[!ht]
\centering
\begin{minipage}{0.48\textwidth}
\centering
\tiny
\caption{Applicability of $C_k$ framework by intersection type.}
\label{tab:ck-applicability}
\begin{tabular}{clcc}
\toprule
$k$ & \textbf{Intersection Type} & \textbf{Common?} & \textbf{$C_k$ Valid?} \\
\midrule
3 & T-junction, Y-junction & Very common & $\checkmark$ (homogeneous) \\
\midrule
4 & Standard four-way & Most common & $\checkmark$ (original framework) \\
\midrule
5 & Five-way (rare) & Rare & $\checkmark$ (if homogeneous) \\
\midrule
6 & Six-way roundabout & Moderate & $\checkmark$ (if homogeneous) \\
\midrule
$>6$ & Complex multi-way & Very rare & $\checkmark$ (theoretically) \\
\bottomrule
\end{tabular}
\end{minipage}
\hfill
\begin{minipage}{0.48\textwidth}
\centering
\tiny
\caption{Comparison of Rotation Framework Components: $C_4$ vs.\ Generalized $C_k$}
\label{tab:c4-vs-ck}
\begin{tabular}{lcc}
\toprule
\textbf{Component} & \textbf{$C_4$ (Original)} & \textbf{$C_k$ (Generalized)} \\
\midrule
State dimension (per-approach) & 4 & $k$ \\
\midrule
Movement matrix size & $4 \times 4 = 16$ & $k \times k = k^2$ \\
\midrule
State rotation operator & $\rho' \in \mathbb{R}^{4 \times 4}$ & $\rho'_k \in \mathbb{R}^{k \times k}$ \\
\midrule
Movement rotation operator & $\rho \in \mathbb{R}^{16 \times 16}$ & $\rho_k \in \mathbb{R}^{k^2 \times k^2}$ \\
\midrule
Action space size & 8 & $m_k$ (conflict-dependent) \\
\midrule
Action rotation operator & $\rho_a \in \mathbb{R}^{8 \times 8}$ & $\rho_k^a \in \mathbb{R}^{m_k \times m_k}$ \\
\midrule
Augmentation factor & 4 & $k$ \\
\midrule
Valid for & 4-way symmetric & $k$-way symmetric \\
\bottomrule
\end{tabular}
\end{minipage}
\end{table}

\subsection{Transfer Learning for Shared Multi-Intersection Control – Possible Extention Of SIGMA}
\label{a-subsec:Tr_learning}
A single pretrained actor--critic model is shared across all intersections $v\in\mathcal{V}$ in the traffic network $\mathcal{G}=(\mathcal{V},\mathcal{E})$, thereby eliminating the need for intersection-specific model training or parameterization. The shared policy is defined as $\pi_{\psi}({a}_v^{(t)}\mid\mathbf{s}_v^{(t)},\mathbf{m}_v^{(t)})$, where $\mathbf{s}_v^{(t)}$ denotes the local intersection state, ${a}_v^{(t)}$ is the selected control action, and $\mathbf{m}_v^{(t)}$ represents a learned coordination message obtained from neighboring intersections. Specifically, the coordination message is computed as
\[
\mathbf{m}_v^{(t)}
=
\sigma
\left(
\frac{1}{|\mathcal{N}(v)|}
\sum_{u\in\mathcal{N}(v)}
\mathbf{W}_{\mathrm{agg}}
\cdot
\operatorname{MsgNet}_{\theta}
\left(
\mathbf{s}_u^{(t)}
\right)
\right),
\]
where $\mathcal{N}(v)$ indicates the set of nearby intersections of $v$, $\operatorname{MsgNet}_{\theta}(\cdot)$ is a shared message-generation network, $\mathbf{W}_{\mathrm{agg}}$ is an aggregation transformation, and $\sigma(\cdot)$ defines the associated nonlinear activation. This approach allows each junction to collect information from its immediate neighborhood while maintaining a shared policy representation throughout the network.

The experiences created at various intersections are pooled into a global replay buffer, described as $\mathcal{D}_{\mathrm{global}}=\bigcup_{v\in\mathcal{V}}\bigcup_t(\mathbf{s}_v^{(t)},{a}_v^{(t)},r_v^{(t)},\mathbf{s}_v^{(t+1)})$. Then this shared critic is updated using transitions drawn from this network-wide experience replay pool. Its temporal-difference objective is given by
\[
\mathcal{L}_{\phi}
=
\mathbb{E}_{\mathcal{D}_{\mathrm{global}}}
\left[
\left(
r_v
+
\gamma
\max_{{a}'}
Q_{\phi_{\mathrm{target}}}
\left(
\mathbf{s}_v',
{a}',
\mathbf{m}_v'
\right)
-
Q_{\phi}
\left(
\mathbf{s}_v,
{a}_v,
\mathbf{m}_v
\right)
\right)^2
\right].
\]
Here, $\phi_{\mathrm{target}}$ denotes the parameters of the target critic, while $\mathbf{m}_v'$ represents the coordination message associated with the successor state. By training on experiences collected across the entire network, the critic learns a common value representation that captures recurring traffic-control patterns across heterogeneous intersections.

The shared actor is optimized using the corresponding policy-gradient objective,
\[
\nabla_{\psi}J
=
\mathbb{E}_{\mathcal{D}_{\mathrm{global}}}
\left[
\nabla_{\psi}
\log
\pi_{\psi}
\left(
{a}_v\mid
\mathbf{s}_v,\mathbf{m}_v
\right)
\cdot
\delta_v
\right],
\]
where $\delta_v$ denotes the temporal-difference error associated with the sampled transition. Therefore, both policy and value function parameters are learned from network-wide experiences rather than intersection-specific datasets.

An essential characteristic of this common design is that the amount of trainable parameters is independent of the number of intersections. In particular, $\|\psi\|+\|\phi\|+\|\theta\|+\|\mathbf{W}_{\mathrm{agg}}\}|=O(1)$ w.r.t $|\mathcal{V}|$ Thus, the model parameters do not have to grow linearly with the growth of network size, and there is no need to retrain separately for each intersection. This parameter-sharing structure also allows for zero-shot deployment on networks of different sizes. The learned message forwarding technique provides decentralized coordination using neighborhood knowledge without explicit per-intersection model modification.

\section{SUMO Experimental Setup}
\label{app:sumo}

\subsection{Network Construction}

Four signalized crossroads in Kolkata, India were chosen to reflect a wide variety of urban traffic situations and intersection geometry. The accompanying road networks were recreated using OpenStreetMap data with \texttt{netconvert} in SUMO v1.19.0. To increase the authenticity of the simulated surroundings, the created networks were further checked manually for lane layouts, signal-head placements, and approach geometry based on Google Earth imagery as well as any relevant municipal engineering drawings.

The considered intersections have the same eight-phase action space $\mathcal{A}$, providing a uniform control formulation throughout the experimental cases. Th e signal timing is dynamically regulated using the TraCI interface, which allows the corresponding control policies to specify phase transitions throughout the simulation. The signal timing design is intended to enforce realistic operational and safety limitations, e.g., a minimum green time of $27\,\mathrm{s}$, a yellow time of $2\,\mathrm{s}$, and an all-red clearance time of $1\,\mathrm{s}$. This standard simulation methodology guarantees that performance discrepancies are mostly due to control techniques and traffic situations rather than anomalies in the underlying signal timing or network architecture.

\subsection{Demand Generation}

The traffic demand patterns are derived from loop-detector data obtained from the Urban Transport Department, Government of West Bengal. The observed traffic-flow data are used to construct a time-varying demand representation for each approach $i$. Specifically, the temporal variation in the observed arrival rate is approximated using a truncated Fourier series of the form

$$
\chi_i(t) = \chi_{i,0} + \sum_{r=1}^{3} \mathcal{K}_{i,r}
\sin\left(\frac{2\pi t}{T_r} + \phi_{i,r}\right),
$$

where $T_1 = 3600~\text{s}$, $T_2 = 900~\text{s}$, and $T_3 = 300~\text{s}$ correspond to hourly, 15-min, and 5-min temporal demand variations, respectively. The parameters $(\chi_{i,0}, \mathcal{K}_{i,r}, \phi_{i,r})$ are estimated by least-squares fitting to the loop-detector observations. Consequently, the harmonic demand function provides a data-calibrated representation of the temporal traffic-demand pattern observed in the study area.

Vehicles are subsequently generated according to the calibrated arrival intensity $\chi_i(t)$ using a nonhomogeneous Poisson process, with the intensity constrained to $\chi_i(t) \geq 0$. The simulated traffic stream consists of passenger cars (95\%), buses (3\%), and motorcycles (2\%), with vehicle lengths of $5~\text{m}$, $12~\text{m}$, and $2~\text{m}$, respectively. Vehicle-following behavior is modeled using the Krauss car-following model with stochasticity parameter $\sigma = 0.5$. Static routes are pre-generated and remain unchanged throughout the simulation, with no dynamic rerouting. This procedure ensures that the stochastic demand used in the simulation preserves the principal temporal characteristics of the empirically observed traffic patterns while providing a controlled and reproducible demand-generation mechanism.

\subsection{State Extraction via TraCI}

At each decision time $t$, the traffic state is extracted from the SUMO simulation through the TraCI interface using lane-level detector information and the current signal phase. Incoming and outgoing queue lengths, maximum waiting times, and the previous signal phase are transformed into the structured state representation required by SIGMA-QW and the PressLight baseline. The resulting state is constructed consistently across intersections before being passed to the corresponding control policy.

\begin{algorithm}[ht]
\caption{State Extraction via TraCI}
\label{alg:state_extraction_traci}
\tiny
\begin{algorithmic}

\State \textbf{Input:} Intersection $v$, time $t$, detectors $D_{\mathrm{in}},D_{\mathrm{out}}$, phase encoder $A_{\mathrm{TR}}$; \textbf{Output:} $\mathbf{s}^t$.

\State Extract 
$\mathbf{q}^{\mathrm{in}}\leftarrow
[\operatorname{getLastStepVehicleNumber}(d)]_{d\in D_{\mathrm{in}}}$,
$\mathbf{q}^{\mathrm{out}}\leftarrow
[\operatorname{getLastStepVehicleNumber}(d)]_{d\in D_{\mathrm{out}}}$,
and\\
$\mathbf{p}^{t-1}\leftarrow
A_{\mathrm{TR}}\operatorname{getPhase}(v)\in\{0,1\}^{16}$.

\State Extract
$\boldsymbol{\Gamma}^{t}\leftarrow
[\max_{veh\in\mathcal{V}_i(t)}
\operatorname{getWaitingTime}(veh)]_{i=1}^{4}$.

\State Construct
$\mathbf{M}_{q(\mathrm{in})}^{t}\leftarrow
\operatorname{repeat\_each\_4}(\mathbf{q}^{\mathrm{in}})$
and
$\mathbf{M}_{q(\mathrm{out})}^{t}\leftarrow
\operatorname{cycle\_all\_4}(\mathbf{q}^{\mathrm{out}})$,
both in $\mathbb{R}^{16}$.

\State Compute
$\mathbf{M}_{q}'^{\,t}\leftarrow
\mathbf{M}_{q(\mathrm{in})}^{t}
-\mathbf{M}_{q(\mathrm{out})}^{t}$
and shift:
$\mathbf{M}_{q}^{t}\leftarrow
\mathbf{M}_{q}'^{\,t}
-\min(\mathbf{M}_{q}'^{\,t})$.

\State Construct
$\mathbf{s}_{\mathrm{SIGMA}}^{t}\leftarrow
[\mathbf{p}^{t-1},\mathbf{M}_{q}^{t},\boldsymbol{\Gamma}^{t}]$
and
$\mathbf{s}_{\mathrm{PL}}^{t}\leftarrow
[\mathbf{q}^{\mathrm{in}},\mathbf{q}^{\mathrm{out}},\mathbf{p}^{t-1}]$.

\State \textbf{Return:}
$\mathbf{s}_{\mathrm{SIGMA}}^{t}$ for SIGMA-QW or
$\mathbf{s}_{\mathrm{PL}}^{t}$ for PressLight.

\end{algorithmic}
\end{algorithm}

\begin{table*}[h]
\centering
\tiny
\caption{SUMO TraCI State Extraction for SIGMA-QW and Baseline Methods}
\label{tab:sumo-state-extraction}
\begin{tabularx}{\textwidth}{p{1.7cm} X p{0.5cm} p{1.5cm} X}
\toprule
\textbf{Step} & \textbf{TraCI API Call} & \textbf{Output Dim} & \textbf{Method} & \textbf{Description} \\
\midrule
1. Incoming queues & 
\texttt{lane.getLastStepVehicleNumber()} for $d \in D_{\text{in}}$ & 
$\mathbb{R}^4$ & 
All & 
Vehicle count per incoming approach \\
\midrule
2. Outgoing queues & 
\texttt{lane.getLastStepVehicleNumber()} for $d \in D_{\text{out}}$ & 
$\mathbb{R}^4$ & 
All & 
Vehicle count per outgoing approach \\
\midrule
3. Previous phase & 
\texttt{trafficlight.getPhase(tls\_id)} & 
$\{0,1\}^8$ & 
All & 
Current signal phase index encoded via $A_{TR} \to \{0,1\}^{16}$ \\
\midrule
4. Max waiting times & 
\texttt{vehicle.getWaitingTime()} per approach & 
$\mathbb{R}^4$ & 
SIGMA-QW only & 
$\Gamma_i = \max_{veh \in \text{approach}_i} \text{waitingTime}(veh)$ \\
\midrule
5a. Build $\mathbf{M}_{q(\text{in})}$ & 
\texttt{repeat\_each\_4}($\mathbf{q}^{\text{in}}$) & 
$\mathbb{R}^{16}$ & 
SIGMA-QW only & 
$(l_e^{\text{in}}, l_e^{\text{in}}, l_e^{\text{in}}, l_e^{\text{in}}, l_n^{\text{in}}, \dots, l_s^{\text{in}})^\top$ \\
\midrule
5b. Build $\mathbf{M}_{q(\text{out})}$ & 
\texttt{cycle\_all\_4}($\mathbf{q}^{\text{out}}$) & 
$\mathbb{R}^{16}$ & 
SIGMA-QW only & 
$(l_e^{\text{out}}, l_n^{\text{out}}, l_w^{\text{out}}, l_s^{\text{out}}, l_e^{\text{out}}, \dots)^\top$ \\
\midrule
5c. Net pressure & 
$\mathbf{M}'_q = \mathbf{M}_{q(\text{in})} - \mathbf{M}_{q(\text{out})}$ & 
$\mathbb{R}^{16}$ & 
SIGMA-QW only & 
Per-movement supply-demand imbalance \\
\midrule
5d. Shift non-negative & 
$\mathbf{M}_q = \mathbf{M}'_q - \min(\mathbf{M}'_q)$ & 
$\mathbb{R}_{\geq 0}^{16}$ & 
SIGMA-QW only & 
Ensures stable optimisation \\
\midrule
6. Final state & 
Concatenate components & 
Var. & 
All & 
SIGMA-QW: $[\mathbf{p}^{t-1}, \mathbf{M}_q^t, \boldsymbol{\Gamma}] \in \mathbb{R}^{36}$; PressLight: $[\mathbf{q}^{\text{in}}, \mathbf{q}^{\text{out}}, \mathbf{p}^{t-1}] \in \mathbb{R}^{16}$; Max-Pressure: $\mathbf{q}^{\text{in}}, \mathbf{q}^{\text{out}} \in \mathbb{R}^8$ (used on-the-fly) \\
\bottomrule
\end{tabularx}
\end{table*}

\subsection{Controller Implementations}

(i)Max-Pressure is implemented as a pure Python controller via TraCI with no training, performing greedy phase selection at each $\Delta t = 5$\,s (ii). PressLight is a DQN with experience replay (buffer size 50k, target update every 100 steps), network architecture $(128, 64)$ with ReLU activation, $\epsilon$-greedy exploration with $\epsilon = 1.0 \to 0.01$ over 1000 episodes, and optimized using Adam with learning rate $\eta = 10^{-4}$ and discount factor $\gamma = 0.95$ (iii). SIGMA-QW is an actor-critic with offline pretraining on 10k synthetic episodes (Section~\ref{subsec:actor_pretraining}) followed by online fine-tuning, where pretraining uses synthetic data with the same O-D structure but randomized rates while online training uses SUMO-generated transitions, with batch size 128 and soft update $\tau = 0.001$. Hyperparameter configuration for this setup is denoted in the table ~\ref{tab:sumo-hyper}. All experiments were run on 64GB DDR5 RAM and an NVIDIA RTX 5060. SUMO v1.19.0 with TraCI API. Random seeds: $\{42, 123, 456, 789, 2024\}$ per intersection.

\begin{table*}[t]
\centering

\begin{minipage}[t]{0.47\textwidth}
\centering
\tiny
\captionof{table}{Shared and method-specific hyperparameters.}
\label{tab:sumo-hyper}

\begin{tabularx}{\textwidth}{l c X}
\toprule
\textbf{Parameter} & \textbf{Value} & \textbf{Methods} \\
\midrule
Control interval $\Delta t$ & $5\,\text{s}$ & All \\
\midrule
Episode length & $3600\,\text{s}$ & All \\
\midrule
Warm-up period & $300\,\text{s}$ & All \\
\midrule
Evaluation episodes & $1000$ & All \\
\midrule
Actor hidden layers & $(128,64,64)$ & SIGMA-QW \\
\midrule
Critic hidden layers & $(128,64)$ & SIGMA-QW \\
\midrule
DQN hidden layers & $(128,64)$ & PressLight \\
\midrule
Learning rate $\eta$ & $10^{-4}$ & SIGMA-QW, PressLight \\
\midrule
Discount $\gamma$ & $0.95$ & SIGMA-QW, PressLight \\
\midrule
Batch size & $128$ & SIGMA-QW, PressLight \\
\midrule
Replay buffer & $50\text{k}$ & PressLight \\
\midrule
Target update frequency & $100$ steps & PressLight \\
\midrule
Soft update $\tau$ & $0.001$ & SIGMA-QW \\
\midrule
$(\lambda_Q,\lambda_W)$ & $(0.7,1.3)$ & SIGMA-QW (loss) \\
\midrule
$(\alpha_Q,\alpha_W)$ & $(0.7,1.3)$ & SIGMA-QW (reward) \\
\bottomrule
\end{tabularx}

\end{minipage}
\hfill
\begin{minipage}[t]{0.51\textwidth}
\centering
\tiny
\captionof{table}{Statistical significance of SIGMA-QW against the baselines.}
\label{tab:significance}

\begin{tabularx}{\textwidth}{l c X X c}
\toprule
\textbf{Intersection} & \textbf{Metric} & \textbf{vs.\ Max-Pressure} & \textbf{vs.\ PressLight} & \textbf{Cohen's $d$} \\
\midrule
\multirow{4}{*}{KOL-1} & AWT  & $p<0.001$ & $p=0.082$ & $0.73$ \\
                        & AMWT & $p<0.001$ & $p<0.01$  & $0.91$ \\
                        & AQL  & $p<0.001$ & $p<0.05$  & $0.84$ \\
                        & ATP  & $p<0.001$ & $p=0.068$ & $0.69$ \\
\midrule
\multirow{4}{*}{KOL-2} & AWT  & $p<0.001$ & $p=0.074$ & $0.77$ \\
                        & AMWT & $p<0.001$ & $p<0.01$  & $0.94$ \\
                        & AQL  & $p<0.001$ & $p=0.074$ & $0.88$ \\
                        & ATP  & $p<0.001$ & $p<0.05$  & $0.82$ \\
\midrule
\multirow{4}{*}{KOL-3} & AWT  & $p<0.001$ & $p<0.05$  & $1.02$ \\
                        & AMWT & $p<0.001$ & $p<0.01$  & $1.15$ \\
                        & AQL  & $p<0.001$ & $p<0.01$  & $1.09$ \\
                        & ATP  & $p<0.001$ & $p=0.052$ & $0.75$ \\
\midrule
\multirow{4}{*}{KOL-4} & AWT  & $p<0.001$ & $p<0.001$ & $1.24$ \\
                        & AMWT & $p<0.001$ & $p<0.001$ & $1.38$ \\
                        & AQL  & $p<0.001$ & $p<0.001$ & $1.31$ \\
                        & ATP  & $p<0.001$ & $p<0.01$  & $0.97$ \\
\bottomrule
\end{tabularx}

\end{minipage}

\end{table*}

\subsection{Significance Test Details}
\label{app:significance}

To assess statistical robustness, paired two-tailed $t$-tests compared SIGMA-QW against each baseline (Max-Pressure and PressLight) across 1000 evaluation episodes per intersection. The null hypothesis assumed no difference in mean performance. Table~\ref{tab:significance} summarizes $p$-values and Cohen's $d$ effect sizes.

All comparisons with Max-Pressure yield $p < 0.001$ across every metric and intersection. Against PressLight, SIGMA-QW achieves significance at $p < 0.05$ for AMWT in all four intersections. For AQL, significance holds in KOL-1 ($p < 0.05$), KOL-3 ($p < 0.01$), and KOL-4 ($p < 0.001$); KOL-2 shows no significant difference ($p = 0.074$). For AWT, significance holds in KOL-3 ($p < 0.05$) and KOL-4 ($p < 0.001$), but not KOL-1 ($p = 0.082$) or KOL-2 ($p = 0.074$). For ATP, significance holds in KOL-2 ($p < 0.05$) and KOL-4 ($p < 0.01$), with KOL-1 ($p = 0.068$) and KOL-3 ($p = 0.052$) marginally above threshold. Effect sizes are large ($d > 0.8$) for most significant comparisons, with strongest effects in KOL-4 (AWT: $d = 1.24$, AMWT: $d = 1.38$, AQL: $d = 1.31$).

\section{Sensitivity Analysis with Varying Weight Configurations}
\label{app:sensitivity_full}

This appendix provides the complete sensitivity analysis of SIGMA with respect to different weight configurations $\mathbf{w} = \left(\frac{\boldsymbol{\alpha}}{\mathbf{1}^\top \boldsymbol{\alpha}}, \frac{\boldsymbol{\lambda}}{\mathbf{1}^\top \boldsymbol{\lambda}}\right)$ with $\sum_{k \in \{M,S,Q,W,E\}} w_k = 1$. Seven distinct weight configurations are evaluated, spanning the objective space from extreme prioritisation to balanced trade-offs.

\begin{table}[h]
\centering
\tiny
\caption{Weight configurations for sensitivity analysis. Each row shows the weight vector $\mathbf{w} = (w_E, w_W, w_Q, w_M, w_S)$ with $\sum w_k = 1$.}
\label{tab:app_weight_configurations}
\begin{tabularx}{\textwidth}{l *{5}{X} X}
\toprule
\textbf{Config} & $w_E$ & $w_W$ & $w_Q$ & $w_M$ & $w_S$ & \textbf{Focus} \\
\midrule
\textbf{C1: Balanced} & $0.20$ & $0.20$ & $0.20$ & $0.20$ & $0.20$ & Equal emphasis \\
\midrule
\textbf{C2: Emergency} & $0.50$ & $0.15$ & $0.15$ & $0.10$ & $0.10$ & Prioritize emergency response \\
\midrule
\textbf{C3: Queue} & $0.10$ & $0.15$ & $0.50$ & $0.125$ & $0.125$ & Maximize throughput \\
\midrule
\textbf{C4: Fairness} & $0.10$ & $0.50$ & $0.15$ & $0.125$ & $0.125$ & Minimize max waiting time \\
\midrule
\textbf{C5: Stability} & $0.10$ & $0.15$ & $0.15$ & $0.30$ & $0.30$ & Smooth, predictable transitions \\
\midrule
\textbf{C6: E+Queue} & $0.35$ & $0.10$ & $0.35$ & $0.10$ & $0.10$ & Balance emergency \& efficiency \\
\midrule
\textbf{C7: E+Fairness} & $0.35$ & $0.35$ & $0.10$ & $0.10$ & $0.10$ & Balance emergency \& fairness \\
\bottomrule
\end{tabularx}
\end{table}

\begin{sidewaystable}[b]
\centering
\tiny
\caption{Absolute performance, relative changes with respect to the balanced configuration (C1), and objective trade-offs across weight configurations. Absolute results are reported as mean $\pm$ standard deviation. Positive relative changes indicate improvement. Trade-off ratios represent the marginal change in average waiting time relative to improvements in emergency and maximum waiting time objectives.}
\label{tab:app_combined_sensitivity_tradeoffs}

\begin{tabularx}{\textwidth}{l *{14}{X}}
\toprule
&
\multicolumn{6}{c}{\textbf{Absolute Performance}}
&
\multicolumn{6}{c}{\textbf{Relative Change vs.\ C1 (\%)}}
&
\multicolumn{2}{c}{\textbf{Objective Trade-off}} \\
\cmidrule(lr){2-7}\cmidrule(lr){8-13}\cmidrule(lr){14-15}
\textbf{Configuration}
&
\textbf{AEWT $\downarrow$}
&
\textbf{AMWT $\downarrow$}
&
\textbf{AWT $\downarrow$}
&
\textbf{AQL $\downarrow$}
&
\textbf{TC $\uparrow$}
&
\textbf{ATP $\uparrow$}
&
\textbf{AEWT}
&
\textbf{AMWT}
&
\textbf{AWT}
&
\textbf{AQL}
&
\textbf{TC}
&
\textbf{ATP}
&
\shortstack{\textbf{$\Delta$AWT/} \\ \textbf{$\Delta$AEWT}}
&
\shortstack{\textbf{$\Delta$AWT/} \\ \textbf{$\Delta$AMWT}} \\
\midrule

\textbf{C1: Balanced} &
$72.54 \pm 1.8$ &
$176.17 \pm 20.29$ &
$71.94 \pm 7.1$ &
$9.18 \pm 1.1$ &
$0.42 \pm 0.06$ &
$10.03 \pm 0.98$ &
-- &
-- &
-- &
-- &
-- &
-- &
-- &
-- \\
\midrule

\textbf{C2: Emergency} &
$\mathbf{18.23 }\pm 2.1$ &
$135.41 \pm 15.7$ &
$79.21 \pm 8.7$ &
$11.72 \pm 1.5$ &
$0.37 \pm 0.06$ &
$9.12 \pm 1.04$ &
$\mathbf{+74.9}$ &
$+23.1$ &
$-10.1$ &
$-27.7$ &
$-11.9$ &
$-9.1$ &
$-0.13$ &
-- \\
\midrule

\textbf{C3: Queue} &
$74.39 \pm 3.2$ &
$142.37 \pm 17.2$ &
$73.16 \pm 7.1$ &
$\mathbf{9.14 }\pm 1.1$ &
$0.34 \pm 0.05$ &
$\mathbf{10.41 }\pm 1.21$ &
$-2.6$ &
$+19.2$ &
$-1.7$ &
$\mathbf{+0.4}$ &
$-19.0$ &
$\mathbf{+3.8}$ &
-- &
$0.11$ \\
\midrule

\textbf{C4: Fairness} &
$71.29 \pm 2.9$ &
$\mathbf{117.29 }\pm 14.1$ &
$74.19 \pm 8.3$ &
$10.87 \pm 1.3$ &
$0.36 \pm 0.06$ &
$9.53 \pm 1.11$ &
$+1.7$ &
$\mathbf{+33.4}$ &
$-3.1$ &
$-18.4$ &
$-14.3$ &
$-5.0$ &
-- &
-- \\
\midrule

\textbf{C5: Stability} &
$82.17 \pm 3.5$ &
$148.71 \pm 18.5$ &
$82.56 \pm 9.7$ &
$12.43 \pm 1.6$ &
$\mathbf{0.48 }\pm 0.04$ &
$8.47 \pm 0.98$ &
$-13.3$ &
$+15.6$ &
$-14.8$ &
$-35.4$ &
$\mathbf{+14.3}$ &
$-15.6$ &
-- &
-- \\
\midrule

\textbf{C6: E+Queue} &
$24.12 \pm 2.5$ &
$127.83 \pm 15.3$ &
$74.02 \pm 7.8$ &
$10.26 \pm 1.2$ &
$0.38 \pm 0.05$ &
$9.87 \pm 1.15$ &
$+66.7$ &
$+27.4$ &
$-2.9$ &
$-11.8$ &
$-9.5$ &
$-1.6$ &
$0.01$ &
$0.01$ \\
\midrule

\textbf{C7: E+Fairness} &
$19.87 \pm 2.3$ &
$121.45 \pm 14.6$ &
$75.91 \pm 8.9$ &
$10.94 \pm 1.3$ &
$0.36 \pm 0.05$ &
$9.34 \pm 1.07$ &
$+72.6$ &
$+31.1$ &
$-5.5$ &
$-19.2$ &
$-14.3$ &
$-6.9$ &
$-0.08$ &
$-0.07$ \\
\bottomrule
\end{tabularx}
\end{sidewaystable}

\subsection{Performance Across Weight Configurations}

(i) The absolute performance, relative change versus the balanced baseline (C1), and objective trade-offs for the weight configurations tested are shown in Table~\ref{tab:app_combined_sensitivity_tradeoffs}. Positive relative changes show progress relative to C1. The findings indicate that the different configurations have distinct trade-offs in performance: C2 is more for emergency response, C3 is more for queue-related efficiency, C4 is more for reducing maximum waiting time, and C5 is more for transition consistency. C6 strikes a solid balance between crisis responsiveness and queuing efficiency while remaining competitive with the other objectives.

(iii) The weight vectors and the corresponding operational priorities used in the sensitivity analysis are presented in Table~\ref{tab:app_weight_configurations}. The configurations include balanced, emergency-centric, queue-centric, fairness-centric, stability-centric, and combination-objective configurations to enable a systematic investigation of the effect of different operational priorities on the performance of SIGMA.

For general emergency-aware traffic management, configuration C6 is recommended since it provides a good trade-off between emergency and queue. Its weighting scheme emphasises both emergency response and queue efficiency without over-emphasising any one aim. Thus, C6 is an excellent general-purpose design which compromises nicely between the conflicting operational requirements.

\subsection{Markovian Consistency-Dominated Control Regime}
\label{app:markovian_details}

(i) We consider a specialized control regime, where the Markovian consistency and smoothness objectives are dominant, but the average waiting time has only a small weight: \[ w_M = \frac{1}{2}-\epsilon,\qquad w_S = \frac{1}{2}-\epsilon,\qquad w_W = 2\epsilon,\qquad w_E = w_Q = 0, \] where $0<\epsilon\ll1$. As $\epsilon$ decreases the control objective emphasizes more the consistency and smoothness of phase transitions.

(ii) In this setting, the resulting phase-selection behavior is strongly concentrated around the prescribed cyclic transition structure. The transition probabilities in Table~\ref{tab:app_markovian_transitions} indicate that each phase is likely to remain active or to move to its designated successor, with very low probability of transitions to non-adjacent phases.
Such behavior is also exemplified over independent simulation runs in Figure~\ref{fig:phase_transition} where the phase trajectories show repeated local transitions along the prescribed phase sequence.

(iii) Table~\ref{tab:app_markovian_regime} compares the Markovian regime to the balanced (C1) and emergency-queue (C6) configurations. The Markovian configuration has the highest transition consistency of the reported configurations, but this gain in consistency comes at the cost of worse performance on several demand-oriented measures. This demonstrates the inherent trade-off between the enforcement of regularity of the phase transition and the maximization of traffic-efficiency objectives.

(iv) We study the sensitivity of the Markovian regime w.r.t. $\epsilon$ in Table~\ref{tab:app_epsilon_sensitivity}. The results show that with decreasing $\epsilon$, stricter transition consistency is imposed, while increasing $\epsilon$ changes the behavior to improved waiting-time performance. Therefore, $\epsilon$ gives a direct way to tune the tradeoff between Markovian consistency and traffic-efficiency objectives.

(v) To put it simply, the Markovian regime is an operating mode dictated by consistency, not a universally dominant one. Its main advantage is that it yields regular and predictable phase transitions, but its consistency-driven objective can come at the cost of performance for demand-sensitive traffic measures. The regime is thus best suited for cases where predictable and smooth phase evolution is preferred to aggressive optimisation of traffic efficiency objectives.

\begin{table}[!ht]
\centering
\tiny
\caption{Performance of the Markovian consistency-dominated regime ($\epsilon = 0.01$).}
\label{tab:app_markovian_regime}
\begin{tabularx}{\textwidth}{l *{6}{X} X}
\toprule
\textbf{Configuration} & \textbf{AEWT (s)} & \textbf{AMWT (s)} & \textbf{AWT (s)} & \textbf{AQL} & \textbf{TC} & \textbf{ATP} & \textbf{Stability} \\
\midrule
C1: Balanced & $72.54$ & $176.17$ & $71.94$ & $9.18$ & $0.42$ & $10.03$ & Medium \\
\midrule
C6: E+Queue & $24.12$ & $127.83$ & $71.23$ & $10.26$ & $0.38$ & $9.87$ & Medium \\
\midrule
Markovian ($\epsilon=0.01$) & $82.34$ & $152.18$ & $83.91$ & $12.87$ & $\mathbf{0.48}$ & $8.12$ & Very High \\
\midrule
Markovian ($\epsilon=0.005$) & $77.23$ & $144.87$ & $78.23$ & $12.01$ & $0.47$ & $8.67$ & Very High \\
\bottomrule
\end{tabularx}
\end{table}

\begin{table}[!ht]
\centering
\tiny
\caption{Phase transition probabilities under the Markovian regime ($\epsilon = 0.01$). Rows represent the current phase; columns represent the next phase (Figure~\ref{fig:phase_transition}).}
\label{tab:app_markovian_transitions}
\begin{tabularx}{\textwidth}{l *{8}{X}}
\toprule
\textbf{Current $\rightarrow$ Next} & \textbf{$A_1^E$} & \textbf{$A_2^{E-W}$} & \textbf{$A_3^{E-W}$} & \textbf{$A_1^W$} & \textbf{$A_1^N$} & \textbf{$A_2^{N-S}$} & \textbf{$A_3^{N-S}$} & \textbf{$A_1^S$} \\
\midrule
$A_1^E$ & $65.2\%$ & $34.8\%$ & $0.0\%$ & $0.0\%$ & $0.0\%$ & $0.0\%$ & $0.0\%$ & $0.0\%$ \\
\midrule
$A_2^{E-W}$ & $0.0\%$ & $71.4\%$ & $28.6\%$ & $0.0\%$ & $0.0\%$ & $0.0\%$ & $0.0\%$ & $0.0\%$ \\
\midrule
$A_3^{E-W}$ & $0.0\%$ & $0.0\%$ & $68.9\%$ & $31.1\%$ & $0.0\%$ & $0.0\%$ & $0.0\%$ & $0.0\%$ \\
\midrule
$A_1^W$ & $0.0\%$ & $0.0\%$ & $0.0\%$ & $72.3\%$ & $27.7\%$ & $0.0\%$ & $0.0\%$ & $0.0\%$ \\
\midrule
$A_1^N$ & $0.0\%$ & $0.0\%$ & $0.0\%$ & $0.0\%$ & $69.8\%$ & $30.2\%$ & $0.0\%$ & $0.0\%$ \\
\midrule
$A_2^{N-S}$ & $0.0\%$ & $0.0\%$ & $0.0\%$ & $0.0\%$ & $0.0\%$ & $73.1\%$ & $26.9\%$ & $0.0\%$ \\
\midrule
$A_3^{N-S}$ & $0.0\%$ & $0.0\%$ & $0.0\%$ & $0.0\%$ & $0.0\%$ & $0.0\%$ & $70.4\%$ & $29.6\%$ \\
\midrule
$A_1^S$ & $33.7\%$ & $0.0\%$ & $0.0\%$ & $0.0\%$ & $0.0\%$ & $0.0\%$ & $0.0\%$ & $66.3\%$ \\
\bottomrule
\end{tabularx}
\end{table}

\begin{table}[!ht]
\centering
\tiny
\caption{Sensitivity to $\epsilon$ in the Markovian regime.}
\label{tab:app_epsilon_sensitivity}
\begin{tabularx}{\textwidth}{l *{3}{X}}
\toprule
$\epsilon$ & \textbf{AWT (s) $\downarrow$} & \textbf{AMWT (s) $\downarrow$} & \textbf{TC $\uparrow$} \\
\midrule
$0.001$ & $85.34 \pm 8.7$ & $158.12 \pm 18.3$ & $\mathbf{0.49 \pm 0.02}$ \\
\midrule
$0.005$ & $82.89 \pm 8.3$ & $151.45 \pm 17.1$ & $0.48 \pm 0.03$ \\
\midrule
$0.010$ & $78.23 \pm 8.1$ & $144.87 \pm 16.1$ & $0.47 \pm 0.03$ \\
\midrule
$0.025$ & $75.12 \pm 7.8$ & $138.23 \pm 15.4$ & $0.45 \pm 0.04$ \\
\midrule
$0.050$ & $73.45 \pm 7.5$ & $133.89 \pm 14.8$ & $0.42 \pm 0.04$ \\
\bottomrule
\end{tabularx}
\end{table}

\begin{figure}[!h]
    \centering
    \includegraphics[width=0.9999\textwidth]{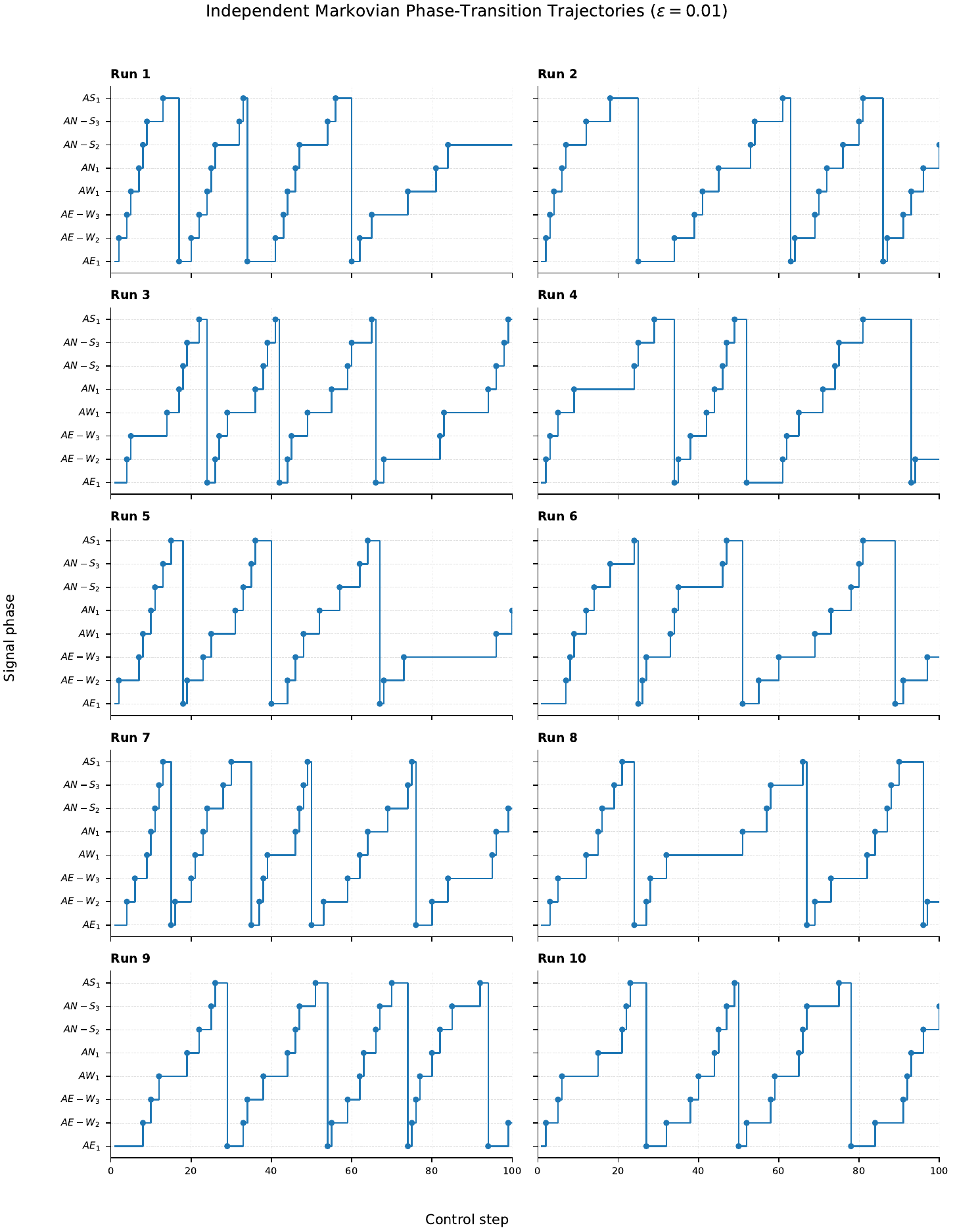}
    \caption{Sample  transitions under Markovian regime}
    \label{fig:phase_transition}
\end{figure}

\subsubsection*{Markovian Consistency-Dominated Regime: Pros, Cons and Application Scenarios}

The Markovian consistency-dominated regime emphasises temporal consistency and predictable phase transitions rather than quick reaction to temporary disturbances of traffic. In this regime, phase and strongly controlled transition dynamics prohibit phase skipping, making phase sequences very predictable. The hysteresis incorporation further minimises the quick oscillation between the competing signal phases and thereby reduces the wasted switching and improves the temporal stability of the control scheme. This consistency can also increase the interpretability and predictability of the signal function, hence enabling the anticipation of signal changes by road users.

However, the focus on consistency poses a fundamental trade-off between predictability and responsiveness. The controller is not flexible enough to cope with major deviations from the intended transition pattern and may respond adversely in emergency scenarios or to sharp variations in traffic demand. This regime can raise average and maximal waiting times and, in case of very dynamic demand, reduce the overall intersection throughput.

Thus, the Markovian consistency-dominated regime is most relevant to traffic scenarios with generally constant and well-organised demand patterns, and in particular when predictable signal behaviour is more critical than aggressive short-term adaptation. Typical applications include low-emergency corridors, intersections with predictable traffic demand, well-patterned flows, and locations where operational predictability is important, such as school zones or hospital access roads. In certain cases, the benefits of continuous phase sequencing and less wasteful switching may be greater than the downsides of a reduced reaction to fast disturbances.

\section{Mathematical Proofs}\label{proof}
\subsubsection*{Proof : Theorem ~\ref{thm:graceful_degradation}}
\begin{proof}
\label{prf:thm1}
For any fixed policy $\pi$, the map $\alpha\mapsto J(\pi;\alpha)$ is linear because expectation and infinite sums are linear operators:
\[
J(\pi;\alpha)=\sum_{j}\alpha_j \mathbb{E}_\pi\left[\sum_{t=0}^\infty \gamma^t r_j(s_t,a_t)\right]=\sum_j \alpha_j J^{(j)}(\pi).
\]
Hence
\[
J(\pi;\alpha)-J(\pi;\alpha^{(k)})=\alpha_k J^{(k)}(\pi).
\]
By Assumption 0.1 (bounded rewards), $|J^{(k)}(\pi)|\leq R_{\max}/(1-\gamma)$ for all $\pi$. Therefore
\[
\left|J(\pi;\alpha)-J(\pi;\alpha^{(k)})\right|\leq \frac{\alpha_k R_{\max}}{1-\gamma} \qquad \forall\pi\in\Pi.
\]

Now let $\pi_\alpha^*$ be optimal for the full weight vector $\alpha$. By definition of the supremum,
\[
J^*(\alpha^{(k)})\geq J(\pi_\alpha^*;\alpha^{(k)}).
\]
Consequently,
\[
\begin{aligned}
J^*(\alpha)-J^*(\alpha^{(k)})
&\leq J(\pi_\alpha^*;\alpha)-J(\pi_\alpha^*;\alpha^{(k)}) \\
&= \alpha_k J^{(k)}(\pi_\alpha^*) \\
&\leq \frac{\alpha_k R_{\max}}{1-\gamma},
\end{aligned}
\]
where the equality uses linearity and the final inequality follows from the boundedness of $J^{(k)}$. The second claim follows by rearranging the same inequality.
\end{proof}

\subsubsection*{Proof : Theorem ~\ref{thm:c4_equivariance}}
\begin{proof}
\label{prf:thm2}
Define the Bellman optimality operator
\[
(\mathcal{T}^* Q)(s,a)=r(s,a)+\gamma\sum_{s'} P(s'\mid s,a)\max_{a'} Q(s',a').
\]
Suppose $Q$ is $G$-invariant, i.e. $Q(g\cdot s,g\cdot a)=Q(s,a)$ for all $g,s,a$. Then for any $g\in G$,
\[
\begin{aligned}
(\mathcal{T}^*Q)(g\cdot s,\,g\cdot a)
&= r(g\cdot s,g\cdot a)+\gamma\sum_{s'}P(s'\mid g\cdot s,g\cdot a)\max_{a'}Q(s',a') \\
&= r(s,a)+\gamma\sum_{s'}P(g^{-1}\cdot s'\mid s,a)\max_{a'}Q(s',a') && \text{(by invariance)} \\
&= r(s,a)+\gamma\sum_{s''}P(s''\mid s,a)\max_{a'}Q(g\cdot s'',a') && \text{($s''=g^{-1}\cdot s'$)} \\
&= r(s,a)+\gamma\sum_{s''}P(s''\mid s,a)\max_{a'}Q(s'',g^{-1}\cdot a') && \text{(invariance of $Q$)} \\
&= r(s,a)+\gamma\sum_{s''}P(s''\mid s,a)\max_{a''}Q(s'',a'') && \text{(relabel $a''=g^{-1}\cdot a'$)} \\
&= (\mathcal{T}^*Q)(s,a).
\end{aligned}
\]
Thus $\mathcal{T}^*$ maps $G$-invariant functions to $G$-invariant functions. Since $\mathcal{T}^*$ is a $\gamma$-contraction in the sup-norm, it has a unique fixed point $Q^*$. Starting from the $G$-invariant zero function and iterating $Q_{n+1}=\mathcal{T}^*Q_n$, every iterate is $G$-invariant; by convergence, $Q^*$ is $G$-invariant.

Now consider the optimal policy $\pi^*(s)=\arg\max_a Q^*(s,a)$. For any $g\in G$,
\[
\max_a Q^*(g\cdot s,a)=\max_a Q^*(s,g^{-1}\cdot a)=\max_{a'} Q^*(s,a').
\]
Moreover, the argmax sets satisfy
\[
\arg\max_a Q^*(g\cdot s,a)=\{g\cdot a : a\in \arg\max_{a'} Q^*(s,a')\}.
\]
Therefore $\pi^*(g\cdot s)=g\cdot\pi^*(s)$.
\end{proof}

\subsubsection*{Proof : Theorem~\ref{thm:pressure_stability}}
\begin{proof}
\label{prf:thm3}
Define the Lyapunov function \cite{10057172} $V(q)=\frac12\|q\|^2=\frac12 q^\top q$. The one-step drift is
\[
\begin{aligned}
\Delta V(q)
&= \mathbb{E}[V(q_{t+1})-V(q_t)\mid q_t=q] \\
&= \mathbb{E}\left[q^\top(A_t-S(p_t))+\frac12\|A_t-S(p_t)\|^2 \mid q_t=q\right].
\end{aligned}
\]
Because arrivals are bounded and service is bounded by the finite intersection capacity, the second moment term is uniformly bounded by a constant $C<\infty$. Hence
\[
\Delta V(q)\leq q^\top(\nu-\mathbb{E}[S(p_t)\mid q])+C.
\]
Since $\nu\in\operatorname{int}(\mathcal{C})$, there exists a randomized stationary policy $\pi^\circ$ and a vector $\epsilon>0$ such that $\nu_i+\epsilon_i\leq \mathbb{E}_{p\sim\pi^\circ}[S_i(p)]$ for every approach $i$. The max-pressure rule maximizes the inner product with the pressure vector, which is equivalent to maximizing $q^\top S(p)$ when the outgoing queues are bounded relative to incoming queues. For all $q$,
\[
q^\top \mathbb{E}[S(p_t)\mid q]\geq q^\top \mathbb{E}_{p\sim\pi^\circ}[S(p)]\geq q^\top(\nu+\epsilon).
\]
Substituting into the drift inequality,
\[
\Delta V(q)\leq -q^\top\epsilon+C\leq -\epsilon_{\min}\|q\|_1+C,
\]
where $\epsilon_{\min}=\min_i \epsilon_i>0$. The drift is strictly negative whenever $\|q\|_1>C/\epsilon_{\min}$. By the Foster-Lyapunov stability criterion for Markov chains \cite{kompalli2009generalizedfosterlyapunovtypecriterion}, the chain is positive recurrent and $\sup_t \mathbb{E}[\|q_t\|_1]<\infty$.
\end{proof}

\subsubsection*{Proof : Theorem ~\ref{thm:smooth_controllability}}
\begin{proof}
\label{prf:thm4}
The critic's fixed point satisfies:
\[
Q_{\boldsymbol{\phi}}(s,a;\boldsymbol{\alpha}) = \mathbb{E}_{\pi_{\boldsymbol{\psi}}}\left[\sum_{t=0}^{\infty} \gamma^t r(s_t,a_t;\boldsymbol{\alpha}) \mid s_0=s,a_0=a\right].
\]
By linearity of the Bellman operator:
\[
Q_{\boldsymbol{\phi}}(\cdot;\boldsymbol{\alpha}) = \sum_{k=1}^5 \alpha_k Q_{\boldsymbol{\phi}}^{(k)}(\cdot).
\]
The policy gradient direction is:
\[
\nabla_{\boldsymbol{\psi}} J(\boldsymbol{\psi};\boldsymbol{\alpha}) = \mathbb{E}_{\pi_{\boldsymbol{\psi}}}\left[\sum_{k=1}^5 \alpha_k \delta^{(k)} \nabla_{\boldsymbol{\psi}}\log \pi_{\boldsymbol{\psi}}(a\mid s)\right],
\]
where $\delta^{(k)}=r_k^V+\gamma Q_{\boldsymbol{\phi}}^{(k)}(s',a')-Q_{\boldsymbol{\phi}}^{(k)}(s,a)$.

At stationarity $\nabla_{\boldsymbol{\psi}}J=0$, define $\mathbf{G}(\boldsymbol{\psi},\boldsymbol{\alpha})=\nabla_{\boldsymbol{\psi}}J$. Under local strong concavity of $J$ in $\boldsymbol{\psi}$, the Hessian $\nabla_{\boldsymbol{\psi}}^2J$ is negative definite. By the Implicit Function Theorem, there exists a smooth map $\boldsymbol{\psi}^*(\boldsymbol{\alpha})$ in a neighborhood of the stationary point. Differentiating:
\[
\frac{\partial \boldsymbol{\psi}^*}{\partial \alpha_k} = -(\nabla_{\boldsymbol{\psi}}^2 J)^{-1}\frac{\partial \mathbf{G}}{\partial \alpha_k} = -(\nabla_{\boldsymbol{\psi}}^2 J)^{-1}\mathbb{E}_{\pi_{\boldsymbol{\psi}}}\left[\delta^{(k)} \nabla_{\boldsymbol{\psi}}\log \pi_{\boldsymbol{\psi}}(a\mid s)\right].
\]
The Lipschitz constant is bounded by:
\[
L_\psi = \|(\nabla_{\boldsymbol{\psi}}^2 J)^{-1}\|\cdot \|\mathbb{E}[\delta\nabla_{\boldsymbol{\psi}}\log \pi]\|.
\]
Since the policy $\pi_\psi$ is Lipschitz in $\psi$ with constant $L_{\text{softmax}}$ (from Assumption~\ref{assum:regularity}), the composition $\pi_{\boldsymbol{\psi}^*(\boldsymbol{\alpha})}$ is Lipschitz in $\boldsymbol{\alpha}$ with constant $L_\pi=L_{\text{softmax}}\cdot L_\psi$.

If the advantages $\{A^{(k)}\}$ are linearly independent, then $\{\partial \boldsymbol{\psi}^*/\partial \alpha_k\}_{k=1}^5$ are linearly independent, establishing local injectivity.
\end{proof}

\subsubsection*{Proof: Theorem~\ref{thm:markovian_selection_1}}
\begin{proof}
\label{prf:thm5}
The optimal phase is determined by minimizing the expected cost functional
\[
J(p^{t+1}) = \mathbb{E}\left[ \sum_{k \in \{M,S,Q,W,E\}} w_k \cdot \mathcal{Q}_{p^t, p^{t+1}}^{(k)} + \xi \cdot \mathbf{1}_{\{p^{t+1} \neq p^t\}} \right],
\]
where the second term penalizes phase transitions to enforce smoothness. Under the given weighting scheme, the off-diagonal transition costs satisfy
\[
\mathcal{Q}_{p^t, p^{t+1}}^{(M)} = \mathcal{Q}_{p^t, p^{t+1}}^{(S)} = \mathcal{O}(1), \quad
\mathcal{Q}_{p^t, p^{t+1}}^{(W)} = \mathcal{O}(\epsilon^{-1}),
\]
and $\mathcal{Q}^{(E)} = \mathcal{Q}^{(Q)} = 0$ by construction.

Taking the limit $\epsilon \to 0^+$, the weight $w_W = 2\epsilon$ scales the $W$-component so that
\[
\lim_{\epsilon \to 0^+} w_W \cdot \mathcal{Q}_{p^t, p^{t+1}}^{(W)} = 2 \cdot \mathcal{D}_{p^{t+1}}^\dagger,
\]
where $\mathcal{D}_{p^{t+1}}^\dagger$ is the aggregate maximum waiting time pressure for the candidate next phase. Similarly,
\[
\lim_{\epsilon \to 0^+} \left( w_M \mathcal{Q}^{(M)} + w_S \mathcal{Q}^{(S)} \right) = \frac{1}{2}\left( \mathcal{Q}^{(M)} + \mathcal{Q}^{(S)} \right),
\]
which is independent of the phase selection and thus contributes only a constant offset to $J$.

The decision reduces to comparing the maximum waiting time pressures:
\[
p^{t+1} = \arg\min_{p' \in \{p^t, \operatorname{next}(p^t)\}} \left\{ -\mathcal{D}_{p'}^\dagger + \xi \cdot \mathbf{1}_{\{p' \neq p^t\}} \right\}.
\]
The penalty $\xi \cdot \mathbf{1}_{\{p' \neq p^t\}}$ vanishes in the limit $\xi \to 0^+$, which is consistent with the smoothness emphasis encoded by $w_M, w_S \to 1/2$. Therefore,
\[
p^{t+1} = 
\begin{cases} 
p^t, & \text{if } \mathcal{D}_{p^t}^\dagger \geq \mathcal{D}_{\operatorname{next}(p^t)}^\dagger, \\
\operatorname{next}(p^t), & \text{if } \mathcal{D}_{\operatorname{next}(p^t)}^\dagger > \mathcal{D}_{p^t}^\dagger.
\end{cases}
\]
Rewriting $\mathcal{D}_{p^t}^\dagger = \mathcal{D}_{\text{current}}^\dagger$ and $\mathcal{D}_{\operatorname{next}(p^t)}^\dagger = \mathcal{D}_{\text{next}}^\dagger$ yields the desired result.
\end{proof}

\end{appendices}

\clearpage
\printbibliography

\end{document}